\documentclass[twoside,11pt]{article}

\usepackage{jmlr2e}
\usepackage{array}
\usepackage[utf8]{inputenc}
\usepackage{bm}
\usepackage{mathrsfs} 
\usepackage{amsmath}
\usepackage{amssymb}
\usepackage{hyperref}
\usepackage{bbm} 
\usepackage{color}
\usepackage{graphicx}
\usepackage{caption}
\usepackage{subcaption}
\usepackage{booktabs}
\usepackage{makecell}

\DeclareMathOperator{\Z}{\mathbb{Z}}

\DeclareMathOperator{\RR}{\mathbb{R}}
\DeclareMathOperator{\C}{\mathbb{C}}
\newtheorem{assumption}{Assumption}
\newcommand{\R}{\mathbb{R}}
 
\newcommand{\E}{\mathbb{E}}
\newcommand{\PP}{\mathbb{P}}

\newcommand{\spr}{\mathrm{spr}}
\newcommand{\mtx}[1]{\boldsymbol{#1}}
\newcommand{\vct}[1]{\boldsymbol{#1}}
\DeclareMathOperator{\tr}{tr}
\renewcommand{\l}{\left}
\renewcommand{\r}{\right}

\newcommand{\va}{\vct{a}}

\newcommand{\mSigma}{\mtx{\Sigma}}
\newcommand{\mPhi}{\mtx{\Phi}}
\newcommand{\mPi}{\mtx{\Pi}}
\newcommand{\oPi}{\overline{\mtx{\Pi}}}
\newcommand{\mPsi}{\mtx{\Psi}}
\newcommand{\mTheta}{\mtx{\Theta}}

\newcommand{\mC}{\mtx{C}}

\newcommand{\mH}{\mtx{H}}
\newcommand{\mE}{\mtx{E}}

\newcommand{\mP}{\mtx{P}}
\newcommand{\mS}{\mtx{S}}

\newcommand{\mM}{\mtx{M}}
\newcommand{\mV}{\mtx{V}}

\newcommand{\mA}{\mtx{A}}

\newcommand{\vX}{\vct{X}}
\newcommand{\vZ}{\vct{Z}}
\newcommand{\vz}{\vct{z}}
\newcommand{\vY}{\vct{Y}}
\newcommand{\vv}{\vct{v}}
\newcommand{\vw}{\vct{w}}
\newcommand{\vs}{\vct{s}}

\newcommand{\mGamma}{\mtx{\Gamma}}
\newcommand{\vr}{\mathbf{r}}
\DeclareMathOperator{\diag}{diag}

\newcommand{\mI}{\mtx{I}}
\newcommand{\mz}{\mtx{0}}

\providecommand{\mL}{\mtx{L}}
\providecommand{\mR}{\mtx{R}}
\providecommand{\mB}{\mtx{B}}

\providecommand{\mY}{\mtx{Y}}

\newcommand{\cblue}[1]{#1}
\usepackage{lastpage}

\ShortHeadings{Scalable estimation of VARMA models}{Paulin and Elvira}
\firstpageno{1}

\begin{document}

\title{Scalable estimation of VARMA models}

\author{\name Daniel Paulin \email daniel.paulin@ntu.edu.sg \\
       \addr College of Computing \& Data Science\\
       Nanyang Technological University\\
       Singapore
       \AND
       \name V\'ictor Elvira \email victor.elvira@ed.ac.uk \\
       \addr School of Mathematics\\
       University of Edinburgh\\
       Edinburgh, UK}

\makeatletter\def\@starteditor{}\def\@editor{}\def\@endeditor{}\makeatother   

\maketitle
\thispagestyle{plain}   

\begin{abstract}%

Vector autoregressive moving-average (VARMA) models have long been considered impractical beyond moderate dimensions: the likelihood is non-convex, the parametrization is identified only up to equivalence, and every evaluation costs a pass over the entire series. Yet their moving-average term captures with a few parameters what a pure autoregression matches only with many lags. We introduce an estimation framework that removes this computational barrier: each optimization iteration is independent of the series length $T$. The framework combines a partial-autocorrelation reparametrization that guarantees stationarity and invertibility by construction, penalized least-squares and covariance-marginalized losses that depend on the data only through fixed-size sufficient statistics, and a Parseval (Fourier) evaluation whose cost is near-linear in the truncation length. This yields two point estimators: a regularized least-squares fit and a covariance-marginalized maximum-a-posteriori estimator; in both cases, using Gaussian priors on the reparametrized VARMA coefficients with separate standard deviations for diagonal and non-diagonal coefficients. We further prove that both estimators recover the infinite-autoregressive representation of the true process at a near-parametric rate in fixed dimension, so the truncation does not introduce asymptotic bias: both estimators still converge to the true dynamics at a near-parametric rate. The same sufficient-statistics machinery extends, at the same leading cost, to seasonal dynamics, exogenous regressors (VARMAX), and rolling-window refits. Empirically, the estimators stay close to the oracle forecast error from $d=10$ to $d=40$ (where the classical conditional-MLE implementation returns non-invertible fits whose forecasts diverge) and match or beat VAR, Bayesian-VAR, component-wise ARMA, and sparse-VARMA baselines on retail-demand, meteorological, and air-quality data. These results bring likelihood-based VARMA estimation, at a per-iteration cost independent of the series length, to the problem sizes at which practitioners have until now relied on VAR models.

\end{abstract}

\begin{keywords}
  VARMA, VARMAX, seasonal VARMA, high-dimensional time series, scalable estimation, partial autocorrelation, forecasting
\end{keywords}

\newpage
\section{Introduction}

Multivariate time series arise throughout science, and vector autoregressive moving average (VARMA) processes are a natural model class for them \citep{box2015time,reinsel2003elements,lutkepohl2005new,tsay2013multivariate,hannan2012statistical,kilian2017structural}. Joint modelling captures cross-variable dependence that componentwise analyses miss, and the class is closed under marginalization. As a result, the observed components of a partially observed vector autoregressive (VAR) model follow a VARMA process (see e.g.\ \citealp{Akaike1974}, or Section~3.7 of \citealp{tsay2013multivariate}).

The estimation of vector autoregressive (VAR) processes is comparatively simple; it reduces to solving large linear systems of equations and makes VARs applicable to high-dimensional series. VARMA estimation is however not easy to be applied. Despite much effort, its non-convex likelihood and identifiability problems have kept it from scaling, and the standard packages (\citealp{MTS}; \citealp{seabold2010statsmodels}) are often unreliable beyond about ten dimensions. In the dense regime, existing methods face important challenges. The likelihood-based methods (state-space Kalman-filter ML, the Whittle likelihood, Bayesian VARMA) are statistically efficient but scale poorly in the dimension or the series length, while the scalable ones (Hannan--Rissanen two-stage least squares, subspace identification) replace the likelihood by a cheaper convex criterion, remaining consistent but statistically less efficient. We review them in Section~\ref{sec:VARMAintro}. To the best of our knowledge, there is still missing  a one-stage, dense, likelihood-motivated estimator whose per-iteration cost does not grow with $T$ and that stays tractable as the number of series grows. We address that gap in this paper. Faced with these difficulties, practitioners have long favoured VARs over VARMAs even where a moving-average term would be more parsimonious \citep[Section~1.1]{lutkepohl2006forecasting}. 

In this paper, we develop a dense VARMA framework with computational efficiency and sound theoretical guarantees that can scale up to high dimension. We minimize a least-squares-type prediction loss, written in reparametrized coordinates (without loss of generality) in which stationarity and invertibility hold by construction. The loss depends on the data only through fixed-size sufficient statistics, so that each optimization iteration is independent of the series length $T$. As the experiments show, the resulting estimators are competitive in accuracy with the maximum likelihood estimator (MLE) on small models and remain tractable at dimensions where MLE is not.

Although we develop the method for the VARMA model, the same sufficient-statistics framework accommodates, without changing its leading computational cost, two extensions that matter in applications: a multiplicative \emph{seasonal VARMA} for periodic data (Section~\ref{sec:seasonal_varma}), and a \emph{VARMAX model} that conditions the dynamics on exogenous regressors. The exogenous terms enter the very same pre-computed sufficient statistics, so VARMAX inherits the same $T$-independent gradient cost. This underpins a central part of our empirical study: one of our three real-data benchmarks forecasts retail demand (units) \emph{conditional} on the observed price path (Section~\ref{sec:exp_dominicks}), a genuine VARMAX problem in which the moving-average dynamics and the exogenous price together prove decisive.

\begin{table}[h]
\centering
\renewcommand{\arraystretch}{1.15}
\caption{Per-gradient computational complexity of VARMA$(p,q)$ fitting methods. Our methods replace the per-gradient $O(T)$ cost of MLE-based approaches by a one-time pre-pass that scales with $T$, after which gradient evaluations are independent of $T$.  The Fourier (Parseval) evaluation of $\mS_2$ further reduces the dependence on the truncation length $m$ from $m^2$ to $m\log m$, making large-$m$ regimes tractable. The MLE bounds are representative dense companion-form costs, not lower bounds for every state-space algorithm. In the per-gradient column, the final $O((p^2{+}q^2)d^3)$ term for our methods is the $T$- and $m$-independent PAR-reparameterisation overhead; the preceding, data-dependent terms are those the Fourier evaluation reduces.}
\label{tab:complexity_intro}
{\small
\begin{tabular}{@{}>{\raggedright\arraybackslash}p{0.40\linewidth}ll@{}}
\toprule
Method & One-time pre-pass & Per-gradient evaluation \\
\midrule
\multicolumn{3}{@{}l}{\emph{Existing approaches}} \\
\quad MLE, recursive cond.\ likelihood \citep{tsay2013multivariate} & none &$O\bigl(T\,d^3\,(p+q)^2\bigr)$ \\
\quad MLE, state-space + Kalman filter \citep{durbin2012time,metaxoglou2007maximum} & none &$O\bigl(T\,d^3\,\max(p,q+1)^3\bigr)$ \\
\midrule
\multicolumn{3}{@{}l}{\emph{This paper: truncated VARMA$(p,q,m)$}} \\
\quad Sufficient-statistics evaluation (Section~\ref{sec:LSloss}) & $O\bigl(T\,m\,d^2\bigr)$ & $O\bigl(m^2\,d^3 + (p^2{+}q^2)d^3\bigr)$ \\
\quad \;+ Fourier (Parseval) evaluation of $\mS_2$ (Section~\ref{sec:complexity_S2}) & $O\bigl(T\log m\,d^2\bigr)$ & $O\bigl(m\log m\,d^2 + m\,d^3 + (p^2{+}q^2)d^3\bigr)$ \\
\bottomrule
\end{tabular}
}
\end{table}

Table~\ref{tab:complexity_intro} summarises the per-gradient computational cost of these approaches. Both existing MLE-based methods scale linearly in $T$ per gradient evaluation, since each optimizer step must touch the entire training history through either the recursive innovations or the Kalman filter. Our truncated VARMA approach absorbs the $T$-dependence into a one-time pre-pass that produces a fixed-size sufficient-statistic summary, after which every gradient evaluation is independent of $T$. The Fourier-based variant additionally avoids materialising the $m^2 d^2$ second-moment block matrix and reduces the per-gradient cost to nearly linear in $m$, making large truncation lengths computationally tractable. The figures in Table~\ref{tab:complexity_intro} are the dominant, data-dependent part of each gradient; the PAR reparametrisation adds a further $O((p^2{+}q^2)\,d^3)$ per evaluation, and the multiplicative \emph{seasonal} VARMA of Section~\ref{sec:seasonal_varma} evaluates at the same leading $T$-independent cost, with its full dependence on the orders $(p,q,P,Q)$, the period $s$, and $(m,d)$ given precisely in~\eqref{eq:sarma_cost}.

\paragraph{Contributions.}
(i) A scalable estimation framework for dense VARMA whose per-iteration optimization cost is independent of the series length $T$: the fixed-size sufficient statistics that summarize a VAR regression are shown to drive the nonlinear VARMA problem as well.
(ii) Statistical guarantees: both estimators recover the infinite-VAR representation at a near-parametric rate in fixed dimension, so scalability does not sacrifice consistency.
(iii) Practical extensions including (a) seasonal VARMA, exogenous regressors (VARMAX), and (b) incremental rolling-window refits, both at the same leading cost, validated on retail-demand, meteorological, and air-quality benchmarks.

The rest of the paper is organized as follows. Section~\ref{sec:VARMAintro} introduces VARMA processes and reviews the two standard MLE-based fitting strategies and their computational bottleneck (Section~\ref{sec:bg_mle}), the basic least-squares formulation (Section~\ref{sec:bg_ls}), the identifiability issues that we do not address (Section~\ref{sec:bg_identifiability}), and the existing software (Section~\ref{sec:bg_software}). Section~\ref{sec:truncated_varma} defines the \emph{truncated VARMA$(p,q,m)$} model and the stationarity/invertibility constraints it must obey. Sections~\ref{sec:LS} and~\ref{sec:MAP} develop the two pointwise loss functions we use to fit the truncated model, the regularised least-squares (RLS) loss and the $\Sigma$-marginalised MAP loss, written directly in the unconstrained PAR coordinates of Section~\ref{sec:VARMAintro} so that the stationarity and invertibility constraints are enforced by construction, and show that both depend on the data only through pre-computed sufficient statistics; Section~\ref{sec:consistency} proves that, for fixed dimension, both estimators recover the infinite-VAR representation of the true process at a near-parametric rate as $T$ and $m$ grow jointly. Both objectives are minimised by L-BFGS on the unconstrained PAR parameters $(\mM^{\mPhi}_{1:p},\mM^{\mTheta}_{1:q})\in\RR^{(p+q)d^2}$, with no projection step required. Section~\ref{sec:complexity_S2} replaces the materialised $m^2\,d^2$ second-moment block by a Fourier/Parseval representation, bringing the per-gradient evaluation cost down to $O(m\log m\,d^2 + m\,d^3)$. Section~\ref{sec:Experiments} presents our experimental results. Appendix~\ref{sec:statespace} gives the LG-SSM equivalence used in Section~\ref{sec:bg_mle}, and Appendix~\ref{sec:online} shows how the sufficient statistics, and hence the entire fit, can be updated incrementally as the training window slides forward, so that refits cost $O(\Delta\,m\,d^2)$ per shift.

\section{Background}

\subsection{VARMA models}\label{sec:VARMAintro}
For $t\in \Z$, let
\begin{equation}\label{eq:VARMAdef}
\vZ_{t}=\sum_{j=1}^{p} \mPhi_j\vZ_{t-j}+\va_t-\sum_{j=1}^{q}\mTheta_j \va_{t-j}
\end{equation}
be a stationary $d$-dimensional VARMA model with mean $\vct{0}$. Here $\vZ_t \in \mathbb{R}^d$, $\mPhi_j\in \mathbb{R}^{d\times d}$ for $1\le j\le p$, $\mTheta_j\in \mathbb{R}^{d\times d}$ for $1\le j\le q$, $\va_t\in \mathbb{R}^d$ is a vector white-noise process with mean $\vct{0}$ and covariance matrix $\mSigma\in \mathbb{R}^{d\times d}$ (independent over time).

\cblue{Let $\mPhi(z):=\mI-\sum_{j=1}^{p}\mPhi_j z^j$ and $\mTheta(z):=\mI-\sum_{j=1}^{q}\mTheta_j z^j$ be the matrix-valued AR and MA \emph{lag polynomials} (so that $\mPhi(0)=\mTheta(0)=\mI$), in terms of which the model~\eqref{eq:VARMAdef} reads $\mPhi(B)\vZ_t=\mTheta(B)\va_t$ with $B$ the backshift operator ($B\vZ_t=\vZ_{t-1}$); here $z\in \C$.}
The condition required for the stationarity of this model can be \cblue{established in the spectral domain as follows.}
\begin{assumption}[Stationarity and invertibility]\label{ass:roots}
We assume that
\begin{equation}\label{eq:roots}
\det \mPhi(z)\neq 0 \text{ and }\det \mTheta(z)\neq 0 \text{ if } |z|\le 1, z\in \C,
\end{equation}
i.e., all complex roots of $\det \mPhi(z)=0$ and $\det \mTheta(z)=0$ are outside the unit circle.
\end{assumption}

\paragraph{Companion matrices and spectral radius reformulation.}
Assumption~\ref{ass:roots} admits an equivalent characterisation in terms of the spectral radii of the AR and MA \emph{block companion matrices}. Define $\mathcal{C}_{\mPhi}\in \RR^{pd\times pd}$, $\mathcal{C}_{\mTheta}\in\RR^{qd\times qd}$ as
\begin{equation}\label{eq:companions}
\begin{aligned}
\mathcal{C}_{\mPhi}&:=\begin{pmatrix}\mz & \mI & \mz & \ldots & \mz\\
\mz & \mz & \mI & \ldots & \mz\\
\vdots & \vdots & \vdots & \ddots & \vdots\\
\mz & \mz & \mz & \ldots & \mI\\
\mPhi_p & \mPhi_{p-1} & \ldots  & \ldots &\mPhi_1
\end{pmatrix},\hspace{5mm}
\mathcal{C}_{\mTheta}&:=\begin{pmatrix}\mz & \mI & \mz & \ldots & \mz\\
\mz & \mz & \mI & \ldots & \mz\\
\vdots & \vdots & \vdots & \ddots & \vdots\\
\mz & \mz & \mz & \ldots & \mI\\
\mTheta_q & \mTheta_{q-1} & \ldots  & \ldots &\mTheta_1
\end{pmatrix}.
\end{aligned}
\end{equation}
The non-zero eigenvalues of $\mathcal{C}_{\mPhi}$ (resp.\ $\mathcal{C}_{\mTheta}$) are the reciprocals of the roots of $\det\mPhi(z)$ (resp.\ $\det\mTheta(z)$), so the root-outside-the-unit-circle conditions of Assumption~\ref{ass:roots} are equivalent to $\spr(\mathcal{C}_{\mPhi})<1$ and $\spr(\mathcal{C}_{\mTheta})<1$, where $\spr(\cdot)$ denotes the spectral radius. Section~\ref{sec:spectral_constraints} returns to the companion matrices to derive the geometric decay of the infinite-VAR coefficients.

\paragraph{Partial-autocorrelation reparametrisation.}
The stable set
\[
\{(\mPhi_{1:p},\mTheta_{1:q}):\spr(\mathcal{C}_{\mPhi})<1,\;\spr(\mathcal{C}_{\mTheta})<1\}
\]
is open, non-convex, and bounded by a complex algebraic variety, making constrained optimisation over it cumbersome. We will exploit the well-known fact that this set admits a smooth bijection from unconstrained Euclidean space, via the multivariate Levinson recursion in partial-autocorrelation form \citep{ansley1986note,heaps2023enforcing}. Reparametrisations of this kind have been used for estimation before: \citet{roy2019constrained} compute exact-likelihood estimates of causal invertible VARMA models by unconstrained optimization in closely related coordinates, and \citet{heaps2023enforcing} enforces stationarity of Bayesian VARs through a prior in these coordinates. Our contribution lies not in the parametrisation but in the loss evaluated on it: the truncated sufficient-statistic criteria of Sections~\ref{sec:LS}--\ref{sec:MAP}, whose per-iteration cost is independent of $T$. The construction is identical for the AR and MA sides, so we first describe a generic transformation acting on a stack of $r$ unconstrained lag matrices, then express the full VARMA reparametrisation as a factorised product.

\textit{Generic transformation $\tau_r$.}
For any $r\in\mathbb{N}$, we define a smooth bijection
\begin{equation}\label{eq:tau_def}
\tau_r : \RR^{r\times d\times d} \;\longrightarrow\; \bigl\{\mA_{1:r}\in\RR^{r\times d\times d}:\spr(\mathcal{C}_{\mA})<1\bigr\},
\end{equation}
where $\mathcal{C}_{\mA}\in\RR^{rd\times rd}$ is the block companion matrix of $\mA_{1:r}$ defined analogously to~\eqref{eq:companions}. Given an unconstrained stack $\mM_{1:r}=(\mM_1,\ldots,\mM_r)\in\RR^{r\times d\times d}$, the map $\tau_r$ produces $\mA_{1:r}=\tau_r(\mM_{1:r})$ in two stages.

\emph{Stage 1 (per-lag contraction).} Write $\mathrm{chol}(\mSigma)$ for the Cholesky factor of a positive-definite matrix $\mSigma$, the unique lower-triangular matrix with positive diagonal entries satisfying $\mathrm{chol}(\mSigma)\,\mathrm{chol}(\mSigma)^\top=\mSigma$. We state both stages in the Cholesky form in which they are implemented. For each $s=1,\ldots,r$, let $\mL_{A,s}:=\mathrm{chol}(\mI+\mM_s^\top\mM_s)$ and define
\begin{equation}\label{eq:contraction}
\mP_s \;:=\; \mM_s\,\mL_{A,s}^{-\top}.
\end{equation}
Then $\mP_s^\top\mP_s=\mI-(\mL_{A,s}^\top\mL_{A,s})^{-1}$, whose eigenvalues are $\sigma_i(\mM_s)^2/(1+\sigma_i(\mM_s)^2)<1$, so $\sigma_{\max}(\mP_s)<1$; conversely, given any $\mP$ with $\sigma_{\max}(\mP)<1$, setting $\mL_{A}^{-1}=\mathrm{chol}(\mI-\mP^\top\mP)$ and $\mM=\mP\mL_{A}^\top$ recovers the unique preimage. Hence~\eqref{eq:contraction} is a smooth diffeomorphism from $\RR^{d\times d}$ onto the open ball of strict contractions $\{\mP\in\RR^{d\times d}:\sigma_{\max}(\mP)<1\}$.

\emph{Stage 2 (Whittle--Ansley--Kohn recursion, innovation-anchored).} The recursion is anchored at a fixed innovation covariance, taken to be $\mI_d$, and is carried out on Cholesky factors $\mL^f_s$, $\mL^b_s$ of the forward and backward partial-prediction covariances $\mSigma^f_s=\mL^f_s(\mL^f_s)^\top$, $\mSigma^b_s=\mL^b_s(\mL^b_s)^\top$. The anchor is first propagated backward: set $\mL^f_r=\mI_d$ and, for $s=r,r-1,\ldots,1$,
\begin{equation}\label{eq:par_backward}
\mL^f_{s-1} \;=\; \mL^f_s\,\mathrm{chol}\bigl(\mI-\mP_s\mP_s^\top\bigr)^{-1}.
\end{equation}
The terminal factor determines the stationary variance $\mGamma_0=\mL^f_0(\mL^f_0)^\top$ consistent with innovation covariance $\mI_d$, and the backward chain starts from $\mL^b_0=\mL^f_0$. The coefficient pass then runs forward: for $s=1,\ldots,r$,
\begin{align}\label{eq:par_recursion}
\mA^{(s)}_s &\;=\; \mL^f_{s-1}\,\mP_s\,(\mL^b_{s-1})^{-1},\nonumber\\
\widetilde{\mA}^{(s)}_s &\;=\; \mL^b_{s-1}\,\mP_s^\top\,(\mL^f_{s-1})^{-1},\nonumber\\
\mA^{(s)}_j &= \mA^{(s-1)}_j - \mA^{(s)}_s\,\widetilde{\mA}^{(s-1)}_{s-j},\quad j=1,\ldots,s-1,\\
\widetilde{\mA}^{(s)}_j &= \widetilde{\mA}^{(s-1)}_j - \widetilde{\mA}^{(s)}_s\,\mA^{(s-1)}_{s-j},\nonumber\\
\mL^b_s &\;=\; \mL^b_{s-1}\,\mathrm{chol}\bigl(\mI-\mP_s^\top\mP_s\bigr),\nonumber
\end{align}
the forward factors $\mL^f_{s-1}$ being those produced by the backward pass~\eqref{eq:par_backward}. All factors are well defined, since $\mI-\mP_s\mP_s^\top\succ\mz$ and $\mI-\mP_s^\top\mP_s\succ\mz$ when $\sigma_{\max}(\mP_s)<1$, and a product of lower-triangular matrices with positive diagonals is again one, so every $\mL^f_s$ and $\mL^b_s$ is a genuine Cholesky factor. The two passes are mutually consistent by construction: the implied forward variance updates $\mSigma^f_s=\mL^f_{s-1}(\mI-\mP_s\mP_s^\top)(\mL^f_{s-1})^\top$ terminate at $\mSigma^f_r=\mI_d$, the innovation anchor. The output is $\tau_r(\mM_{1:r}) := (\mA^{(r)}_1,\ldots,\mA^{(r)}_r)$; the intermediate stack $\widetilde{\mA}^{(r)}_{1:r}$ encodes the time-reversed (backward-AR) coefficients and is not returned. Cholesky factorization is smooth (indeed real-analytic) on the positive-definite cone, so $\tau_r$ is smooth, including at $\mM=\mz$, where every factor equals $\mI_d$ and $\tau_r$ is locally the identity. The construction is the matrix partial-autocorrelation parametrisation of \citet{ansley1986note} and \citet{heaps2023enforcing}, the multivariate counterpart of the scalar partial-autocorrelation diffeomorphism of \citet{barndorff1973parametrization}; it has been used for VARMA coefficients before by \citet{roy2019constrained}. The variant that uses principal symmetric square roots in place of Cholesky factors differs from~\eqref{eq:par_recursion} only by a two-sided orthogonal change of the standardized coordinates of each $\mP_s$ and has the same image; we work with the Cholesky form throughout because it is the implemented one: it costs a triangular factorization and triangular solves per stage, and its automatic differentiation is well behaved at coincident eigenvalues (as at $\mM=\mz$), where eigendecomposition-based square-root implementations are not, although the principal square-root map itself is smooth there. Anchoring at the innovation covariance, rather than at a fixed marginal variance, is what makes the image the full stable set.

\begin{proposition}[PAR bijection]\label{prop:par_bijection}
For every $r\ge1$ and $d\ge1$, the map $\tau_r$ defined by~\eqref{eq:contraction}, \eqref{eq:par_backward} and~\eqref{eq:par_recursion} is a smooth bijection from $\RR^{r\times d\times d}$ onto the open stable set $\{\mA_{1:r}\in\RR^{r\times d\times d}:\spr(\mathcal{C}_{\mA})<1\}$, with smooth inverse.
\end{proposition}

The proof, in Appendix~\ref{sec:proof_par}, verifies that the backward pass~\eqref{eq:par_backward} is the unique solution of the anchoring equations, that every contraction stack produces a Schur-stable coefficient stack (a matrix form of the classical Schur--Levinson stability argument, via the $J$-unitarity of the Cholesky-normalized lattice section), that any stable coefficient stack arises from a unique contraction stack (through the stationary autoregression with innovation covariance $\mI_d$), that the forward pass reproduces the coefficients, and that both directions are smooth. The anchoring is what carries the full stable set: with a fixed marginal variance $\mGamma_0=\mI_d$ instead, a stationary $\mathrm{VAR}(1)$ with coefficient $\mA_1$ has innovation covariance $\mI_d-\mA_1\mA_1^\top\succ\mz$, so representability forces $\sigma_{\max}(\mA_1)<1$, a strict subset of the stable set (for instance $\mA_1=\left(\begin{smallmatrix}1/2&2\\0&1/2\end{smallmatrix}\right)$ is stable with $\sigma_{\max}>1$).

\textit{Joint VARMA reparametrisation.}
Applying $\tau_r$ separately to the AR side (with $r=p$) and the MA side (with $r=q$) yields the factorised reparametrisation
\begin{align}\label{eq:T_definition}
\nonumber&\mathcal{T}:\RR^{(p+q)\,d^2}\longrightarrow\bigl\{(\mPhi_{1:p},\mTheta_{1:q}):\spr(\mathcal{C}_{\mPhi})<1,\,\spr(\mathcal{C}_{\mTheta})<1\bigr\},\\
&\mathcal{T}\bigl(\mM^{\mPhi}_{1:p},\mM^{\mTheta}_{1:q}\bigr) := \bigl(\tau_p(\mM^{\mPhi}_{1:p}),\,\tau_q(\mM^{\mTheta}_{1:q})\bigr),
\end{align}
where $\mM^{\mPhi}_{1:p}\in\RR^{p\times d\times d}$ and $\mM^{\mTheta}_{1:q}\in\RR^{q\times d\times d}$ are unconstrained matrices. Because $\tau_p$ and $\tau_q$ are independent diffeomorphisms (Proposition~\ref{prop:par_bijection}), $\mathcal{T}$ is itself a smooth bijection between the unconstrained space $\RR^{(p+q)d^2}$ and the open stable set on the right-hand side of~\eqref{eq:T_definition}. We exploit this in Sections~\ref{sec:RLS} and~\ref{sec:MAP} to write the RLS and $\Sigma$-marginalised MAP losses as unconstrained functions of $(\mM^{\mPhi}_{1:p},\mM^{\mTheta}_{1:q})\in\RR^{(p+q)d^2}$, eliminating the need for projection steps or augmented-Lagrangian penalties during L-BFGS optimisation, as in \citet{roy2019constrained}.

As explained in \cite[Section 14.4.1]{box2015time}, under Assumption \ref{ass:roots},  we have
\begin{equation}\label{eq:infiniteVAR}
\vZ_t:=\sum_{j=1}^{\infty}\mPi_j \vZ_{t-j} +\va_t,
\end{equation}
where the matrices $\mPi_j$ satisfy the recursions
\begin{align}\label{eq:mPirec}
\mPi_j&=\mTheta_1 \mPi_{j-1}+\mTheta_2\mPi_{j-2}+\ldots+\mTheta_q \mPi_{j-q}+\mPhi_j, \quad j=1,2,\ldots,p,\\
\label{eq:mPirec2}
\mPi_j&=\mTheta_1 \mPi_{j-1}+\mTheta_2\mPi_{j-2}+\ldots+\mTheta_q \mPi_{j-q} \quad \text{ for }j>p,
\end{align}
where $\mPi_0=-\mI$, $\mPi_j=\mz$ for $j<0$.

\subsection{Maximum-likelihood estimation and its computational cost}\label{sec:bg_mle}

The natural estimator of the VARMA$(p,q)$ parameters $(\mPhi_{1:p},\mTheta_{1:q},\mSigma)$ is the maximum-likelihood estimator (MLE) under the Gaussian assumption $\va_t\overset{\mathrm{iid}}{\sim}N(\mz,\mSigma)$. Two standard ways of evaluating the Gaussian likelihood and its gradient are used in the literature, both with per-iteration cost that scales linearly in the sample size $T$.

\paragraph{Recursive evaluation of the conditional likelihood.}
The first approach evaluates the conditional likelihood $p(\vZ_{1:T}\mid \mPhi_{1:p},\mTheta_{1:q},\mSigma)$ recursively in $t$ by computing the one-step-ahead innovations $\hat{\vZ}_{t+1\mid 1:t}$ from the VARMA recursion; see Section~14.4.5 of \cite{box2015time} or Section~5.3.1 of \cite{reinsel2003elements} for the standard formulas. \citet{tsay2013multivariate} implements a conditional-likelihood gradient descent in which each iteration costs $O\!\bigl(d^3\,T\,(p+q)^2\bigr)$ flops; the dependence on $T$ comes from the inner forward pass, and the $(p+q)^2$ factor from the recursive structure of the gradient.

\paragraph{State-space evaluation via Kalman filter.}
The second approach exploits the fact that every stationary invertible VARMA$(p,q)$ admits an equivalent finite-dimensional linear Gaussian state-space (LG-SSM) representation; we give an explicit construction in Appendix~\ref{sec:statespace}. The Gaussian likelihood and its gradient with respect to $(\mPhi_{1:p},\mTheta_{1:q},\mSigma)$ can then be evaluated by running a Kalman filter on the equivalent LG-SSM. \citet{durbin2012time} provides a comprehensive treatment of state-space likelihood methods, and \citet{metaxoglou2007maximum} apply this machinery specifically to VARMA estimation. A standard companion-form LG-SSM realisation of a VARMA$(p,q)$ has state dimension
\begin{equation}\label{eq:ssm_dim}
d_x \,=\, d\,\max(p,\,q+1)
\end{equation}
(the minimal McMillan degree can be smaller under pole--zero cancellations or reduced-rank coefficient structure, but this companion realisation is the one used in practice), giving a representative per-likelihood-evaluation cost dominated by the $O(d_x^3)$ Kalman covariance update at each time step:
\begin{equation}\label{eq:kf_cost}
O\bigl(T\,d_x^3\bigr)\;=\;O\bigl(T\,d^3\,\max(p,q+1)^3\bigr).
\end{equation}

\paragraph{Common bottleneck.}
Both approaches share two characteristics that make them expensive in high dimensions. First, every gradient evaluation must touch all $T$ observations through the recursion (or the filter), so cost scales linearly in $T$. Second, the cubic dependence on $d$ (or on $d_x$) forbids scaling much beyond a few tens of dimensions even with modest $T$. In practice we observe that for $d\gtrsim 10$ MLE-based fits become impractical, taking hours per fit and exhibiting frequent convergence failures from the non-convex likelihood surface.

\subsection{Least squares estimators}\label{sec:bg_ls}

When the goal is forecasting rather than parameter identification, a common alternative is to fit the VARMA model by minimising the one-step-ahead squared prediction error. With \[\hat{\vZ}_{t+1\mid 1:t}=\E[\vZ_{t+1}\mid\vZ_{1:t}=\vz_{1:t}],\] where the expectation is with respect to the $\mathrm{VARMA}(p,q)$ predictive distribution with parameters $(\mPhi_{1:p},\mTheta_{1:q},\mSigma)$ treated as a deterministic tuple. We let
\begin{align}\label{eq:linftydef}
L_{\infty}(\mPhi_{1:p},\mTheta_{1:q},\mSigma):=\sum_{t=0}^{T-1}\|\hat{\vZ}_{t+1\mid 1:t}-\vz_{t+1}\|^2.
\end{align}
For $\mSigma=\mI$ and Gaussian vector white-noise process $\va_t$, $L_\infty$ coincides with the conditional Gaussian negative log-likelihood up to a constant (see Section~14.5 of \cite{box2015time}). Direct evaluation of $L_\infty$ is, however, no cheaper than the maximum-likelihood approach: the recursive structure of $\hat{\vZ}_{t+1\mid 1:t}$ forces every gradient evaluation to use all $T$ observations, with the same $O(T)$ scaling as the recursive MLE.

\subsection{Identifiability of VARMA parameters}\label{sec:bg_identifiability}

Even at the population level, the VARMA$(p,q)$ parameters $(\mPhi_{1:p},\mTheta_{1:q},\mSigma)$ are not in general uniquely determined by the observed process: distinct $(\mPhi_{1:p},\mTheta_{1:q},\mSigma)$ tuples can produce the same transfer function $\mPhi(L)^{-1}\mTheta(L)$, and thus the same Wold representation and the same predictive distribution of $\vZ_t$. The classical fix is to restrict attention to a canonical form (e.g.\ the \emph{Echelon form}) in which the parametrisation is uniquely identified by a tuple of so-called Kronecker indices, but locating those indices is itself a combinatorial search over an $O((p+q)^d)$-size set, and therefore scales poorly with $d$. \cite{duker2025vector} review this and related identifiability questions; \cite{wilms2023sparse} propose a convex two-stage procedure that first fits a truncated VAR to estimate residuals and then estimates $(\mPhi_{1:p},\mTheta_{1:q},\mSigma)$ via penalised least squares with a sparsity-inducing regulariser, side-stepping the Kronecker-index search at the cost of relying on the $\ell_1$ structure of the chosen penalty.

In this paper we do not address VARMA identifiability. Our focus is instead on building forecasts for high-dimensional time series, for which any tuple $(\mPhi_{1:p},\mTheta_{1:q},\mSigma)$ in an equivalence class of observationally indistinguishable parametrisations gives identical predictions of $\vZ_t$, and is therefore equally good. 

\subsection{Existing software for VARMA fitting}\label{sec:bg_software}

VARMA is implemented in several time-series packages, including \texttt{statsmodels}~\citep{seabold2010statsmodels} for Python and \texttt{MTS}~\citep{tsay2013multivariate} for R, both of which use one of the two MLE approaches above. Empirically these packages are slow even for moderate dimensions ($d\gtrsim 10$) and moderate sample sizes ($T\gtrsim 10^4$); a single fit can take hours and is prone to convergence failures, as discussed at length in \cite{lutkepohl2006forecasting}. Section~\ref{sec:exp_synth} documents these failures on simulated data. The remainder of this paper develops an alternative approach that scales well with both $d$ and $T$.

 \section{Truncated VARMA models}\label{sec:truncated_varma}

The classical VARMA$(p,q)$ likelihood is expensive to evaluate (Section~\ref{sec:bg_mle}). We exploit instead the infinite-VAR representation~\eqref{eq:infiniteVAR} of a stationary invertible VARMA$(p,q)$ and truncate it at a finite lag $m$, producing what we call a \emph{truncated VARMA$(p,q,m)$ model}: a VAR($m$) whose coefficients are constrained to lie on a lower-dimensional structured subset of $\RR^{m \times d\times d}$  defined by the recursions \eqref{eq:mPirec}-\eqref{eq:mPirec2}.
 The model is well-defined provided that the spectral radii of the AR and MA companion matrices stay strictly below $1$ (Section~\ref{sec:spectral_constraints}), and choosing $m$ large enough relative to the MA companion's spectral radius makes the truncation arbitrarily accurate (Section~\ref{sec:trunc_asymptotics}). The associated loss functions and pre-computed sufficient statistics, derived in Sections~\ref{sec:LS}--\ref{sec:complexity_S2}, keep the parameter count and per-iteration cost independent of $T$.

\subsection{The truncated VARMA$(p,q,m)$ model}\label{sec:trunc_varma_def}

A stationary invertible VARMA$(p,q)$ admits the infinite VAR representation~\eqref{eq:infiniteVAR},
\[
\vZ_t=\sum_{j=1}^{\infty}\mPi_j\vZ_{t-j}+\va_t,
\]
with coefficients $\mPi_1,\ldots, \mPi_p$ contained in the $(\mPi_{1:p},\mTheta_{1:q},\mSigma)$ parametrization, and $(\mPi_j)_{j>p}$ generated by the recursion~\eqref{eq:mPirec2} from $(\mPi_{1:p},\mTheta_{1:q})$. 

It is natural to consider a truncated version of this infinite sum. Approximating an invertible process by a finite autoregression of slowly growing order is the classical long-autoregression device of \citet{hannan1984multivariate} and \citet{lewis1985prediction}. We adopt the same approximation but do not fit the $m\,d^2$ coefficients of an unrestricted VAR$(m)$: the truncation enters only the loss, whose free parameters remain the $(p+q)\,d^2$ VARMA coefficients.
\begin{definition}[Truncated VARMA$(p,q,m)$]\label{def:tvarma}
Let $p,q,m$ be positive integers with $m\ge p+q$. The \emph{truncated VARMA$(p,q,m)$} model parameterised by $(\mPi_{1:p},\mTheta_{1:q},\mSigma)$ is
\begin{equation}
\label{eq:tvarma_def}
\vZ_t = \sum_{j=1}^m \mPi_j(\mPi_{1:p},\mTheta_{1:q})\,\vZ_{t-j} +\va_t,
\end{equation}
where $\va_t\in \mathbb{R}^d$ is a vector white-noise process with mean $\vct{0}$ and covariance matrix $\mSigma\in \mathbb{R}^{d\times d}$, and the lag matrices $\mPi_j\in\R^{d\times d}$ are obtained from the recursion~\eqref{eq:mPirec2} truncated at $m$ terms.
Equivalently, $(\mPi_1,\ldots,\mPi_m)$ lies in the image of the polynomial map
$(\mPi_{1:p},\mTheta_{1:q})\mapsto\mPi_{1:m}$, a subset of $\RR^{m\times d\times d}$ with at most $\big((p+q)d^2\big)$ degrees of freedom.
\end{definition}

Under Assumption~\ref{ass:roots} the spectral radius $\spr(\mathcal{C}_{\mTheta})$ of the MA companion in~\eqref{eq:companions} is strictly below $1$, and Eq.~\eqref{eq:mPij} together with Gelfand's formula~\eqref{eq:gelfand} imply $\|\mPi_j\|=O(\rho^j)$ for any $\rho\in(\spr(\mathcal{C}_{\mTheta}),1)$. The tail $\sum_{j=m+1}^{\infty}\mPi_j\vZ_{t-j}$ representing the truncation error is therefore exponentially small in $m$.

The truncated VARMA$(p,q,m)$ admits a VAR$(m)$ representation, and therefore has Markov order at most $m$, but is parameterized by only $(p+q)\,d^2$ free coordinates rather than the $m\,d^2$ of an unconstrained VAR$(m)$. The constraint encodes the assumption that the VAR coefficients arise from an underlying VARMA$(p,q)$ structure, which provides both interpretability and statistical efficiency, since the effective parameter count is independent of $m$.

\subsection{Spectral radius constraints}\label{sec:spectral_constraints}

For the truncated VARMA$(p,q,m)$ of Definition~\ref{def:tvarma} to be a meaningful approximation of the underlying VARMA($p$,$q$), the spectral radii of the AR and MA companion matrices must remain strictly below $1$. This subsection makes Assumption~\ref{ass:roots} explicit as a constraint on $(\mPhi_{1:p},\mTheta_{1:q})$ and links it to the geometric decay of the infinite-VAR coefficients $\mPi_j$. Throughout the rest of the paper we enforce these constraints by construction via the unconstrained PAR reparametrisation of Section~\ref{sec:VARMAintro}, so no projection step is needed during the L-BFGS optimisation of the loss functions of Sections~\ref{sec:LS}--\ref{sec:MAP}.

\paragraph{Decay of $\mPi_j$ via the companion matrices.}
Recall from Section~\ref{sec:VARMAintro} that Assumption~\ref{ass:roots} is equivalent to $\spr(\mathcal{C}_{\mPhi})<1$ and $\spr(\mathcal{C}_{\mTheta})<1$ for the block companion matrices defined in Eq.~\eqref{eq:companions}. Inverting the recursion~\eqref{eq:mPirec} for $j=1,\ldots,p$ also yields the reverse identity
\begin{equation}\label{eq:phi_from_pi_theta}
\mPhi_j \;=\; \mPi_j \,-\, \sum_{k=1}^{\min(j,q)}\mTheta_k\,\mPi_{j-k},\qquad j=1,\ldots,p,
\end{equation}
with $\mPi_0=-\mI$, which expresses the AR coefficients $\mPhi_{1:p}$ entering $\mathcal{C}_{\mPhi}$ in terms of $(\mPi_{1:p},\mTheta_{1:q})$. This recovery costs $O(pq\,d^2)$ flops.
The invertibility constraint $\spr(\mathcal{C}_{\mTheta})<1$ also governs the rate at which the truncation approximates the infinite-VAR representation~\eqref{eq:infiniteVAR}: setting
\(
\mV^{\mPi}:=(\mPi_{p-q+1},\ldots,\mPi_p)'
\)
(with $\mPi_0=-\mI$ and $\mPi_j=\mz$ for $j<0$), the recursion~\eqref{eq:mPirec2} gives, for $j\ge p+1$,
\begin{equation}\label{eq:mPij}
\mPi_j = \mE_q'\,\mathcal{C}_{\mTheta}^{\,j-p}\,\mV^{\mPi},
\qquad
\mE_q:=\bigl(\mz,\ldots,\mz,\mI\bigr)'\in\RR^{qd\times d}.
\end{equation}
Gelfand's formula \citep{lax2002functional} states that for any submultiplicative matrix norm $\|\cdot\|$,
\begin{equation}\label{eq:gelfand}
\spr(\mathcal{C}) = \lim_{k\to\infty}\|\mathcal{C}^k\|^{1/k},
\qquad \|\mathcal{C}^k\|^{1/k}\ge \spr(\mathcal{C}) \text{ for every } k\ge 1.
\end{equation}
Applied to~\eqref{eq:mPij}, it implies that $\|\mPi_j\|$ decays geometrically: for every $\rho>\spr(\mathcal{C}_{\mTheta})$ there is a constant $C_\rho$ with $\|\mPi_j\|\le C_\rho\,\rho^{j}$, so $\|\mPi_j\|\to 0$ as $j\to\infty$ (the precise, non-asymptotic constant is given in Appendix~\ref{sec:proof_trunc_bias}). The truncation error of the VARMA$(p,q,m)$ relative to the infinite-VAR representation therefore decays geometrically in $m$ once both spectral radii are bounded away from $1$; the next subsection makes this quantitative.

\subsection{Truncation bias}\label{sec:trunc_asymptotics}

The truncated VARMA$(p,q,m)$ generates a different stationary process than the original VARMA$(p,q)$, because the truncation drops a non-zero (though geometrically small) tail of the infinite-VAR sum. We quantify the mismatch through the stationary autocovariances. Let $\mGamma^{\infty}_l:=\E[\vZ_t\vZ_{t-l}']$ denote the autocovariance of the original VARMA$(p,q)$ at lag $l$, and $\mGamma^{(m)}_l$ the autocovariance of the truncated VARMA$(p,q,m)$ generated by Eq.~\eqref{eq:tvarma_def} with the same $(\mPhi_{1:p},\mTheta_{1:q},\mSigma)$.

\begin{proposition}[Non-asymptotic geometric truncation bias]\label{prop:trunc_bias}
Suppose Assumption~\ref{ass:roots} holds, let $\rho\in(\spr(\mathcal{C}_{\mTheta}),1)$, write $\mtx{A}(z):=\mTheta(z)^{-1}\mPhi(z)$ for the infinite-VAR operator, and set
\begin{equation}\label{eq:trunc_constants}
\begin{aligned}
&G_\rho:=\max_{|z|=1/\rho}\bigl\|\mI-\mtx{A}(z)\bigr\|_{\mathrm{op}},
\qquad
\mu_{\min}:=\min_{|z|=1}\lambda_{\min}\!\bigl(\mtx{A}(z)^{*}\mtx{A}(z)\bigr),\\
&\beta_\bullet:=\inf_{|z|\le1}\sigma_{\min}\bigl(\mtx{A}(z)\bigr),
\qquad
\beta_\circ:=\min_{|z|=1}\sigma_{\min}\bigl(\mtx{A}(z)\bigr)=\sqrt{\mu_{\min}}.
\end{aligned}
\end{equation}
Under Assumption~\ref{ass:roots} all four are finite, with $0<\beta_\bullet\le\beta_\circ$ ($G_\rho$ measures the non-normality of the moving-average inversion; $\beta_\circ=\sqrt{\mu_{\min}}$ is the on-circle conditioning of the spectral density of \citealp{basu2015regularized}). Define
\begin{equation}\label{eq:trunc_m0}
m_0:=\min\Bigl\{\,m\ge p+q:\ \frac{G_\rho\,\rho^{\,m+1}}{1-\rho}\le\frac{\beta_\bullet}{2}\,\Bigr\}.
\end{equation}
Then for every $m\ge m_0$ the truncated VARMA$(p,q,m)$ is itself stationary and
\begin{equation}\label{eq:gamma_bias}
\bigl\|\mGamma^{\infty}_l-\mGamma^{(m)}_l\bigr\|_{\mathrm{op}} \;\le\; C_{\rho}\,\rho^m
\quad\text{for every } l\ge 0,
\qquad
C_{\rho}:=\frac{6\,\|\mSigma\|_{\mathrm{op}}\,G_\rho\,\rho}{(1-\rho)\,\mu_{\min}^{3/2}}.
\end{equation}
The constants $G_\rho,\mu_{\min},\beta_\bullet,\beta_\circ,m_0$ are computable functions of $(\mPhi,\mTheta,\rho)$ alone, and $C_\rho$ depends additionally on $\mSigma$; in particular the order-zero covariance $\mGamma^{\infty}_0$ is recovered up to an explicit $O(\rho^m)$ error, uniformly in lag $l$.
\end{proposition}

The bound~\eqref{eq:gamma_bias} follows by comparing the spectral densities of the two stationary processes: they share the same innovation covariance $\mSigma$ and their autoregressive operators differ only by the dropped tail $\sum_{j>m}\mPi_j$, whose operator norm decays geometrically by a Cauchy estimate on the analytic operator $\mTheta(z)^{-1}\mPhi(z)$ (with an explicit, non-asymptotic constant), the passage to the autocovariance gap being governed by the on-circle conditioning $\mu_{\min}$ of the spectral density. The complete, fully non-asymptotic proof is given in Appendix~\ref{sec:proof_trunc_bias}.

In practice the constants in~\eqref{eq:gamma_bias} are unimportant: for any $\rho<1$, choosing $m=O(\log(1/\varepsilon)/\log(1/\rho))$ makes the bias smaller than any prescribed tolerance $\varepsilon$. The truncated VARMA$(p,q,m)$ is therefore a faithful surrogate for the original VARMA whenever $m$ is several times the time-scale set by the slowest mode of $\mathcal{C}_{\mTheta}$.

 \section{Least Squares loss for truncated VARMA models}\label{sec:LS}

We now develop the loss functions used to fit the truncated VARMA$(p,q,m)$ model of Definition~\ref{def:tvarma} from data. Section~\ref{sec:LSloss} introduces the basic least-squares objective $L_m$ and shows that it admits a sufficient-statistics decomposition that makes each evaluation cost $O(m^2 d^3)$, independent of the number of observations $T$. For a pure VAR, the reduction of least-squares criteria to sample autocovariances is classical (see, e.g., \citealp{lutkepohl2005new}); the point here is that the same fixed-size summary, sufficient for the conditional Gaussian truncated-VAR regression criterion, continues to drive the nonlinear truncated-VARMA problem, in which $\mPi_{1:m}$ depends on the free parameters through the map~\eqref{eq:M_to_Pi_Theta}. Section~\ref{sec:eval_loss} then explains how to evaluate $L_m$ in practice from these sufficient statistics. Section~\ref{sec:RLS} adds Bayesian regularisation to obtain the regularised least-squares (RLS) loss.

We work throughout in the \emph{unconstrained PAR parametrisation} of Section~\ref{sec:VARMAintro}: the free parameters are $(\mM^{\mPhi}_{1:p},\mM^{\mTheta}_{1:q})\in\RR^{(p+q)\,d^2}$, with the VARMA coefficients obtained as
\begin{equation}\label{eq:M_to_Pi_Theta}
\begin{aligned}
(\mPhi_{1:p},\mTheta_{1:q}) \;&=\; \mathcal{T}(\mM^{\mPhi}_{1:p},\mM^{\mTheta}_{1:q}),\\
\mPi_j \;&=\; \mPhi_j + \sum_{k=1}^{\min(j,q)}\mTheta_k\,\mPi_{j-k}\quad(j=1,\ldots,p),\quad\mPi_0=-\mI,
\end{aligned}
\end{equation}
followed by the recursion~\eqref{eq:mPirec2} to extend $\mPi_{1:p}$ to $\mPi_{1:m}$. Because $\mathcal{T}$ is a smooth bijection of $\RR^{(p+q)d^2}$ onto the open strictly-stable set (Section~\ref{sec:VARMAintro}), the spectral-radius constraints of Section~\ref{sec:spectral_constraints} are enforced \emph{by construction} and no projection step is required during L-BFGS optimisation. The objectives $L_m$, $\mathcal{L}_m^{\mathrm{RLS}}$, and the marginalised MAP loss of Section~\ref{sec:MAP} are all minimised as unconstrained functions of $(\mM^{\mPhi}_{1:p},\mM^{\mTheta}_{1:q})\in\RR^{(p+q)d^2}$ via the composite map of~\eqref{eq:M_to_Pi_Theta}.

\subsection{Efficient Least Squares loss function of VAR}\label{sec:LSloss}

Least squares loss functions have been widely used in time series forecasting, and their consistency for estimating VAR process parameters under fourth moment noise assumptions (i.e. not necessarily Gaussian noise) has been established in \cite{mann1943statistical}. In the setting of VAR processes with Gaussian noise, the least squares estimator is known to be equivalent to the MLE (see Sections 3.3 - 3.4 of \cite{lutkepohl2005new}).

We consider the one-step-ahead squared prediction error of our truncated VARMA 
\begin{equation}
\label{eq:lmdef}
L_{m}(\mPi_{1:p},\mTheta_{1:q})=\sum_{t=m}^{T-1}\l\|\sum_{j=1}^m \mPi_j \vz_{t+1-j}-\vz_{t+1}\r\|^2,
\end{equation}
where  $\mPi_{1:m}$ are the first $m$ coefficients of the infinite VAR representation \eqref{eq:infiniteVAR} of the VARMA process, these can be computed recursively based on the parameters $\mPi_{1:p}, \mTheta_{1:q}$  by \eqref{eq:mPirec2} for $j\ge p+1$. Note that the parameters $\mPi_{1:p}, \mTheta_{1:q}$ allow us to recover the original parameter $\mPhi_{1:p}$ by using \eqref{eq:mPirec} for $1\le j\le p$.

This loss function attempts to find the parameters of a VARMA process whose $m$-truncated infinite VAR representation has the smallest 1-step prediction error (in MSE). Note that unlike $\ell_{\infty}$, $\ell_m$ does not depend on the noise covariance parameter $\Sigma$ (as this does not affect the coefficients of the infinite VAR representation).

\subsection{Evaluation of the loss function}\label{sec:eval_loss}
First, we introduce some necessary notation. Let 
\begin{align*}
\vz_{t:t-m+1}&=\l(\begin{matrix} \vz_{t}\\ \vdots\\ \vz_{t-m+1}\end{matrix}\r), \quad \mPi_{1:m}=\l(\begin{matrix}\mPi_1\\ \vdots \\ \mPi_m\end{matrix}\r), \quad \oPi_{1:m}=\l(\begin{matrix}\mPi_1'\\ \vdots \\ \mPi_m'\end{matrix}\r).
\end{align*}
Here $\vz_{t:t-m+1}$ is a column vector of dimension $md$, while $\mPi_{1:m}$ and $\oPi_{1:m}$ are block matrices with $md$ rows and $d$ columns ($m$ blocks of size $d\times d$), with $\oPi$ corresponding to block-wise transpose. Using these, and expanding the square with the cyclic property of the trace, the loss can be written as
\begin{align}\label{eq:LmPiTheta:def}
&  L_m(\mPi_{1:m})=\mS_0-2 \tr(\oPi_{1:m}' \mS_1)+\tr( \oPi_{1:m}' \mS_2 \oPi_{1:m})
\intertext{for}
\nonumber&\mS_0:=\sum_{t=m}^{T-1}\|\vz_{t+1}\|^2, \quad \mS_1:=\sum_{t=m}^{T-1} 
\vz_{t:t-m+1}\vz_{t+1}',  \quad \mS_2:=\sum_{t=m}^{T-1}\vz_{t:t-m+1}\vz_{t:t-m+1}'.
\end{align}
After we have computed $\mS_0$, $\mS_1$ and $\mS_2$, further evaluation of the loss function no longer requires access to the observations, i.e. the cost no longer depends on the number of observations $T$. This is in contrast with the original log-likelihood function, where we need to access all observations for each likelihood evaluation.

Note that $L_m$ is a function of the truncated AR coefficients $\mPi_{1:m}$ alone (it does not take $\mTheta$ as a separate argument). For a VARMA$(p,q)$ the $\mPi_{1:m}$ are the truncation of its AR($\infty$) expansion, produced from $(\mPhi_{1:p},\mTheta_{1:q})$ by the recursion~\eqref{eq:M_to_Pi_Theta}; in the unconstrained PAR coordinates this composes as $(\mPhi_{1:p},\mTheta_{1:q})=\mathcal{T}(\mM^{\mPhi}_{1:p},\mM^{\mTheta}_{1:q})$ followed by $\mPi_{1:m}=\mPi_{1:m}(\mPhi,\mTheta)$, so the objective actually minimised is $L_m\bigl(\mPi_{1:m}(\mathcal{T}(\mM^{\mPhi}_{1:p},\mM^{\mTheta}_{1:q}))\bigr)$, which we abbreviate $L_m(\mathcal{T}(\mM^{\mPhi},\mM^{\mTheta}))$ in the sequel (and likewise $\mR(\mathcal{T}(\cdot))$ for the residual scatter).

\subsection{Regularized Least Squares}\label{sec:RLS}

The least-squares loss $L_m$ of Section~\ref{sec:LSloss} has no explicit regularization on the coefficient matrices $\mPi_{1:p}$ and $\mTheta_{1:q}$. When the number of free parameters $pd^2+qd^2$ is of the same order as the effective sample size $N:=T-m$, the unregularized optimum can become highly unstable. In this subsection we describe a lightweight Bayesian regularization that preserves the sufficient-statistics computational structure of $L_m$ while introducing data-adaptive shrinkage on the coefficient matrices. We call the resulting estimator the \emph{Regularized Least Squares} (RLS) estimator.

\paragraph{Gaussian likelihood with isotropic noise.}
We adopt the simplest Gaussian model in which the innovation covariance is held fixed at the identity, $\mSigma=\mI_d$. Under this assumption the conditional Gaussian negative log-likelihood of the truncated VAR$(m)$ working model (up to a constant that does not depend on the parameters) reduces to
\begin{equation}\label{eq:rls_lik}
-\log p\!\bigl(\vz_{m+1:T}\mid \mPi_{1:p},\mTheta_{1:q},\mSigma=\mI_d\bigr)
=\tfrac{1}{2}\, L_m(\mPi_{1:m})
+\mathrm{const},
\end{equation}
i.e., half the quadratic criterion~\eqref{eq:LmPiTheta:def}. In particular, the likelihood depends on the observations only through the sufficient statistics $\mS_0$, $\mS_1$, $\mS_2$ of Section~\ref{sec:LSloss}, so each evaluation is $O(m^2 d^3)$, independent of $T$.

\paragraph{Gaussian shrinkage prior on the PAR parameters.}
Following the approach of \citet{heaps2023enforcing}, we place a centred Gaussian shrinkage prior \emph{directly on the unconstrained PAR parameters} $(\mM^{\mPhi}_{1:p},\mM^{\mTheta}_{1:q})$, partitioning the entries into four groups (diagonal vs off-diagonal, AR vs MA side) and assigning each group a single scalar variance. For every $j=1,\ldots,p$, $k=1,\ldots,q$, and $1\le a,b\le d$ with $a\ne b$:
\begin{align*}
(\mM^{\mPhi}_j)_{aa} \mid \sigma^2_{\mPhi,\mathrm{diag}}\ &\overset{\mathrm{iid}}{\sim}\ N\bigl(0,\,\sigma^2_{\mPhi,\mathrm{diag}}\bigr),
&(\mM^{\mPhi}_j)_{ab} \mid \sigma^2_{\mPhi,\mathrm{off}}\ &\overset{\mathrm{iid}}{\sim}\ N\bigl(0,\,\sigma^2_{\mPhi,\mathrm{off}}\bigr),\\
(\mM^{\mTheta}_k)_{aa} \mid \sigma^2_{\mTheta,\mathrm{diag}}\ &\overset{\mathrm{iid}}{\sim}\ N\bigl(0,\,\sigma^2_{\mTheta,\mathrm{diag}}\bigr),
&(\mM^{\mTheta}_k)_{ab} \mid \sigma^2_{\mTheta,\mathrm{off}}\ &\overset{\mathrm{iid}}{\sim}\ N\bigl(0,\,\sigma^2_{\mTheta,\mathrm{off}}\bigr).
\end{align*}
The four scalar variances
\(
\sigma^2_{\mPhi,\mathrm{diag}},\ \sigma^2_{\mPhi,\mathrm{off}},\ \sigma^2_{\mTheta,\mathrm{diag}},\ \sigma^2_{\mTheta,\mathrm{off}}
\)
control the amount of shrinkage applied to each group. Because $\mathcal{T}$ is locally the identity at $\mM=\mz$ (the contraction map of Eq.~\eqref{eq:contraction} satisfies $\mathrm{d}\mP_s/\mathrm{d}\mM_s=\mI$ there), the diagonal/off-diagonal split on $\mM$ closely matches the corresponding split on the original $(\mPhi,\mTheta)$: shrinking diagonal $\mM$-entries less aggressively allows the within-channel partial autocorrelations to remain near their data-driven values, while off-diagonal entries (cross-channel partial effects) are penalised more heavily. Separating AR from MA variances accommodates the different roles the two parts play in the Wold representation. Pushed forward by $\mathcal{T}$, this Gaussian prior on $\mM$ induces a smooth, fully supported prior on the strictly stable VARMA coefficients $(\mPhi_{1:p},\mTheta_{1:q})$ that increasingly penalizes approaches to the stability boundary in the unconstrained coordinates \citep{heaps2023enforcing}. The result is a soft, stationarity-encouraging prior with no truncation discontinuity.

The four group variances are treated as \emph{fixed} hyperparameters. Rather than hand-tune them, we select them once per data set by cross-validation: a dyadic grid search chooses the variances minimizing an out-of-sample predictive log-likelihood on a held-out tail of the first training window. We defer the details of this selection procedure to the numerical experiments (Section~\ref{sec:exp_protocol}), and for now simply treat
\(
\sigma^2_{\mPhi,\mathrm{diag}},\ \sigma^2_{\mPhi,\mathrm{off}},\ \sigma^2_{\mTheta,\mathrm{diag}},\ \sigma^2_{\mTheta,\mathrm{off}}
\)
as given constants.

\paragraph{A separate, fixed ridge on the first-lag own-channel diagonals.}
The within-channel first-order dynamics (the diagonal entries $(\mM^{\mPhi}_1)_{aa}$ and $(\mM^{\mTheta}_1)_{aa}$ of the \emph{lag-one} AR and MA matrices) carry the dominant persistence of each series and are typically large and well identified. By contrast, the higher-lag own-channel diagonals $(\mM^{\mPhi}_j)_{aa}$ and $(\mM^{\mTheta}_k)_{aa}$ with $j,k\ge 2$ decay quickly and are mostly noise. Folding all lags into a single diagonal variance is therefore undesirable: the many small higher-lag entries pull the shared variance down and \emph{over-shrink} the lag-one persistence, which is exactly the signal we wish to preserve. We therefore \emph{exclude the lag-one own-channel diagonals from the cross-validated diagonal groups} and give them their own centred Gaussian prior with a separate \emph{fixed} variance $\sigma^2_{\mathrm{diag},1}$,
\begin{equation}\label{eq:lag1_diag_prior}
(\mM^{\mPhi}_1)_{aa}\ \overset{\mathrm{iid}}{\sim}\ N\!\bigl(0,\sigma^2_{\mathrm{diag},1}\bigr),
\qquad
(\mM^{\mTheta}_1)_{aa}\ \overset{\mathrm{iid}}{\sim}\ N\!\bigl(0,\sigma^2_{\mathrm{diag},1}\bigr),
\qquad a=1,\ldots,d,
\end{equation}
chosen large ($\sigma^2_{\mathrm{diag},1}=4$ in our experiments) so that these coefficients are only lightly ridged and stay close to their data-driven values, while the off-diagonal and higher-lag diagonal groups can be shrunk more aggressively. Keeping $\sigma^2_{\mathrm{diag},1}$ separate from the cross-validated groups is deliberate: a single variance shared with the near-zero higher-lag diagonals would couple the lag-one persistence to them, defeating the purpose. With this split, the cross-validated diagonal groups below run over the remaining lags only ($j=2,\ldots,p$ and $k=2,\ldots,q$), so their entry counts become $(p-1)\,d$ and $(q-1)\,d$ (and vanish entirely for first-order models $p=q=1$); the off-diagonal groups are unchanged.

\paragraph{The regularized least-squares loss.}
Writing $n_{\mathrm{diag}}:=d$ and $n_{\mathrm{off}}:=d(d-1)$, and collecting the four cross-validated group sums of squares (the diagonal sums running over lags $\ge 2$) together with the two lag-one diagonal sums held out for the fixed ridge,
\begin{align*}
S_{\mPhi}^{\mathrm{diag}} &:=\sum_{j=2}^p \sum_{a=1}^{d}(\mM^{\mPhi}_j)_{aa}^2,
& S_{\mPhi}^{\mathrm{off}}\phantom{x}  &:=\sum_{j=1}^p \sum_{a\ne b}(\mM^{\mPhi}_j)_{ab}^2,
& S_{\mPhi}^{\mathrm{diag},1} &:=\sum_{a=1}^{d}(\mM^{\mPhi}_1)_{aa}^2,\\
S_{\mTheta}^{\mathrm{diag}} &:=\sum_{k=2}^q \sum_{a=1}^{d}(\mM^{\mTheta}_k)_{aa}^2,
& S_{\mTheta}^{\mathrm{off}}\phantom{x}  &:=\sum_{k=1}^q \sum_{a\ne b}(\mM^{\mTheta}_k)_{ab}^2,
& S_{\mTheta}^{\mathrm{diag},1} &:=\sum_{a=1}^{d}(\mM^{\mTheta}_1)_{aa}^2,
\end{align*}
the negative log of the Gaussian shrinkage prior is, up to additive constants, a sum of quadratic group-ridge penalties $\tfrac{1}{2}\,S_g/\sigma^2_g$. Adding these to the half-quadratic data term $\tfrac{1}{2}L_m$ gives the \emph{regularized least-squares (RLS) loss}
\begin{equation}\label{eq:rls_loss}
\boxed{
\begin{aligned}
\mathcal{L}_m^{\mathrm{RLS}}(\mM^{\mPhi}_{1:p},\mM^{\mTheta}_{1:q})
&\,=\, \tfrac{1}{2}\, L_m\!\bigl(\mathcal{T}(\mM^{\mPhi}_{1:p},\mM^{\mTheta}_{1:q})\bigr)
+\frac{S_{\mPhi}^{\mathrm{diag},1}+S_{\mTheta}^{\mathrm{diag},1}}{2\,\sigma^2_{\mathrm{diag},1}}\\
&\quad
+\frac{S_{\mPhi}^{\mathrm{diag}}}{2\,\sigma^2_{\mPhi,\mathrm{diag}}}
+\frac{S_{\mPhi}^{\mathrm{off}}}{2\,\sigma^2_{\mPhi,\mathrm{off}}}
+\frac{S_{\mTheta}^{\mathrm{diag}}}{2\,\sigma^2_{\mTheta,\mathrm{diag}}}
+\frac{S_{\mTheta}^{\mathrm{off}}}{2\,\sigma^2_{\mTheta,\mathrm{off}}}.
\end{aligned}
}
\end{equation}
The data-fit term $L_m\!\bigl(\mathcal{T}(\mM^{\mPhi}_{1:p},\mM^{\mTheta}_{1:q})\bigr)$ uses Eq.~\eqref{eq:M_to_Pi_Theta} to obtain $\mPi_{1:m}$ from $(\mM^{\mPhi}_{1:p},\mM^{\mTheta}_{1:q})$ and depends on the data only through the sufficient statistics $\mS_0,\mS_1,\mS_2$; the five ridge penalties depend on no data. The per-group variances are the fixed, cross-validated hyperparameters introduced above (with the lag-one own-channel diagonals held at $\sigma^2_{\mathrm{diag},1}$). The objective is a smooth function of the unconstrained $(\mM^{\mPhi}_{1:p},\mM^{\mTheta}_{1:q})\in\RR^{(p+q)d^2}$, minimized by L-BFGS without any spectral-radius constraint.

\paragraph{Properties.}
\begin{itemize}
\item \emph{Unconstrained parameter space.} The state vector for L-BFGS is $(\mM^{\mPhi}_{1:p},\mM^{\mTheta}_{1:q})\in\RR^{(p+q)d^2}$, which is unconstrained (no spectral-radius constraints, no projections, no augmented-Lagrangian penalties). The group variances are fixed hyperparameters chosen by cross-validation, not free parameters, so the optimization is over the coefficients alone, avoiding both the extra dimensions and the parameterization sensitivity of a joint MAP that retains $\log\sigma^2$ as variables.

\item \emph{Sufficient-statistics evaluation.} The likelihood term $\tfrac{1}{2}L_m\!\bigl(\mathcal{T}(\mM^{\mPhi}_{1:p},\mM^{\mTheta}_{1:q})\bigr)$ depends on the data only through $\mS_0,\mS_1,\mS_2$ (cf.~Section~\ref{sec:LSloss}); the ridge-prior terms depend on no data. The PAR map $\mathcal{T}$ adds an $O((p^2{+}q^2)\,d^3)$ overhead per evaluation (a Durbin--Levinson recursion, $O(k^2d^3)$ per factor of order $k$). Each value-and-gradient evaluation of $\mathcal{L}_m^{\mathrm{RLS}}$ therefore costs $O(m^2 d^3 + (p^2{+}q^2)\,d^3)$, independent of $T$.

\item \emph{Relation to ridge regression.} The penalty in~\eqref{eq:rls_loss} is a group-wise Tikhonov (ridge) regularizer on the PAR coordinates. Specialising to $\mTheta_{1:q}\equiv\mz$ and a regime where $\mathcal{T}$ is well-approximated by the identity (small $\mM^{\mPhi}$, i.e.\ near-white-noise) with a common variance $\sigma^2$ across groups, it reduces to
\(
\tfrac{1}{2\sigma^2}\|\mM^{\mPhi}_{1:p}\|_F^2 + \mathrm{const}\approx \tfrac{1}{2\sigma^2}\|\mPhi_{1:p}\|_F^2,
\)
i.e.\ Tikhonov-regularized (ridge) regression on the truncated VAR representation.
\end{itemize}

Compared to the $\mSigma$-marginalized MAP estimator of Section~\ref{sec:MAP}, the RLS estimator is the cheaper of the two: it avoids the $O(d^3)$ Cholesky and log-determinant of $\mPsi_0+\mR$ at every iteration, and its likelihood is simply the quadratic $L_m$ already computed from the sufficient statistics. The price paid is that $\mSigma$ is assumed isotropic, an assumption that is accurate only after the observations $\vz_t$ have been approximately whitened, e.g., by a rolling standardization step as in our applications.

\section{MAP estimation with marginalized covariance}\label{sec:MAP}

The least-squares loss $L_m$ from the previous section treats all residual directions equally and does not model the noise covariance $\mSigma$. In this section, we extend the approach to a full Bayesian framework in which $\mSigma$ is given an inverse-Wishart prior and then marginalized out analytically. The result is a \emph{marginalized MAP} loss that (i) still depends on the data only through precomputed sufficient statistics, (ii) adapts to the noise structure without requiring explicit estimation of $\mSigma$, and (iii) stays numerically well-behaved (bounded below even when the residual scatter is rank-deficient) thanks to the positive-definite inverse-Wishart scale $\mPsi_0$. Coefficient shrinkage is supplied \emph{separately}, by the PAR-coordinate priors of Section~\ref{sec:RLS}, not by the covariance marginalization.

\subsection{Gaussian likelihood and residual scatter matrix}

Let $N:=T-m$ and define the residual vectors
\[
\vr_t:=\vz_{t+1}-\sum_{j=1}^m \mPi_j \vz_{t+1-j},
\qquad t=m,\dots,T-1.
\]
Conditionally on $\mSigma$, the conditional Gaussian negative log-likelihood of the truncated VAR$(m)$ working model, given the first $m$ observations (up to a constant), is
\begin{equation}\label{eq:nll_Sigma}
-\log p(\vz_{m+1:T}\mid \mPi_{1:p},\mTheta_{1:q},\mSigma)
=
\frac{N}{2}\log\det(\mSigma)
+\frac{1}{2}\tr\!\bigl(\mSigma^{-1}\mR\bigr),
\end{equation}
where $\mR\in\RR^{d\times d}$ is the \emph{residual scatter matrix}
\begin{equation}\label{eq:R_def}
\mR:=\sum_{t=m}^{T-1}\vr_t\vr_t'.
\end{equation}
Crucially, $\mR$ can be expressed in terms of the precomputed sufficient statistics from Section~\ref{sec:LSloss}. Let
\[
\mS_0^{(2)}:=\sum_{t=m}^{T-1}\vz_{t+1}\vz_{t+1}' \in\RR^{d\times d}
\]
be the matrix-valued analogue of the scalar statistic $\mS_0$. Then expanding the definition of $\vr_t$ gives
\begin{align}
\mR
&=
\mS_0^{(2)}
-\sum_{j=1}^m \mPi_j \mS_1(j)
-\sum_{j=1}^m \mS_1(j)' \mPi_j'
+\sum_{i=1}^m \sum_{j=1}^m \mPi_i \mS_2(ij)\mPi_j'.
\label{eq:R_residual_cov}
\end{align}
Thus $\mR$, and hence the likelihood \eqref{eq:nll_Sigma}, depends on the observations only through $\mS_0^{(2)}$, $\mS_1$, and $\mS_2$.

\subsection{Inverse-Wishart prior on \texorpdfstring{$\mSigma$}{Sigma} and marginalization}

We place a conjugate inverse-Wishart prior on the noise covariance:
\[
\mSigma \sim \mathrm{IW}(\mPsi_0,\, \nu_0),
\]
with scale matrix $\mPsi_0\in\RR^{d\times d}$ and degrees of freedom $\nu_0$. The prior is proper precisely when $\mPsi_0\succ\mz$ and $\nu_0>d-1$; the improper (Jeffreys-like) limit $\mPsi_0=\mz$, $\nu_0=0$ is also admissible and is discussed below. In our implementation $\mPsi_0 = \psi_0 \mI_d$ for a scalar $\psi_0\ge 0$, and throughout this paper we use the default $\psi_0=1$, $\nu_0=d+4$, i.e.\ $\mSigma\sim\mathrm{IW}(\mI_d,\,d+4)$, which is proper ($\mI_d\succ\mz$ and $d+4>d-1$). The choice $\nu_0=d+4$ is the smallest \emph{integer} degrees-of-freedom value for which the inverse-Wishart prior on $\mSigma$ has finite second moments (the inverse-Wishart variance exists only for $\nu_0>d+3$), so the prior is only mildly informative while still letting the data dominate. The density of the inverse-Wishart distribution is proportional to
\[
|\mSigma|^{-(\nu_0+d+1)/2}\exp\!\Bigl(-\tfrac{1}{2}\tr(\mPsi_0\mSigma^{-1})\Bigr).
\]
Combining with the Gaussian likelihood \eqref{eq:nll_Sigma} and integrating over $\mSigma$, we obtain the marginal likelihood
\begin{align*}
p(\vz_{m+1:T}\mid \mPi_{1:p},\mTheta_{1:q})
&= \int p(\vz_{m+1:T}\mid \mPi_{1:p},\mTheta_{1:q},\mSigma)\,p(\mSigma)\,\mathrm{d}\mSigma \\
&\propto |\mPsi_0 + \mR|^{-(\nu_0+N)/2}.
\end{align*}
This is a standard result from the conjugacy of the inverse-Wishart with the multivariate Gaussian (see, e.g., Chapter~3 of \citealp{gelman2013bayesian}). The negative log marginal likelihood is therefore
\begin{equation}\label{eq:marginal_loglik}
-\log p(\vz_{m+1:T}\mid \mPi_{1:p},\mTheta_{1:q})
= \frac{\nu_0+N}{2}\,\log\det(\mPsi_0 + \mR) + \mathrm{const.}
\end{equation}

\begin{remark}[Default prior and improper limit]
Throughout, the default is the proper prior $\mSigma\sim\mathrm{IW}(\mtx{I}_d,\,d+4)$, which adds the unit ridge $\mtx{I}_d$ to $\mR$ inside the log-determinant in~\eqref{eq:marginal_loglik}. Setting $\nu_0=0$, $\mPsi_0=\mtx{0}$ recovers the improper (Jeffreys-like) prior $p(\mSigma)\propto|\mSigma|^{-(d+1)/2}$, under which the marginal likelihood reduces to $\tfrac{N}{2}\log\det\mR+\mathrm{const}$; we do not use this limit, as the $\mR$-only criterion can degenerate when $\mR$ is singular.
\end{remark}

\subsection{Global coefficient regularization}\label{sec:JMAP_reg}

We place the same Gaussian shrinkage prior on the unconstrained PAR parameters as for RLS (Section~\ref{sec:RLS}): the fixed lag-one own-channel diagonal ridge~\eqref{eq:lag1_diag_prior}, and the four cross-validated groups with the diagonal groups running over lags $\ge2$, inducing via push-forward through $\mathcal{T}$ a smooth, stability-respecting prior on the VARMA coefficients \citep{heaps2023enforcing}. In terms of the group sums of squares of Section~\ref{sec:RLS}, the negative log-prior is, up to constants,
\begin{align}\label{eq:reg_Pi}
\mathrm{Reg}_{\mPhi}
&= \frac{S_{\mPhi}^{\mathrm{diag}}}{2\,\sigma^2_{\mPhi,\mathrm{diag}}}
+ \frac{S_{\mPhi}^{\mathrm{off}}}{2\,\sigma^2_{\mPhi,\mathrm{off}}}
+ \frac{S_{\mPhi}^{\mathrm{diag},1}}{2\,\sigma^2_{\mathrm{diag},1}},\\
\mathrm{Reg}_{\mTheta}
&= \frac{S_{\mTheta}^{\mathrm{diag}}}{2\,\sigma^2_{\mTheta,\mathrm{diag}}}
+ \frac{S_{\mTheta}^{\mathrm{off}}}{2\,\sigma^2_{\mTheta,\mathrm{off}}}
+ \frac{S_{\mTheta}^{\mathrm{diag},1}}{2\,\sigma^2_{\mathrm{diag},1}},\label{eq:reg_Theta}
\end{align}
identical to the coefficient penalties in the RLS loss~\eqref{eq:rls_loss}. This global (lag-pooled) parameterisation uses one variance per group rather than one per lag, which empirically improves stability when $p$ and $q$ are large relative to $N$, at the cost of preventing lag-specific shrinkage.

\subsection{The $\Sigma$-marginalized MAP loss}

Combining the marginalized likelihood \eqref{eq:marginal_loglik} with the M-prior regularization \eqref{eq:reg_Pi}--\eqref{eq:reg_Theta}, the full \emph{marginalized MAP loss} is
\begin{equation}\label{eq:marginalized_MAP}
\boxed{
\mathcal{L}_m^{\mathrm{marg}}(\mM^{\mPhi}_{1:p},\mM^{\mTheta}_{1:q})
= \frac{\nu_0+N}{2}\,\log\det\!\bigl(\mPsi_0 + \mR(\mathcal{T}(\mM^{\mPhi}_{1:p},\mM^{\mTheta}_{1:q}))\bigr)
+ \mathrm{Reg}_{\mPhi} + \mathrm{Reg}_{\mTheta},
}
\end{equation}
where $\mR$ is the residual scatter matrix \eqref{eq:R_residual_cov} evaluated at $(\mPi_{1:m},\mTheta_{1:q})$ obtained from $(\mM^{\mPhi}_{1:p},\mM^{\mTheta}_{1:q})$ via the composite map of~\eqref{eq:M_to_Pi_Theta}, and $N=T-m$. The objective is minimised as an unconstrained function of $(\mM^{\mPhi}_{1:p},\mM^{\mTheta}_{1:q})\in\RR^{(p+q)d^2}$; the PAR map $\mathcal{T}$ enforces $\spr(\mathcal{C}_{\mPhi})<1$ and $\spr(\mathcal{C}_{\mTheta})<1$ by construction, so no projection or augmented-Lagrangian step is needed.

This loss has several desirable properties:
\begin{enumerate}
\item \textbf{Covariance-free.} The noise covariance $\mSigma$ has been marginalized out analytically and does not appear as a parameter. This reduces the parameter count by $d(d+1)/2$ compared to approaches that optimize over $\mSigma$ directly.
\item \textbf{Sufficient-statistics form.} As with RLS, the loss depends on the data only through $\mS_0^{(2)}$, $\mS_1$, $\mS_2$; after the one-time pre-pass, each evaluation costs $O(m^2d^3+(p^2{+}q^2)d^3)$, independent of $T$.
\item \textbf{Numerical stability from $\mPsi_0$.} The positive-definite scale $\mPsi_0$ keeps the data term bounded below: since $\mPsi_0+\mR\succeq\mPsi_0\succ\mz$, we have $\log\det(\mPsi_0+\mR)\ge\log\det\mPsi_0>-\infty$, so even a rank-deficient residual scatter $\mR$ (as when the truncated model interpolates the in-sample residuals) cannot drive the criterion to $-\infty$. The inverse-Wishart prior therefore regularizes the \emph{covariance} estimate, shrinking it toward $\mPsi_0$; shrinkage of the \emph{coefficient} matrices $\mPi_{1:p},\mTheta_{1:q}$ comes separately, from the PAR-coordinate priors of Section~\ref{sec:RLS}.
\item \textbf{Unconstrained optimisation.} Stationarity and invertibility are guaranteed by $\mathcal{T}$, with no constraints on the state vector.
\end{enumerate}

\begin{remark}[Recovering $\hat{\mSigma}$]
Although $\mSigma$ is not optimized over, a point estimate can be recovered after fitting. The posterior mode of $\mSigma$ under the inverse-Wishart prior, given the MAP estimates $(\hat{\mM}^{\mPhi}_{1:p},\hat{\mM}^{\mTheta}_{1:q})$ and the corresponding $\hat{\mR}=\mR\!\bigl(\mathcal{T}(\hat{\mM}^{\mPhi}_{1:p},\hat{\mM}^{\mTheta}_{1:q})\bigr)$, is
$\hat{\mSigma} = (\mPsi_0 + \hat{\mR}) / (\nu_0 + N + d + 1).$
\end{remark}

\paragraph{Optimisation by L-BFGS.}
The RLS loss~\eqref{eq:rls_loss} and the marginalised MAP loss~\eqref{eq:marginalized_MAP} are both smooth functions of $(\mM^{\mPhi}_{1:p},\mM^{\mTheta}_{1:q})\in\RR^{(p+q)d^2}$, composed of matrix multiplications, FFTs, a Cholesky log-determinant, and the fixed-variance ridge penalties; a value-and-gradient pair is obtained in one reverse-mode autodiff pass \citep{jax2018github}, storing no recursion intermediates. The PAR map $\mathcal{T}$ enforces $\spr(\mathcal{C}_{\mPhi})<1$ and $\spr(\mathcal{C}_{\mTheta})<1$ by construction, so both losses are minimised as unconstrained functions on $\RR^{(p+q)d^2}$. Because the parameter dimension $(p+q)d^2$ grows quadratically with $d$ and easily reaches the thousands, we use the limited-memory quasi-Newton method L-BFGS \citep{liu1989limited}, which builds an implicit positive-definite Hessian approximation from the last few value-gradient pairs and produces each search direction in time linear in the parameter dimension. L-BFGS also warm-starts naturally across rolling-window refits, the previous window's optimum serving as the initial point when the window shifts by small steps (Appendix~\ref{sec:online}); the experiments of Section~\ref{sec:Experiments} instead refit each window from a pilot estimate (Section~\ref{sec:exp_protocol}). We optimize this nonconvex objective numerically and do not claim convergence to the global optimum. The recovery guarantee of Section~\ref{sec:consistency} is stated for a global empirical minimizer over a compact ball (and, more generally, for any point whose objective is near the minimum, Theorem~\ref{thm:pi_rate} and Remark~\ref{rem:compactness}); it therefore bounds the statistical error of a well-optimized fit but does not certify that the L-BFGS iterate reaches it.

\subsection{Multiplicative seasonal VARMA}\label{sec:seasonal_varma}
Many series, hourly or daily measurements in particular, carry a strong \emph{periodic} component: even after a deterministic seasonal mean is removed, the anomalies retain a season-recurring autocorrelation that an ordinary VAR or dense VARMA can only mimic with many extra lags. The multivariate analogue of the multiplicative seasonal ARMA of \citet[Ch.~9]{box2015time}, the seasonal VARMA$(p,q)\times(P,Q)_s$, captures this parsimoniously by multiplying a regular and a seasonal lag polynomial on each side:
\begin{equation}\label{eq:sarma}
\Bigl(\mI-\sum_{j=1}^{P}\mPhi^{(s)}_j B^{sj}\Bigr)\Bigl(\mI-\sum_{i=1}^{p}\mPhi_i B^{i}\Bigr)\vZ_t
=\Bigl(\mI-\sum_{k=1}^{Q}\mTheta^{(s)}_k B^{sk}\Bigr)\Bigl(\mI-\sum_{l=1}^{q}\mTheta_l B^{l}\Bigr)\va_t,
\end{equation}
where $B$ is the backshift operator ($B\vZ_t=\vZ_{t-1}$) and $s$ is the seasonal period. Expanding the two products gives an ordinary VARMA whose AR and MA polynomials have orders $p+sP$ and $q+sQ$, so~\eqref{eq:sarma} is a \emph{restricted} VARMA: of the $p+sP$ AR lags only those at the regular lags $1{:}p$, the seasonal lags $s,2s,\dots,sP$, and the interaction lags $sr+j$ ($1\le r\le P$, $1\le j\le p$; possibly overlapping the regular and seasonal lags) carry nonzero coefficients, and these are determined by just $(p+P)d^2$ free parameters (and likewise $(q+Q)d^2$ on the MA side) rather than the $(p+sP+q+sQ)d^2$ of the unrestricted expansion. A single seasonal factor at lag $s$ thus encodes a full period of dynamics that an unrestricted VAR would need $\ge s$ lags and $s\,d^2$ parameters to represent.

The model~\eqref{eq:sarma} plugs directly into the truncated-VARMA machinery of Sections~\ref{sec:truncated_varma}--\ref{sec:MAP}: the four factor polynomials are convolved once into the expanded $(\mPhi_{1:p+sP},\mTheta_{1:q+sQ})$, the infinite-VAR coefficients $\mPi_j$ that solve the Box--Jenkins relation~\eqref{eq:mPirec}--\eqref{eq:mPirec2} are formed directly from these factors by a single FFT inversion, truncated at $j=m$, and fed to the same sufficient-statistic loss. Stationarity and invertibility are again enforced coordinate-free: each of the four factors is parameterized in its \emph{own} unconstrained PAR coordinates through the map $\tau$ of Section~\ref{sec:VARMAintro}, so writing $(\mM^{\mPhi},\mM^{\mPhi^{(s)}},\mM^{\mTheta},\mM^{\mTheta^{(s)}})$ for the four unconstrained stacks, the maps $\tau_p,\tau_P,\tau_q,\tau_Q$ make each factor's block-companion matrix Schur stable, so each factor polynomial has no determinantal zeros in the closed unit disk. The determinant of the expanded AR (or MA) operator is the product of its factors' determinants, evaluating the seasonal factor at $z^{s}$ (with $|z^{s}|\le1$ whenever $|z|\le1$), hence also nonvanishing there, so~\eqref{eq:sarma} is stationary and invertible by construction, and the $\Sigma$-marginalized MAP loss~\eqref{eq:marginalized_MAP} (or, with $\mSigma=\mI$, the RLS criterion of Section~\ref{sec:RLS}) is minimized as an unconstrained function of the $(p+P+q+Q)d^2$ factor parameters. We apply this seasonal VARMA in the meteorological and air-quality benchmarks of Sections~\ref{sec:exp_singapore} and~\ref{sec:exp_beijing}.

\paragraph{Relation to the standard multiplicative seasonal VARMA.}
Equation~\eqref{eq:sarma} is the multiplicative seasonal VARMA of \citet{lutkepohl2005new,tsay2013multivariate} (see also \citealp{reinsel2003elements}), written in the convention in which the seasonal factor \emph{left}-multiplies the regular factor: in the scalar case the ordering of the factors is immaterial, but matrix coefficients need not commute, so for $d>1$ the convention fixes the restricted family, and we use the stated one throughout. Our treatment changes how the model is parameterized and fit, not the model class. The PAR map acts per factor, and the four factor polynomials are then multiplied in coefficient space to form the expanded $(\mPhi_{1:p+sP},\mTheta_{1:q+sQ})$; since $\det\!\bigl[\mPhi(z)\,\mPhi^{(s)}(z^{s})\bigr]=\det\mPhi(z)\,\det\mPhi^{(s)}(z^{s})$ (likewise on the MA side), the expanded operator is stable and invertible exactly when each factor is, so parameterizing each factor over its own stable region via $\tau$ traverses precisely this restricted family. The departures from the textbook treatment are in estimation only: the per-factor PAR-shrinkage and inverse-Wishart priors (Sections~\ref{sec:RLS}--\ref{sec:JMAP_reg}), unconstrained optimization in the PAR coordinates, and the truncated infinite-VAR loss, whose approximation improves as $m\to\infty$.

\paragraph{Computational cost.} The per-gradient cost of fitting~\eqref{eq:sarma} (independent of $T$ and, at leading order, identical to that of a plain VARMA of the same dimension) is analyzed in detail in Appendix~\ref{sec:cost_appendix}, building on the Fourier evaluation developed in Section~\ref{sec:complexity_S2}.

\paragraph{Seasonal VARMAX.}
The exogenous-regressor (VARMAX) extension composes with the seasonal factorization, giving a \emph{seasonal VARMAX}: the right-hand side of~\eqref{eq:sarma} is augmented with a finite exogenous lag polynomial $\mtx{\Xi}(B)=\sum_{k=0}^{s_x}\mtx{\Xi}_k B^{k}$ acting on observed covariates $\vct{x}_t\in\R^{d_x}$ (a known price or temperature path, say):
\begin{equation*}
\Bigl(\mI-\sum_{j=1}^{P}\mPhi^{(s)}_j B^{sj}\Bigr)\Bigl(\mI-\sum_{i=1}^{p}\mPhi_i B^{i}\Bigr)\vZ_t
=\mtx{\Xi}(B)\,\vct{x}_t+\Bigl(\mI-\sum_{k=1}^{Q}\mTheta^{(s)}_k B^{sk}\Bigr)\Bigl(\mI-\sum_{l=1}^{q}\mTheta_l B^{l}\Bigr)\va_t .
\end{equation*}
Solving for the innovation gives the truncated residual $\va_t=\vZ_t-\sum_{j=1}^{m}\mPi_j\vZ_{t-j}-\sum_{j=0}^{m}\mPi^{\mathrm{x}}_j\vct{x}_{t-j}$, in which $\mPi_{1:m}$ are exactly those of the seasonal VARMA and $\mPi^{\mathrm{x}}$ is the truncation of $\bigl[\mTheta^{(s)}(z^{s})\,\mTheta(z)\bigr]^{-1}\mtx{\Xi}(z)$, computed by the same FFT inversion. The sufficient-statistic structure, the $T$-independent gradients (working dimension $d+d_x$), the incremental rolling updates, and the built-in stationarity and invertibility all carry over, the exogenous block $\mtx{\Xi}$ being unconstrained; Appendix~\ref{sec:VARMAX_appendix} gives the details. The case $P=Q=0$ is the plain VARMAX of the retail benchmark (Section~\ref{sec:exp_dominicks}); the seasonal case suits periodic series driven by known covariates, forecast \emph{conditional} on the covariate path.

\subsection{Statistical guarantees}\label{sec:consistency}

We close the development of the two point estimators with a statistical guarantee. Its target is the sequence of infinite-VAR coefficients $\mPi_{1:\infty}$ of Eq.~\eqref{eq:infiniteVAR}, equivalently, the one-step predictor, or the transfer function $\mTheta(z)^{-1}\mPhi(z)$. Unlike the pair $(\mPhi_{1:p},\mTheta_{1:q})$, this functional is identifiable (Section~\ref{sec:bg_identifiability}): observationally equivalent VARMA parametrisations share the same $\mPi_{1:\infty}$, and it is $\mPi_{1:m}$ alone that enters our losses and forecasts. Accordingly, in the spirit of the classical transfer-function consistency theory for Gaussian VARMA maximum likelihood \citep{dunsmuir1976vector,hannan2012statistical}, we prove recovery of $\mPi_{1:\infty}$ without any identifiability conditions, and (as is standard in that literature) over a compact parameter set, for exact or approximate empirical minimizers.

\begin{assumption}\label{ass:consistency}
\begin{enumerate}
\item[(i)] The observations $\vz_{1:T}$ are a sample of a stationary $d$-dimensional Gaussian VARMA$(\bar p,\bar q)$ process~\eqref{eq:VARMAdef} whose parameters $(\mPhi^0_{1:\bar p},\mTheta^0_{1:\bar q})$ satisfy Assumption~\ref{ass:roots}, with i.i.d.\ innovations $\va_t\sim N(\vct{0},\mSigma)$, $\mSigma\succ\mz$.
\item[(ii)] The fitted orders satisfy $p\ge\bar p$ and $q\ge\bar q$. Writing $n_{\mathrm{par}}:=(p+q)d^2$ and $\vartheta$ for the stacked PAR coordinates $(\mM^{\mPhi}_{1:p},\mM^{\mTheta}_{1:q})\in\RR^{n_{\mathrm{par}}}$, let $\vartheta^0:=\mathcal{T}^{-1}(\mPhi^0_{1:p},\mTheta^0_{1:q})$ be the PAR coordinates of the true coefficients padded with zero matrices up to orders $(p,q)$ (padding leaves the lag polynomials, hence Assumption~\ref{ass:roots}, unchanged and appends only zero eigenvalues to the block companion spectrum, and $\mathcal{T}$ is a bijection onto the stable set by Proposition~\ref{prop:par_bijection} and the padding observation of Appendix~\ref{sec:proof_par}, so $\vartheta^0$ is well defined). Assume $\vartheta^0\in K_B:=\{\vartheta\in\RR^{n_{\mathrm{par}}}:\|\vartheta\|_2\le B\}$ for a fixed radius $B$. Minimizing over $K_B$ needs no projection step in practice: the smooth radial squashing $\sigma_B(\vartheta):=B\vartheta/\sqrt{B^2+\|\vartheta\|_2^2}$, the scalar analogue of the contraction $\mM\mapsto\mM(\mI+\mM^\top\mM)^{-1/2}$ inside $\mathcal{T}$ itself, is a diffeomorphism of $\RR^{n_{\mathrm{par}}}$ onto the open ball, so unconstrained optimization in the squashed coordinates produces $K_B$-valued points whose objective comes arbitrarily close to the minimum over $K_B$; these are the approximate minimizers that Theorem~\ref{thm:pi_rate} admits through its tolerance $\eta$, although a minimizer lying on the boundary sphere itself is not attained at any finite squashed coordinate.
\item[(iii)] $\hat\vartheta_T$ is a measurable $K_B$-valued estimator targeting the RLS loss~\eqref{eq:rls_loss} or the $\Sigma$-marginalized MAP loss~\eqref{eq:marginalized_MAP} (with $\mPsi_0\succ\mz$ and $\nu_0>d-1$ in the latter case, so the inverse-Wishart prior is proper). The hyperparameters $(\mPsi_0,\nu_0)$ are fixed; the group ridge variances may be data-dependent, selected in any measurable way (for instance by validation on a held-out block of the same trajectory, as in Section~\ref{sec:Experiments}), provided each selected variance is at least a fixed floor $\sigma^2_{\min}>0$. Writing $\mathcal{L}$ for the objective at the selected variances, the estimator is required to satisfy, for a (possibly random) $\eta\ge0$,
\[
\mathcal{L}(\hat\vartheta_T)\;\le\;\inf_{\vartheta\in K_B}\mathcal{L}(\vartheta)+\eta ,
\]
and we call the smallest such $\eta$ its optimization error ($\eta=0$ for an exact minimizer over $K_B$).
\end{enumerate}
\end{assumption}

\begin{theorem}[Recovery of $\mPi_{1:\infty}$ for approximate empirical minimizers]\label{thm:pi_rate}
Let Assumption~\ref{ass:consistency} hold, write $\mPi_j(\hat\vartheta_T)$ for the infinite-VAR coefficients~\eqref{eq:mPirec}--\eqref{eq:mPirec2} of the fitted parameters and $\mPi^0_j:=\mPi_j(\vartheta^0)$ for those of the truth, and set
\[
\Lambda_\delta \;:=\; n_{\mathrm{par}}\log T+\log(1/\delta).
\]
There exist constants $\rho_\ast\in(0,1)$, $C\ge1$, $c>0$, $m_1\in\mathbb{N}$ and $T_0$, depending only on $(\mPhi^0,\mTheta^0,\mSigma)$, the orders $(p,q,d)$, the radius $B$ and the fixed hyperparameters $(\mPsi_0,\nu_0)$ and the variance floor $\sigma^2_{\min}$ (but not on $T$, $m$ or $\delta$) such that for every $T\ge T_0$, every $\delta\in(0,1/2)$ with $\Lambda_\delta\le c\,T$, and every truncation lag $m$ with
\[
m_1\;\le\; m\;\le\; T/2,
\]
any estimator satisfying Assumption~\ref{ass:consistency}(iii), for either loss, obeys with probability at least $1-\delta$,
\begin{equation}\label{eq:pi_rate}
\sum_{j=1}^{\infty}\bigl\|\mPi_j(\hat\vartheta_T)-\mPi^0_j\bigr\|_F^{\,2}
\;\le\; C\Bigl(\frac{\Lambda_\delta}{T}+\rho_\ast^{\,2m}+\frac{\eta}{T}\Bigr).
\end{equation}
No lower bound on $m$ beyond the constant $m_1$ is imposed: the truncation enters only through the reported term $\rho_\ast^{2m}$, and each corollary below imposes of it what it needs; the choice $m\ge\log T/\log(1/\rho_\ast)$ makes that term at most $T^{-2}$.
\end{theorem}

\begin{corollary}[Consistency and rate in summed operator norm]\label{cor:pi_consistency}
Under Assumption~\ref{ass:consistency}:
\begin{enumerate}
\item[(i)] (Consistency.) If $m=m(T)\to\infty$ with $m\le T/2$ and $\eta/T\overset{P}{\to}0$, then
$\sum_{j\ge1}\|\mPi_j(\hat\vartheta_T)-\mPi^0_j\|_{\mathrm{op}}\overset{P}{\longrightarrow}0$.
\item[(ii)] (Rate.) If in addition $m\ge\log T/\log(1/\rho_\ast)$ and there is an event $E_T$ with $\PP(E_T)\ge1-\delta'$ on which $\eta\le C'\,n_{\mathrm{par}}\log T$ for a constant $C'$ (in particular, deterministically with $\delta'=0$), then with probability at least $1-\delta-\delta'$,
\[
\sum_{j\ge1}\bigl\|\mPi_j(\hat\vartheta_T)-\mPi^0_j\bigr\|_{\mathrm{op}}
\;\le\; C\sqrt{\frac{\Lambda_\delta\,\log T}{T}}
\;=\; O\Bigl(\log T\sqrt{\tfrac{n_{\mathrm{par}}}{T}}\Bigr)\quad\text{for fixed }\delta.
\]
\end{enumerate}
\end{corollary}

The proof is given in Appendix~\ref{sec:proof_pi_rate}. The coupling between $m$ and $T$ is weak: the truncation may be any lag in $[m_1,T/2]$ (the only remaining requirement, $\Lambda_\delta\le cT$, is a sample-size condition that does not involve $m$), and the rate regime of part~(ii) already holds at $m\asymp\log T$, the same logarithmic growth of the fitted order as in the long-autoregression theory of \citet{hannan1984multivariate} and \citet{lewis1985prediction}. This reflects the geometric truncation bias of Proposition~\ref{prop:trunc_bias}: a bias decaying only polynomially in $m$ would force a polynomial balance between $m$ and $T$. In practice $m$ is set from the time-scale of the slowest moving-average mode of the fitted model, which is well inside the feasible set and decays at a fixed rate; the bound above is stated through the uniform rate $\rho_\ast$ over the whole set, and is correspondingly conservative. The $T^{-1/2}$ dependence of the rate cannot be improved: already in a one-parameter Gaussian VAR$(1)$ submodel, a two-point (Le~Cam) argument shows that no estimator does better than $c/\sqrt{T}$ for this loss (Remark~\ref{rem:pi_lower}); we do not pursue a matching $n_{\mathrm{par}}$-dependent lower bound for the summed operator-norm loss.

Both estimators are handled by one decomposition of the empirical excess loss, into a self-normalized martingale and a Gaussian quadratic form; the $\Sigma$-marginalized JMAP loss localizes to a $\mSigma^{-1}$-weighted version of the same quadratic form, so the two share the rate and differ only in constants, as the experiments of Section~\ref{sec:Experiments} bear out, RLS and JMAP separating only through the conditioning of $\mSigma$. Gaussianity is used only through the innovations' moment-generating function and concentration of Gaussian quadratic forms (we do not pursue the classical fourth-moment extensions). Finally, the theorem is a fixed-order, fixed-dimension result: its constants depend on $(p,q,d)$ and $B$, and regimes in which $d$, $p$ or $q$ grow with $T$ are not treated; the high-dimensional claims of this paper are computational (cf.\ \citealp{basu2015regularized} for sparse-VAR theory in the growing-dimension regime).

\begin{remark}[Optimization error and compactness]\label{rem:compactness}
Because Theorem~\ref{thm:pi_rate} is stated for approximate empirical minimizers, it already covers inexact optimization: the guarantee degrades only through the additive $\eta/T$, so an exact minimizer ($\eta=0$) attains the rate, and by Corollary~\ref{cor:pi_consistency} so does any $\hat\vartheta$ whose objective is within $O(n_{\mathrm{par}}\log T)$ of the minimum with high probability. Assumption~\ref{ass:consistency}(iii) also covers the cross-validated ridge variances of the experimental pipeline, since the proof touches the selected variances only through the penalty bound $P_B$ and the selection respects a fixed floor (the validation search box of Section~\ref{sec:exp_protocol} is bounded below). The implementation minimizes the same objectives by unconstrained L-BFGS, which returns a local stationary point of a nonconvex objective; we do not certify that its optimization error $\eta$ is small, so the theorem informs but does not by itself guarantee the computed estimate.
For the statistical guarantee we analyze the compactly constrained version of the estimator, minimizing over the fixed ball $K_B$ of Assumption~\ref{ass:consistency}(ii), which contains a representative $\vartheta^0$ of the true transfer function. Constraining to a compact set is standard for nonconvex, stability-constrained time-series parametrizations \citep{dunsmuir1976vector,hannan2012statistical}: on $K_B$ the spectral radii of all fitted companion matrices are bounded away from one, ensuring the uniform geometric-decay and Lipschitz constants of the implied infinite-VAR coefficients on which the proof rests (Lemma~\ref{lem:uniform_geom} in Appendix~\ref{sec:proof_pi_rate}), and, more generally, uniform mixing of the associated processes. The implementation minimizes the same objectives by L-BFGS that is unconstrained in the PAR coordinates, stability and invertibility being supplied by the map itself (Proposition~\ref{prop:par_bijection}), with the ridge penalties keeping the iterates in a bounded region in practice; the ball is a proof device only and is not enforced, and the radius $B$ may be taken as large as desired, at the price of larger constants in Theorem~\ref{thm:pi_rate}. These constants are not uniform in $B$, and they deteriorate as the fitted models approach the stability boundary: a larger $B$ admits companion matrices with spectral radius closer to one, so the uniform decay rate $\rho_\ast$ of Lemma~\ref{lem:uniform_geom} tends to one while the geometry constants $\bar G,\bar L$ grow, and the implicit constants degrade for processes whose roots sit near the unit circle.
\end{remark}

The guarantee extends to the multiplicative seasonal model of Section~\ref{sec:seasonal_varma} to the seasonal model with only notational changes, because the estimators enter the analysis solely through the expanded infinite-VAR coefficients and the smooth, per-factor PAR parametrization.

\begin{corollary}[Seasonal VARMA]\label{cor:sarma_rate}
Let the observations be a sample of a stationary Gaussian multiplicative seasonal VARMA~\eqref{eq:sarma} with period $s$ and factor orders $(\bar p,\bar q,\bar P,\bar Q)$, each of whose four factor polynomials is zero-free on the closed unit disk, with i.i.d.\ innovations $N(\vct{0},\mSigma)$, $\mSigma\succ\mz$. Fit the seasonal model with the same period, factor orders $p\ge\bar p$, $q\ge\bar q$, $P\ge\bar P$, $Q\ge\bar Q$, and truncation $m$, by minimizing the RLS or $\Sigma$-marginalized MAP loss (the latter with $\mPsi_0\succ\mz$ and $\nu_0>d-1$) over a ball $K_B$ in the four unconstrained PAR stacks $(\mM^{\mPhi},\mM^{\mPhi^{(s)}},\mM^{\mTheta},\mM^{\mTheta^{(s)}})\in\RR^{n_{\mathrm{par}}}$, where now $n_{\mathrm{par}}:=(p+q+P+Q)\,d^2$ and $K_B$ contains the inverse-PAR coordinates of the true four factor polynomials after zero-padding their coefficient arrays to the fitted orders. Then the conclusions of Theorem~\ref{thm:pi_rate} and Corollary~\ref{cor:pi_consistency} hold for the expanded infinite-VAR coefficients $\mPi_{1:\infty}$, with constants depending in addition on the period $s$. Moreover, $\rho_\ast$ may be taken of the form $\tilde\rho^{\,1/s}$ with $\tilde\rho<1$, so the rate regime of Corollary~\ref{cor:pi_consistency}(ii) requires $m\gtrsim s\log T$: the truncation must resolve the seasonal time-scale, consistent with the settings $m\gg s$ used in Sections~\ref{sec:exp_singapore} and~\ref{sec:exp_beijing}.
\end{corollary}

The short proof (stability and analyticity of the expanded operators on a disk of radius $(1/\bar\rho)^{1/s}$, after which the argument of Appendix~\ref{sec:proof_pi_rate} applies verbatim) is given in Appendix~\ref{sec:pf_sarma}.

\section{Efficient likelihood evaluation via Fourier methods}\label{sec:complexity_S2}

Both losses developed in the previous two sections, the regularised least-squares loss $\mathcal{L}_m^{\mathrm{RLS}}$ of Section~\ref{sec:RLS} and the $\Sigma$-marginalised MAP loss $\mathcal{L}_m^{\mathrm{marg}}$ of Section~\ref{sec:MAP}, are built around the truncated least-squares criterion $L_m$ of Section~\ref{sec:LSloss}. The RLS loss is $\tfrac12 L_m$ plus four data-independent ridge penalties; the MAP loss is $\tfrac{\nu_0+N}{2}\log\det(\mPsi_0+\mR)$ plus the same ridge penalties, where $\mR$ is the residual scatter matrix~\eqref{eq:R_residual_cov} that produces $L_m=\tr(\mR)$. In both cases the entire data-dependence of the loss reduces to the same block-quadratic form in the truncated VAR coefficients,
\begin{equation}\label{eq:Q_def}
\mathcal{Q}(\mPi_{1:m})\,:=\,\sum_{i=1}^{m}\sum_{j=1}^{m}\mPi_i\,\mS_2(i,j)\,\mPi_j' \;\in\;\R^{d\times d},
\end{equation}
where the $m\times m$ block $\mS_2$ has $(i,j)$ block
\(
\mS_2(i,j)=\sum_{t=m}^{T-1}\vz_{t+1-i}\vz_{t+1-j}'
\)
(cf.\ Eq.~\eqref{eq:LmPiTheta:def}). Materialised as a 4-tensor $\mS_2\in\R^{m\times m\times d\times d}$, this object would require $m^2 d^2$ floats of storage and $\Theta(m^2 d^3)$ flops per evaluation of $\mathcal{Q}$, both of which are prohibitive in the large-$m$ regime needed for the truncation~\eqref{eq:gamma_bias} to be accurate, since L-BFGS makes many evaluations per fit and we refit on a rolling basis.

We avoid materialising $\mS_2$ entirely by exploiting its block-Toeplitz-plus-correction structure. Each block depends almost exclusively on the lag $\tau=i-j$:
\begin{equation}\label{eq:S2_decomp}
\mS_2(i,j)=\widehat{\mC}_{i-j}\;-\;\Delta^{\mathrm{head}}_{i,j}\;-\;\Delta^{\mathrm{tail}}_{i,j},
\end{equation}
where
\(
\widehat{\mC}_{\tau}:=\sum_{s=1}^{T-|\tau|}\vz_{s+\max(-\tau,0)}\,\vz_{s+\max(\tau,0)}'
\)
is the (untruncated) empirical autocovariance of $\vZ$ at lag $\tau$; for $\tau\ge0$ this reads $\sum_{s=1}^{T-\tau}\vz_s\vz_{s+\tau}'$, every index lies in $[1,T]$, and $\widehat{\mC}_{-\tau}=\widehat{\mC}_{\tau}'$. The corrections $\Delta^{\mathrm{head}}_{i,j},\Delta^{\mathrm{tail}}_{i,j}$ remove the terms of $\widehat{\mC}_{i-j}$ whose summation index falls outside the (narrower) range of $\mS_2(i,j)$. By the block symmetry $\mS_2(i,j)=\mS_2(j,i)'$ it suffices to record them for $i\ge j$ (so $\tau=i-j\ge0$); blocks with $i<j$ then follow by transposition, which sidesteps any negative-lag indexing. For $i\ge j$, $\mS_2(i,j)=\sum_{s=m-i+1}^{T-i}\vz_s\vz_{s+\tau}'$ sums over $[m-i+1,\,T-i]$ rather than the full range $[1,\,T-\tau]$ of $\widehat{\mC}_\tau$, and the difference is exactly
\begin{equation}\label{eq:boundary_def}
\Delta^{\mathrm{head}}_{i,j}=\sum_{s=1}^{m-i}\vz_s\,\vz_{s+\tau}',\qquad
\Delta^{\mathrm{tail}}_{i,j}=\sum_{s=T-i+1}^{T-i+j}\vz_s\,\vz_{s+\tau}'\qquad(\tau=i-j\ge0),
\end{equation}
whose indices all lie in $[1,T]$ (the ranges $[1,m-i]$, $[m-i+1,T-i]$, $[T-i+1,T-\tau]$ partition $[1,T-\tau]$) and which involve only the boundary segments
\(
\mY^{\mathrm{head}}:=(\vz_1,\ldots,\vz_{m-1})'\in\R^{(m-1)\times d}
\)
and
\(
\mY^{\mathrm{tail}}:=(\vz_{T-m+1},\ldots,\vz_{T})'\in\R^{m\times d}.
\)
The decomposition~\eqref{eq:S2_decomp} replaces the $m^2d^2$ tensor $\mS_2$ by three lighter objects:
\begin{itemize}
\item the autocovariance sequence $\widehat{\mC}_{0:m-1}\in\R^{m\times d\times d}$ ($md^2$ floats);
\item the boundary blocks $\mY^{\mathrm{head}}\in\R^{(m-1)\times d}$ and $\mY^{\mathrm{tail}}\in\R^{m\times d}$ ($\sim 2md$ floats);
\end{itemize}
total $O(md^2)$ storage, a factor-$m$ reduction. These objects are computed once per training window in a single FFT-based pre-pass over the data, costing $O(T\log m\,d^2)$ (Appendix~\ref{sec:cost_appendix}), and are stored alongside $\mS_0$, $\mS_1$ as part of the sufficient statistics of Section~\ref{sec:LSloss}.

\paragraph{Toeplitz part via Parseval.}
Frequency-domain evaluation of Gaussian time-series criteria goes back to the Whittle likelihood \citep{whittle1953analysis}; the identity below applies the same Parseval mechanism not to the data likelihood, whose frequency sum grows with $T$, but to the quadratic form in the truncated coefficients, whose length is set by $m$.
Substituting~\eqref{eq:S2_decomp} into~\eqref{eq:Q_def} splits $\mathcal{Q}$ into a block-Toeplitz quadratic plus two boundary Grams:
\begin{equation}\label{eq:Q_split}
\mathcal{Q}(\mPi_{1:m})
\;=\;\underbrace{\sum_{i,j=1}^{m}\mPi_i\,\widehat{\mC}_{i-j}\,\mPi_j'}_{\text{Toeplitz part }\mathcal{Q}^{\mathrm{toep}}}
\;-\;\underbrace{(\mPi*\mY^{\mathrm{head}})^\top(\mPi*\mY^{\mathrm{head}})}_{\text{head Gram}}
\;-\;\underbrace{(\mPi\star\mY^{\mathrm{tail}})^\top(\mPi\star\mY^{\mathrm{tail}})}_{\text{tail Gram}},
\end{equation}
where $*$ denotes linear convolution and $\star$ a cross-correlation, both of length-$m$ filters $\mPi$ with the boundary segments. The Toeplitz part is the Gram of the truncated residuals: writing $\vct{r}_t:=\sum_{k=1}^{m}\mPi_k\,\vz_{t-k}\in\R^{d}$ for the length-$m$ filtering of $\vZ$ by $\mPi$, one has $\mathcal{Q}^{\mathrm{toep}}(\mPi)=\sum_t \vct{r}_t\vct{r}_t'$, which makes its $d\times d$ shape and positive-semidefiniteness manifest. By the convolution theorem $\widetilde{\vct r}_f=\widetilde{\mPi}_f\,\widetilde{\vz}_f$, so Parseval's theorem applied to this matrix Gram gives
\begin{equation}\label{eq:parseval}
\mathcal{Q}^{\mathrm{toep}}(\mPi)=\frac{1}{N_f}\sum_{f=0}^{N_f-1}\widetilde{\mPi}_f\,\widetilde{\mC}_f\,\widetilde{\mPi}_f^\ast,
\end{equation}
where $\widetilde{\mPi}_f=\sum_{k=1}^{m}\mPi_k\,e^{-2\pi i fk/N_f}$ is the DFT of the coefficient sequence, $\widetilde{\mC}_f$ is the DFT of the Hermitian-extended autocovariance $\widehat{\mC}$ at $N_f=2m$ frequencies (a real, positive-semidefinite spectral-density estimate), and $\ast$ denotes conjugate transpose. The orientation ($\widetilde{\mPi}_f$ on the left and $\widetilde{\mPi}_f^\ast$ on the right) is the one consistent with the block convention $\mPi_i\,\widehat{\mC}_{i-j}\,\mPi_j'$ of~\eqref{eq:Q_split}. Eq.~\eqref{eq:parseval} is computed in three steps:
\begin{enumerate}
\item $\widetilde{\mPi}\leftarrow \mathrm{FFT}_{N_f}(\mPi)$: $O(m\log m\cdot d^2)$;
\item $\widetilde{\mM}_f\leftarrow \widetilde{\mPi}_f\,\widetilde{\mC}_f$ for each $f$: $O(m\,d^3)$;
\item $\mathcal{Q}^{\mathrm{toep}}\leftarrow \frac{1}{N_f}\sum_{f=0}^{N_f-1}\Re\!\bigl(\widetilde{\mM}_f\,\widetilde{\mPi}_f^\ast\bigr)$: $O(m\,d^3)$, a direct frequency-domain contraction (the real part discards the imaginary terms, which cancel by the Hermitian symmetry of the summand).
\end{enumerate}
The spectral density $\widetilde{\mC}$ depends on the data only, not on $\mPi$, and is cached once per window. The zero-padding to $N_f=2m$ makes the length-$N_f$ circular convolution reproduce the linear filtering exactly, with no aliasing.

\paragraph{Boundary parts via short convolutions.}
The head and tail Grams in~\eqref{eq:Q_split} are quadratic forms in length-$m$ convolutions of $\mPi$ with the boundary blocks of size $O(m\times d)$. Each convolution is a single FFT-based filter at length $N_f=2m$, costing $O(m\log m\cdot d^2)$; the Gram matrix is then an $O(md^2)$ outer-product reduction. Both boundary corrections are therefore cheaper than the Toeplitz part.

\paragraph{Total cost.}
Summing the three contributions in~\eqref{eq:Q_split}, the per-evaluation cost of $\mathcal{Q}$ (and hence of $\mR$ in~\eqref{eq:R_residual_cov} and of $L_m$ in~\eqref{eq:LmPiTheta:def}) is
\begin{equation}\label{eq:flopcount}
\underbrace{O(m\log m\cdot d^2)}_{\text{FFTs}}+\underbrace{O(m\,d^3)}_{\text{per-frequency }d\times d\text{ matmuls and final }\mPi^\top\mH\text{ contraction}}.
\end{equation}
Both the storage and the compute scale only linearly in $m$ (up to logarithmic factors), so the Fourier construction makes long-memory regimes tractable. The savings persist down the autodiff graph: reverse-mode differentiation through Eq.~\eqref{eq:parseval} is again a length-$N_f$ FFT pair plus a per-frequency conjugate matmul, so the gradient costs the same order as the forward.

\paragraph{Remarks.}
\begin{itemize}
\item The same Parseval trick handles the residual scatter $\mR$ in Eq.~\eqref{eq:R_residual_cov} verbatim: the linear term $\sum_j\mPi_j\,\mS_1(j)$ is already $O(md^3)$ without help, but the quadratic term $\sum_{i,j}\mPi_i\mS_2(i,j)\mPi_j'$ is exactly $\mathcal{Q}(\mPi_{1:m})$.
\item The single $O(T\log m\,d^2)$ pre-pass is amortised across all L-BFGS iterations of a single fit, and across all daily refits in a rolling-window scheme via incremental updates of the cumulative sums underlying $\widehat{\mC}_\tau$, $\mY^{\mathrm{head}}$ and $\mY^{\mathrm{tail}}$.
\end{itemize}

The multiplicative-seasonal factors of Section~\ref{sec:seasonal_varma} and the exogenous regressors of the VARMAX extension leave this evaluation unchanged at leading order: a seasonal VARMA evaluates at the same $O(m\log m\,d^2+m\,d^3)$ per-gradient cost as a plain VARMA of equal dimension, and the VARMAX at the same cost with $d$ replaced by $d+d_x$, independent of $T$ in all cases. The full per-term accounting~\eqref{eq:sarma_cost}, including the dependence on the orders $(p,q,P,Q)$ and the period $s$, is given in Appendix~\ref{sec:cost_appendix}.

\section{Numerical results}\label{sec:Experiments}
We evaluate the two pointwise estimators of Sections~\ref{sec:RLS}--\ref{sec:MAP}: the regularized least-squares estimator (RLS) and the $\Sigma$-marginalized MAP estimator (JMAP). Both use the PAR-shrinkage prior of Sections~\ref{sec:RLS}--\ref{sec:JMAP_reg} with per-group ridge variances chosen by cross-validation (Section~\ref{sec:exp_protocol}) and the fixed lag-one diagonal ridge $\sigma^2_{\mathrm{diag},1}=4$ of Eq.~\eqref{eq:lag1_diag_prior}. We report forecast mean-squared error (MSE) by forecast horizon. Measured timings for both are reported in Appendix~\ref{sec:runtimes}.

\paragraph{Common cross-validation protocol.}\label{sec:exp_protocol}
On all four benchmarks we tune every method's shrinkage and order hyperparameters by a single, common protocol, so the model classes compete on an equal footing. We hold out a validation tail of the first training window, one fifth of its length in every experiment, fit each candidate on the observations preceding it, and score its out-of-sample predictive log-likelihood over the validation origins; the minimizing hyperparameters are then frozen for the test.

For the PAR-shrinkage VARMA(X) the prior variances are selected on a dyadic grid. The own-channel diagonal groups are held at $\sigma_{\mathrm{diag}}=2$, with the lag-one own-channel diagonals at the fixed ridge $\sigma^2_{\mathrm{diag},1}=4$ of Eq.~\eqref{eq:lag1_diag_prior}, and one shared value is swept over the eight points $\sigma_{\mathrm{off}}\in\{2^{1},2^{0},2^{-1},\dots,2^{-6}\}$; this value governs the off-diagonal $\mPhi$ and $\mTheta$ groups together with the exogenous groups where a VARMAX is fit. The eight fits share no state, so the selection runs one process per grid point; an earlier version of these experiments used a derivative-free Nelder--Mead search over the four group variances, whose simplex is sequential and needed about sixty objective evaluations per candidate order. Each fit, in the variance search and the final refit alike, runs L-BFGS for $2000$ iterations at limited-memory depth $100$ on the sufficient-statistics loss of Sections~\ref{sec:RLS}--\ref{sec:MAP}, initialized from a Hannan--Rissanen pilot \citep{hannan1982recursive} fit on the training observations and mapped to the PAR coordinates; the selected configuration is refit on the full training window under the same budget. The discrete order $(p,q[,s])$ is swept by an outer grid, the variance search run within each candidate order. The ridge and Bayesian VAR baselines jointly select the lag order $p$ and their prior variances, as detailed below; the componentwise baselines select their orders by the same criterion, per product on the retail panel and shared across the series on the hourly panels.

\paragraph{VAR baselines.} The ridge VAR(X) fits each equation by generalized ridge least squares on the lag design, penalizing own-lag coefficients by $1/\sigma^2_{\mathrm{diag}}$, cross-series coefficients by $1/\sigma^2_{\mathrm{off}}$, and, where price enters, the exogenous block by its own variance; multi-step forecasts are iterated from the fitted coefficients, conditioning on the known future price where applicable. The Bayesian VAR(X) places the corresponding diagonal/off-diagonal Gaussian prior on the coefficients and treats $\mSigma$ as a plug-in, the residual covariance refined by a few posterior-mean/residual-covariance fixed-point iterations; conditional on that $\mSigma$ the coefficient posterior is Gaussian in closed form. Forecasts are posterior-predictive means, averaged over $120$ posterior coefficient draws at every horizon: on the retail panel the $h$-step forecast is iterated per draw and the forecasts averaged, and for the pure VAR the same average is computed by averaging the per-draw iterated coefficient maps, which is equivalent because the forecast is linear in those maps. At $h=1$ the forecast is also linear in the coefficients, so the average reduces to a posterior-mean plug-in; at longer horizons it is the genuine model average. On the retail panel the diagonal variance is pinned and the cross-series and exogenous variances are tied to the swept dyadic value; on the hourly panels the diagonal and off-diagonal variances are swept on their own grids. Orders and prior variances are selected by the same predictive-log-likelihood criterion as every other model.

\paragraph{Feasible set.} Stability and invertibility hold at every iterate by construction, through the PAR map of Section~\ref{sec:VARMAintro}, so the optimization is unconstrained and needs no projection step. Two devices from the theory are implemented on top of the map: the coefficients are rescaled as $\mPhi_j\mapsto c^{\,j}\mPhi_j$ (and likewise for $\mTheta$) with the fixed cap $c=0.995$ (substituting $\lambda=c\mu$ in the characteristic polynomial shows this scales the companion spectrum by exactly $c$), and the ball $K_B$ of Assumption~\ref{ass:consistency}(ii) is enforced through the squashing $\sigma_B$ of that assumption with radius $B=100$. Both enter through the same map for the two MAP fits, so the estimators explore one parameter set. On the synthetic and retail benchmarks the fits are interior and the devices are inert: the fitted PAR norms stay below $23$, where $\sigma_B$ is a rescaling by $0.97$ or more, and the cap is slack. On the hourly panels they are not inert, and this is their designed role: the diurnal cycle is close to deterministic, the fitted seasonal factors push toward the unit root, and the PAR coordinates grow without bound as the companion radius approaches one, so several JMAP fits reach norms near $B$ with the capped radius attaining $c$. There the ball and the cap supply the persistence bound that the compactness assumption calls for, which is exactly the job the theory gives them; the forecasts are insensitive to the residual drift in these coordinates, which runs along weakly identified AR--MA cancellations. The recovery guarantee is unaffected, being stated for any $K_B$-valued near-minimizer, which the squashed optimizer produces by construction. The optimization error remains uncertified (Remark~\ref{rem:compactness}), and it, rather than the feasible-set devices, is what separates the computed estimate from the analysed one.

\paragraph{Significance tests.} Forecast comparisons are Diebold--Mariano tests on the per-origin squared-error loss, with a Newey--West long-run variance and the \citet{harvey1997testing} small-sample correction against a $t_{n-1}$ reference. Every comparison uses the JMAP fit of the panel's leading model class as a fixed reference rather than the per-horizon winner, so a column reports one estimator against all others; the reference row itself carries no stars. Stars give the two-sided level of that comparison, $^{*}p<0.1$, $^{**}p<0.05$, $^{***}p<0.01$. On the hourly panels $n$ runs to $3.3\times10^{4}$ (Beijing) and $5.3\times10^{4}$ (Singapore) origins, so the test resolves differences of a tenth of a percent and a star there records detectability rather than practical size, for which the MSE values themselves should be read.

\subsection{Synthetic data}\label{sec:exp_synth}
We simulate stationary and invertible VARMA$(2,2)$ processes directly in the unconstrained PAR coordinates of Section~\ref{sec:VARMAintro}: each unconstrained matrix $\mM^{\mPhi}_j,\mM^{\mTheta}_k$ ($j,k=1,2$) has independent Gaussian entries with standard deviation $s_{\mathrm{diag}}$ on the diagonal and $s_{\mathrm{off}}$ off the diagonal, and is mapped through $\mathcal{T}$ to a strictly stable $(\mPhi_{1:2},\mTheta_{1:2})$. The innovations are Gaussian, $\va_t\sim N(\vct{0},\mSigma)$, with a \emph{random, strongly ill-conditioned} covariance $\mSigma=\mtx{Q}\,\mathrm{diag}(\lambda_1,\dots,\lambda_d)\,\mtx{Q}^{\top}$, where $\mtx{Q}$ is a Haar-random orthogonal matrix and the eigenvalues are log-uniform, $\log\lambda_i\overset{\mathrm{iid}}{\sim}\mathrm{Unif}(-10,0)$, giving a condition number up to $e^{10}\approx 2\times10^{4}$ (median $\approx 5\times10^{3}$ over replications); each series is afterwards standardized to unit training variance. This anisotropic, cross-correlated noise is deliberate: it is what makes the covariance model matter, and hence what separates the $\mSigma$-aware JMAP from the isotropic-noise RLS. We consider two regimes: a \emph{dense} regime $(s_{\mathrm{diag}},s_{\mathrm{off}})=(0.5,0.5)$ and a \emph{sparse} regime $(0.5,0.2)$ with weaker cross-channel coupling. For each regime we sweep the dimension $d\in\{10,20,40\}$, and let the training length grow with it: $T\in\{200,400,800\}$ at $d=10$, $\{400,800,1600\}$ at $d=20$, and $\{800,1600,3200\}$ at $d=40$. A VARMA$(2,2)$ carries $4d^2$ coefficients against a scalar sample of size $Nd$, so the information ratio $N/(4d)$ is held near $4.5$ on the shortest rung of every dimension and doubles along each ladder; a grid fixed in $T$ instead confounds dimension with sample size, and its $d=40,T=200$ corner supplies fewer observations than parameters. Each cell is averaged over $8$ independent replications. All estimators are fit with truncation $m=20$; the per-group ridge variances are selected by the cross-validation of Section~\ref{sec:exp_protocol}, and JMAP places the inverse-Wishart prior $\mathrm{IW}(\mtx{I}_d,\,d+4)$ on the innovation covariance. Both fits run L-BFGS for $2000$ iterations. Because the data-generating process is known, we compute the \emph{exact} minimum-MSE VARMA forecast and report each estimator's error relative to this oracle, $\mathrm{MSE}/\mathrm{MSE}_{\mathrm{oracle}}\ge 1$, so that $1$ is optimal.
\begin{table}[t]\centering\small
\caption{Synthetic data: oracle-relative forecast MSE at $h=1$ (lower is better; $1$ is the oracle). The training length grows with the dimension, holding the information ratio $N/(4d)$ near $4.5$ on the shortest rung of each $d$. Best in each row in bold. \texttt{bigtime} \citep{wilms2023sparse} is the sparse-VARMA package baseline (cross-validated, same $8$ seeds).}\label{tab:synth_h1}
\begin{tabular}{cc ccc c ccc}
\toprule
&& \multicolumn{3}{c}{dense $(0.5,0.5)$} && \multicolumn{3}{c}{sparse $(0.5,0.2)$}\\
\cmidrule{3-5}\cmidrule{7-9}
$d$ & $T$ & RLS & JMAP & \texttt{bigtime} && RLS & JMAP & \texttt{bigtime}\\
\midrule
10 & 200 & 1.45 & \textbf{1.11} & 1.33 && 1.41 & \textbf{1.10} & 1.26\\
   & 400 & 1.32 & \textbf{1.05} & 1.23 && 1.28 & \textbf{1.04} & 1.17\\
   & 800 & 1.15 & \textbf{1.02} & 1.19 && 1.16 & \textbf{1.02} & 1.13\\
\cmidrule(lr){1-9}
20 & 400 & 1.48 & \textbf{1.09} & 1.42 && 1.37 & \textbf{1.07} & 1.29\\
   & 800 & 1.32 & \textbf{1.04} & 1.33 && 1.31 & \textbf{1.03} & 1.22\\
   & 1600 & 1.15 & \textbf{1.03} & 1.30 && 1.16 & \textbf{1.02} & 1.19\\
\cmidrule(lr){1-9}
40 & 800 & 1.41 & \textbf{1.11} & 1.55 && 1.34 & \textbf{1.07} & 1.33\\
   & 1600 & 1.22 & \textbf{1.06} & 1.47 && 1.16 & \textbf{1.05} & 1.26\\
   & 3200 & 1.14 & \textbf{1.04} & 1.43 && 1.12 & \textbf{1.03} & 1.24\\
\bottomrule
\end{tabular}
\end{table}
\begin{table}[t]\centering\small
\caption{Synthetic data: oracle-relative forecast MSE at $h=2$ (columns as in Table~\ref{tab:synth_h1}).}\label{tab:synth_h2}
\begin{tabular}{cc ccc c ccc}
\toprule
&& \multicolumn{3}{c}{dense $(0.5,0.5)$} && \multicolumn{3}{c}{sparse $(0.5,0.2)$}\\
\cmidrule{3-5}\cmidrule{7-9}
$d$ & $T$ & RLS & JMAP & \texttt{bigtime} && RLS & JMAP & \texttt{bigtime}\\
\midrule
10 & 200 & 1.32 & \textbf{1.08} & 1.19 && 1.30 & \textbf{1.07} & 1.16\\
   & 400 & 1.21 & \textbf{1.04} & 1.12 && 1.20 & \textbf{1.03} & 1.10\\
   & 800 & 1.11 & \textbf{1.02} & 1.10 && 1.11 & \textbf{1.02} & 1.07\\
\cmidrule(lr){1-9}
20 & 400 & 1.36 & \textbf{1.07} & 1.23 && 1.29 & \textbf{1.06} & 1.18\\
   & 800 & 1.25 & \textbf{1.04} & 1.17 && 1.22 & \textbf{1.02} & 1.13\\
   & 1600 & 1.11 & \textbf{1.03} & 1.15 && 1.12 & \textbf{1.02} & 1.10\\
\cmidrule(lr){1-9}
40 & 800 & 1.32 & \textbf{1.09} & 1.26 && 1.25 & \textbf{1.05} & 1.17\\
   & 1600 & 1.17 & \textbf{1.06} & 1.22 && 1.13 & \textbf{1.04} & 1.13\\
   & 3200 & 1.11 & \textbf{1.03} & 1.19 && 1.09 & \textbf{1.03} & 1.12\\
\bottomrule
\end{tabular}
\end{table}
\begin{figure}[t]\centering
\begin{subfigure}{0.32\textwidth}\includegraphics[width=\linewidth]{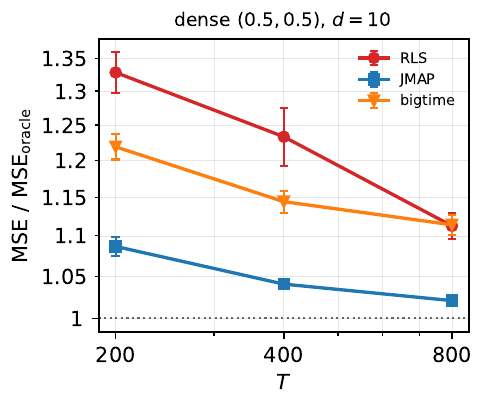}\end{subfigure}\hfill
\begin{subfigure}{0.32\textwidth}\includegraphics[width=\linewidth]{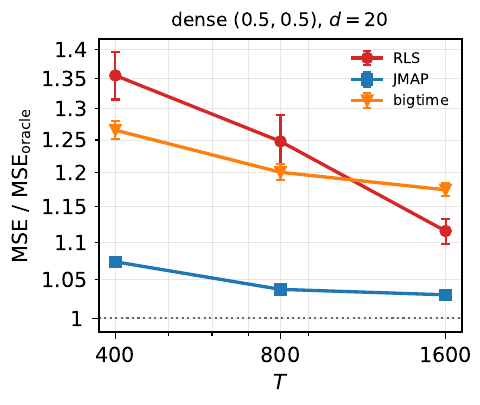}\end{subfigure}\hfill
\begin{subfigure}{0.32\textwidth}\includegraphics[width=\linewidth]{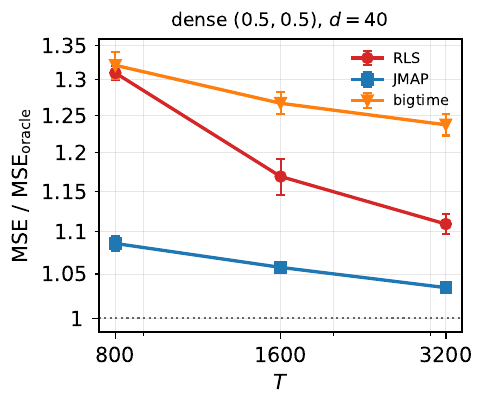}\end{subfigure}

\medskip
\begin{subfigure}{0.32\textwidth}\includegraphics[width=\linewidth]{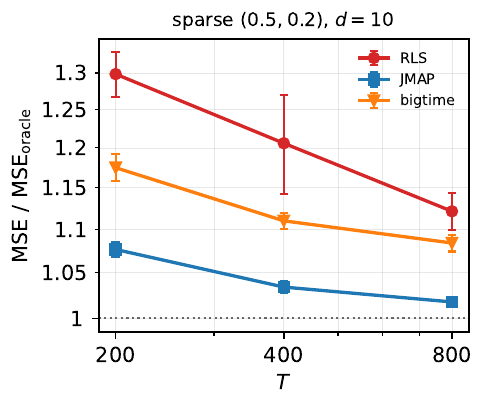}\end{subfigure}\hfill
\begin{subfigure}{0.32\textwidth}\includegraphics[width=\linewidth]{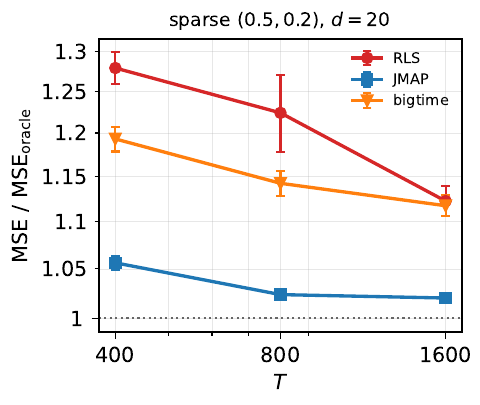}\end{subfigure}\hfill
\begin{subfigure}{0.32\textwidth}\includegraphics[width=\linewidth]{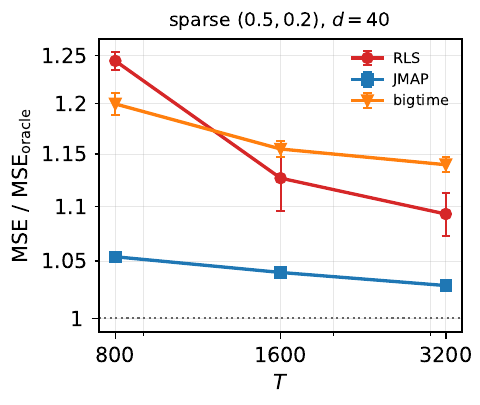}\end{subfigure}
\caption{Oracle-relative forecast MSE, horizon-averaged over $h=1,2,3$, versus training length $T$ (both axes log), for the dense $(0.5,0.5)$ regime (\emph{top}) and the sparse $(0.5,0.2)$ regime (\emph{bottom}) at $d=10,20,40$. The training length grows with the dimension, so each panel has its own ladder. The $\mSigma$-aware JMAP approaches the oracle (dotted line at $1$) from above; RLS and \texttt{bigtime} trail it at every $(d,T)$, and \texttt{bigtime}'s gap widens with $d$. Error bars are $\pm1$ standard error over the eight seeds.}\label{fig:synth_oracle}
\end{figure}
Tables~\ref{tab:synth_h1} and~\ref{tab:synth_h2} report the oracle-relative error at $h=1,2$, and Figure~\ref{fig:synth_oracle} plots the $h=1$ value against $T$. The $\mSigma$-aware JMAP leads everywhere. It improves on the isotropic-noise RLS in all eighteen cells and in all $144$ cell-seed pairs, by $20.7\%$ on average at $h=1$ (paired $t=-17.8$ across the pooled seeds), the gain coming from modelling the cross-channel covariance RLS ignores. The advantage is a finite-sample one: it shrinks monotonically along every ladder, from $30.9\%$ to $12.6\%$ at $d=10$ and from $27.5\%$ to $9.9\%$ at $d=40$, as both estimators converge on the oracle. The two estimators also select different shrinkage: along the dense $d=10$ ladder RLS relaxes from $\sigma_{\mathrm{off}}=0.25$ to $1.5$ while JMAP stays at $0.25$; at the larger dimensions both relax with $T$. The oracle-relative error still grows with $d$ and decays toward $1$ with $T$, and the sparse regime sits closer to its oracle than the dense one, its weak cross-channel block carrying little recoverable signal. The sparse-VARMA package \texttt{bigtime} trails throughout, by $16$ to $41\%$ in the dense regime and $11$ to $24\%$ in the sparse one, the gap widening with $d$ and barely closing as $T$ grows.

Two standard \textsf{R} VARMA packages serve as external reference points on the same data. \texttt{bigtime} \citep{wilms2023sparse}, a \emph{sparse} VARMA fit by penalised least squares, is reliable and appears as a column of Tables~\ref{tab:synth_h1}--\ref{tab:synth_h2}: it trails the $\mSigma$-aware JMAP in all eighteen cells of the design, and beats the isotropic RLS in nine of them, concentrated at the shortest training length of each dimension and in the sparse regime, where its sparsity assumption is closer to the truth. Its gap to JMAP widens with the dimension, from $0.17$ at $d=10,T=800$ to $0.39$ at $d=40,T=3200$ in the dense regime.

The other, \texttt{MTS} \citep{tsay2013multivariate}, implements the classical Gaussian conditional MLE with no constraint on the fitted operators. We fit VARMA$(2,2)$ with it on the $d=10$, $T=800$ series of both regimes, the same sixteen series the estimators above are fit on. Two of the sixteen stop with a singular Hessian. Of the fourteen that return a fit, twelve have a \emph{non-invertible} moving-average operator, with companion spectral radius between $1.016$ and $1.444$; every fitted autoregressive operator is stable, with radius at most $0.907$, so it is invertibility alone that fails. A non-invertible $\mTheta$ makes the innovation recursion that generates the forecasts grow geometrically in its own past output, and the forecasts follow: across those twelve runs the oracle-relative error at $h=1$ ranges from $9.8\times10^{23}$ to numerical overflow, against $1.02$ for JMAP on the same series. The two invertible fits forecast at $1.10$ and $1.15$, still behind JMAP. We therefore exclude \texttt{MTS} from the tables. This is exactly the failure the PAR reparametrisation removes, the map $\mathcal{T}$ rendering every iterate strictly stable \emph{and} invertible by construction.

\subsection{Dominick's retail scanner data: demand with price as an exogenous driver}\label{sec:exp_dominicks}
Our second real-data benchmark turns the retail panel into a VARMA\emph{X} problem: forecasting category-level \emph{demand} (units) while conditioning on the contemporaneous per-unit \emph{price}, which embeds the store's promotional activity. From the Dominick's Finer Foods scanner database (Kilts Center, University of Chicago Booth) we forecast a $d=12$ basket comprising every food and drink category of the archive with under $10\%$ of weeks missing (cereal, cookies, bottled juice, cheese, canned soup, snack crackers, crackers, front-end candies, frozen juice, canned tuna, soft drinks, and frozen entrees) over the $399$-week record; every excluded food or drink category is missing at least $23.6\%$, and the members share the shopping-basket and promotional co-movement a multivariate model can exploit. The endogenous series are log units, the exogenous series the log volume-weighted price per unit (total dollars over total units), so a promotion enters the model as a transient price cut. We model the first difference of the logs, which renders the unit-root-like levels stationary, and truncate the implied autoregressive representation at $m=20$ lags. Forecasts use a rolling test of $96$ weekly origins at $h=1,2,3$: every model is \emph{fit once} on the first $299$-week window and frozen, with the price for each forecast week supplied. This is therefore \emph{conditional} forecasting given the observed (realized) price path, not an unconditional demand forecast: conditioning isolates the demand--price dynamics, and a planning application would substitute the retailer's posted price for the realized one.
\begin{table}[t]\centering\small
\caption{Dominick's ($d=12$ grocery basket, demand with price exogenous, first-differenced growth, truncation $m=20$): standardized-scale forecast MSE on the $96$-origin rolling test, by horizon. Every model is fit \emph{once} on the first $299$-week window and frozen for all origins; order and shrinkage are chosen by the cross-validation of Section~\ref{sec:exp_protocol} (selected order in parentheses). Because a small fraction of category-weeks are missing, each entry is the per-origin mean squared (standardized-growth) error over the categories \emph{observed} at that origin, averaged equally over the $96$ origins. Best per column in \textbf{bold}; stars give the Diebold--Mariano significance of the difference from the JMAP fit (\citet{harvey1997testing}-corrected, Newey--West HAC long-run variance, $t_{95}$ reference): $^{*}p<0.1$, $^{**}p<0.05$, $^{***}p<0.01$.}\label{tab:dominicks}
\begin{tabular*}{\linewidth}{@{\extracolsep{\fill}}l*{3}{>{\centering\arraybackslash}p{0.15\linewidth}}@{}}
\toprule
model & $h=1$ & $h=2$ & $h=3$\\
\midrule
\multicolumn{4}{@{}l@{}}{\rlap{\emph{PAR-shrinkage VARMAX$(p,q,s)$ with price exogenous, order and ridge variances by CV:}}}\\
\quad RLS\ $(1,1,1)$ & \textbf{0.4442}$^{***}$ & 0.6938 & 0.7213\\
\quad JMAP\ $(1,1,1)$ & 0.5051 & \textbf{0.6583} & \textbf{0.6789}\\
\midrule
\multicolumn{4}{@{}l@{}}{\rlap{\emph{PAR-shrinkage VARMA$(p,q)$ on units only (price-free ablation), order and ridge by CV:}}}\\
\quad RLS\ $(1,1)$ & 0.7516$^{***}$ & 1.0678$^{***}$ & 1.0857$^{***}$\\
\quad JMAP\ $(1,1)$ & 0.7619$^{***}$ & 1.0719$^{***}$ & 1.0912$^{***}$\\
\midrule
\multicolumn{4}{@{}l@{}}{\rlap{\emph{Ridge VARX$(p,p)$, order and shrinkage by CV:}}}\\
\quad VARX\ $(p{=}1)$ & 0.6541$^{*}$ & 0.8757$^{*}$ & 0.9268$^{*}$\\
\midrule
\multicolumn{4}{@{}l@{}}{\rlap{\emph{Bayesian BVARX$(p,p)$, order and prior by CV, model-averaged:}}}\\
\quad BVARX\ $(p{=}1)$ & 0.6343$^{*}$ & 0.8257$^{*}$ & 0.8650$^{*}$\\
\midrule
\multicolumn{4}{@{}l@{}}{\rlap{\emph{Componentwise ARMAX (per-product \texttt{SARIMAX}), orders $(p,q)$ and own-price lag $r$ by CV:}}}\\
\quad comp.\ ARMAX & 0.5999$^{*}$ & 0.7919$^{*}$ & 0.7938\\
\midrule
\multicolumn{4}{@{}l@{}}{\rlap{\emph{Hierarchical-lasso sparse VARX (\texttt{bigtime}), lag and penalty by CV:}}}\\
\quad sparse VARX\ $(p{=}4)$ & 0.5361 & 0.7511 & 0.7855\\
\midrule
\multicolumn{4}{@{}l@{}}{\rlap{\emph{Naive baselines:}}}\\
\quad training-set mean & 1.0735$^{***}$ & 1.0743$^{***}$ & 1.0848$^{***}$\\
\quad stay-in-place (RW) & 1.0735$^{***}$ & 1.0742$^{***}$ & 1.0848$^{***}$\\
\bottomrule
\end{tabular*}
\end{table}

Every method is tuned by the common protocol of Section~\ref{sec:exp_protocol}, jointly over structure and shrinkage (selected orders in Table~\ref{tab:dominicks}); the JMAP variant marginalizes $\mSigma$ under the default $\mathrm{IW}(\mI_d,\,d+4)$ prior. The truncation is $m=20$ and the validation tail is one fifth of the $299$-week training window. The predictive log-likelihood used for selection here conditions on the known future price and is evaluated only on the channels observed at each origin, so the exogenous VARMAX models and the price-free ablations are scored on one criterion. Two contrasts are built in. To quantify the exogenous price we add the price-free \emph{VARMA} counterparts (units alone); to separate the moving-average term from the cross-series coupling we add a ridge VARX and Bayesian BVARX (no MA), a per-product componentwise ARMAX (\texttt{SARIMAX}, no cross-series coupling), and the hierarchical-lasso sparse VARX of \texttt{bigtime}, the closest match the package offers, as it implements VAR\emph{X} but not VARMA\emph{X}. Naive random-walk and training-set-mean baselines anchor the scale.

Price dominates: every price-conditioned model roughly halves the naive $h{=}1$ error, and dropping price (the units-only VARMA) inflates the best $h{=}1$ MSE by two thirds. The moving-average term is also decisive, the full VARMAX beating the moving-average-free ridge and Bayesian VARX by $25$--$37\%$ across horizons. This is over-differencing (first-differencing a near-integrated series induces a near-unit moving-average root), so the baselines that recover it otherwise, the longer sparse autoregression of \texttt{bigtime} ($0.536$ at $h=1$, lag $4$) and the componentwise ARMAX ($0.600$), fall between the VARMAX and the short VARX; as on the seasonal panels, the gain comes from the moving-average term, not the cross-series coupling. Among the VARMAX estimators, both cross-validated to order $(1,1,1)$, the isotropic-noise RLS attains the lowest $h{=}1$ error, $12\%$ below JMAP (the only significant difference between the two, $p<0.001$), while JMAP is ahead by $5$--$6\%$ at $h=2,3$; over horizons the two average within $1\%$. The remaining models trail with the significance shown in Table~\ref{tab:dominicks} (DM on the per-origin loss against the JMAP fit as the fixed reference, Newey--West HAC, \citet{harvey1997testing} correction, $n=96$); the sparse VARX is the closest and is not separated from JMAP.

\subsection{Singapore meteorological data: a seasonal VARMA}\label{sec:exp_singapore}
Our third real-data benchmark isolates a structure absent from the previous two: strong \emph{periodicity}. Hourly weather is dominated by a diurnal cycle, and even after its deterministic part is removed the anomalies retain a daily-recurring autocorrelation. This is the regime for which the multiplicative \emph{seasonal} VARMA is designed, and it lets us add a fourth model to the comparison, a seasonal VARMA whose seasonal factor an ordinary VAR or dense VARMA can only mimic with many extra lags.

\paragraph{Model and prior.}
We fit the multiplicative seasonal VARMA of Section~\ref{sec:seasonal_varma}, Eq.~\eqref{eq:sarma}, with diurnal period $s=24$ and the four factor orders tied to a common $p=P=q=Q$ chosen by cross-validation; stationarity and invertibility hold by construction through the per-factor PAR maps. All four factor-stacks share the same PAR-shrinkage prior, fixed lag-one diagonal ridge, and per-group ridge variances as the regular VARMA, and (for JMAP) the same $\mathrm{IW}(\mI_d,\,d+4)$ prior on $\mSigma$ of Sections~\ref{sec:RLS}--\ref{sec:MAP}. JMAP minimizes the $\Sigma$-marginalized MAP loss, RLS the isotropic-noise criterion.

\paragraph{Data and protocol.}
We forecast $d=11$ hourly meteorological series for Singapore (2\,m temperature, relative humidity, dew point, precipitation, mean-sea-level pressure, total cloud cover, the zonal and meridional 10\,m wind components, wind gusts, shortwave radiation, and wind speed) from the Open-Meteo reanalysis archive (ECMWF-IFS) over 2017--2025. Each series is deseasonalized by regressing out, per channel by ordinary least squares on the training window only, an additive calendar design---an intercept, $24$ hour-of-day indicators, a piecewise-linear annual cycle (a tent basis interpolating between monthly knots), and a piecewise-linear weekly cycle---and standardized by its training-window standard deviation, so the targets are standardized hourly anomalies on which a zero (climatological-anomaly) forecast scores MSE $\approx 1$. We use a rolling scheme of six annual test windows (2020--2025); for each, every model is refit on the preceding three years ($\approx\!26{,}000$ hourly observations) and scored at horizons $h=1,\dots,6$ hours. The seasonal period is $s=24$ (diurnal) and the truncation $m=1024$.

Every hyperparameter is chosen once by the common protocol of Section~\ref{sec:exp_protocol}, on a validation tail (the last $219$ days, one fifth) of the first, 2017--2019 training window, then frozen and refit on each test year. The orders swept are $p=P=q=Q\in\{1,\dots,4\}$ for the seasonal VARMA and the seasonal-VAR ablation, $p\in\{12,24,36\}$ for the non-seasonal dense-VARMA ablation, and $p\in\{24,48,96,168\}$ (AR) and $p=q\in\{24,48\}$ (ARMA), shared across the series, for the componentwise baselines; the ridge/Bayesian VAR ($p\in\{24,48,96\}$) and the sparse \texttt{bigtime} VAR/VARMA are tuned as on the other panels. The truncation is $m=1024$ and the seasonal period $s=24$. Table~\ref{tab:singapore} reports the per-origin forecast MSE, averaged over the six test years.

\begin{table}[t]\centering\small
\caption{Singapore weather ($d=11$ hourly anomalies): forecast MSE per origin (the mean over the $11$ series), averaged equally over the test origins of the six annual windows (2020--2025), at horizons $h=1,\dots,6$ hours. Order and ridge/prior hyperparameters are chosen once by the common cross-validation protocol on the first training window and held fixed across windows (selected order in parentheses). Best in each column in bold.}\label{tab:singapore}
\begin{tabular*}{\linewidth}{@{\extracolsep{\fill}}l cccccc@{}}
\toprule
model & $h=1$ & $h=2$ & $h=3$ & $h=4$ & $h=5$ & $h=6$\\
\midrule
\multicolumn{7}{l}{\emph{Seasonal PAR-shrinkage VARMA$(p,p,p,p)_{24}$, $p$ and ridge variances by CV:}}\\
\quad RLS \ ($p{=}1$) & 0.2768$^{***}$ & 0.4621$^{***}$ & 0.5762$^{***}$ & 0.6535$^{***}$ & 0.7075$^{***}$ & 0.7473$^{***}$\\
\quad JMAP \ ($p{=}4$) & \textbf{0.2747} & \textbf{0.4535} & \textbf{0.5605} & \textbf{0.6306} & \textbf{0.6777} & 0.7111\\
\midrule
\multicolumn{7}{l}{\emph{Dense VARMA$(p,p)$ (no seasonal factor), order and ridge variances by CV:}}\\
\quad JMAP & 0.2797$^{***}$ & 0.4614$^{***}$ & 0.5690$^{***}$ & 0.6383$^{***}$ & 0.6842$^{***}$ & 0.7161$^{***}$\\
\quad RLS & 0.2800$^{***}$ & 0.4609$^{***}$ & 0.5683$^{***}$ & 0.6376$^{***}$ & 0.6837$^{***}$ & 0.7160$^{***}$\\
\midrule
\multicolumn{7}{l}{\emph{Seasonal VAR$(p,0,P,0)_{24}$ (no MA), order by CV:}}\\
\quad RLS & 0.2783$^{***}$ & 0.4640$^{***}$ & 0.5791$^{***}$ & 0.6571$^{***}$ & 0.7114$^{***}$ & 0.7509$^{***}$\\
\quad JMAP & 0.2786$^{***}$ & 0.4649$^{***}$ & 0.5800$^{***}$ & 0.6576$^{***}$ & 0.7112$^{***}$ & 0.7501$^{***}$\\
\midrule
\multicolumn{7}{l}{\emph{Ridge / Bayesian VAR$(p)$, order and prior by CV:}}\\
\quad ridge VAR & 0.2796$^{***}$ & 0.4609$^{***}$ & 0.5678$^{***}$ & 0.6364$^{***}$ & 0.6815$^{***}$ & 0.7128\\
\quad Bayesian VAR & 0.2780$^{***}$ & 0.4584$^{***}$ & 0.5649$^{***}$ & 0.6330$^{**}$ & 0.6780 & \textbf{0.7092}\\
\midrule
\multicolumn{7}{l}{\emph{Componentwise univariate baselines, order by CV:}}\\
\quad AR & 0.2904$^{***}$ & 0.4703$^{***}$ & 0.5779$^{***}$ & 0.6478$^{***}$ & 0.6944$^{***}$ & 0.7267$^{***}$\\
\quad ARMA & 0.2885$^{***}$ & 0.4643$^{***}$ & 0.5679$^{***}$ & 0.6349$^{***}$ & 0.6796 & 0.7109\\
\midrule
\multicolumn{7}{l}{\emph{Sparse VARMA package \texttt{bigtime} \citep{wilms2023sparse}, lag and penalty by CV:}}\\
\quad sparse VAR & 0.3032$^{***}$ & 0.5082$^{***}$ & 0.6462$^{***}$ & 0.7458$^{***}$ & 0.8190$^{***}$ & 0.8743$^{***}$\\
\quad sparse VARMA & 0.2954$^{***}$ & 0.5045$^{***}$ & 0.6454$^{***}$ & 0.7453$^{***}$ & 0.8179$^{***}$ & 0.8720$^{***}$\\
\midrule
\multicolumn{7}{l}{\emph{Naive baselines:}}\\
\quad training-set mean & 1.1441$^{***}$ & 1.1441$^{***}$ & 1.1441$^{***}$ & 1.1441$^{***}$ & 1.1441$^{***}$ & 1.1441$^{***}$\\
\quad stay-in-place (RW) & 0.3388$^{***}$ & 0.6138$^{***}$ & 0.8161$^{***}$ & 0.9791$^{***}$ & 1.1094$^{***}$ & 1.2148$^{***}$\\
\bottomrule
\end{tabular*}
\end{table}

Table~\ref{tab:singapore} exhibits a third regime, distinct from the previous two panels: the predictability is dominated by \emph{periodicity}, and the seasonal VARMA captures it most economically. The seasonal PAR-shrinkage JMAP leads at the first five horizons, ahead of the seasonal RLS by $0.76\%$ at $h=1$ with the gap widening to $5.1\%$ by $h=6$; at $h=6$ the Bayesian VAR edges ahead of it by $0.3\%$. The separations are small in absolute terms but consistent across horizons. The advantage is one of parameter economy: cross-validation selects a seasonal order $p=P=q=Q=4$ (just $4\cdot 4\,d^2=1936$ coefficients, encoding an expanded VARMA of AR and MA order $4+24\cdot 4=100$), whereas without the seasonal factor the dense VARMA climbs to $p=12$--$36$ and the VAR to $p=48$, three or more times as many coefficients, to reach across a full diurnal period. A single seasonal factor at lag $24$ represents, with a third or less of the parameters, the daily-recurring autocorrelation that the flat models can only approximate with two dozen or more lags, and it also forecasts slightly more accurately, improving on the dense VARMA and the cross-validated VAR by $1.8\%$ at $h=1$. The sparse \texttt{bigtime} VAR/VARMA are the weakest multivariate models, behind even the componentwise baselines; sparsity is the wrong bias here, as the denser ozone panel of Section~\ref{sec:exp_beijing} examines in detail.

Both the seasonal autoregression and the seasonal moving average contribute, and (unlike the differenced retail panel of Section~\ref{sec:exp_dominicks}, where a non-seasonal moving average carried the gain) here the gain comes from the seasonal structure, as on the Beijing panel below. Dropping the seasonal moving-average factor (the seasonal VAR) costs $1.3\%$ at $h=1$ and, because that factor reaches a full day ahead, a growing ${\approx}5.5\%$ by $h=6$, whereas the ordinary moving average is inert (the dense VARMA tracks the ridge VAR to within half a percent at every horizon). The componentwise baselines, which discard all cross-channel coupling, are weakest throughout, so the coupling helps. A Diebold--Mariano test \citep{diebold1995comparing,harvey1997testing} on the pooled hourly path ($n=52{,}608$) makes every between-family gap significant at $p<0.01$; the two seasonal PAR-shrinkage estimators are a practical tie, separable only by the large sample.

\subsection{Beijing air-quality data: structure versus sparsity}\label{sec:exp_beijing}
Our final benchmark revisits the periodic regime of Section~\ref{sec:exp_singapore} on a larger, denser panel, and uses it to ask a sharper question: when the predictable structure is spread \emph{diffusely} across many channels and lags, does \emph{sparse} estimation help or hurt? We forecast hourly ozone (O$_3$) at the $d=12$ monitoring stations of the Beijing multi-site air-quality network \citep{zhang2017cautionary} over 2013--2019 (stopping before the COVID-19 disruption). Ozone is photochemically driven, with a strong, regular diurnal cycle and pronounced cross-station coupling, so (as in the meteorological panel) its anomalies retain a daily-recurring autocorrelation that the multiplicative seasonal VARMA of Section~\ref{sec:seasonal_varma} is built to capture. What is new here is the cross-station signal, which is high-dimensional, making this a natural test of the sparse VAR/VARMA estimators of the \texttt{bigtime} package \citep{wilms2023sparse} against our dense, shrinkage-based estimators.

\paragraph{Data and protocol.}
Each station's hourly O$_3$ series is preprocessed causally and independently within each window: a missing station-hour is filled by the contemporaneous cross-sectional mean of the other stations; the series are log-transformed; deseasonalized by regressing out, per station on the training window only, an additive calendar design (intercept, hour-of-day indicators, piecewise-linear annual and weekly cycles) together with a \emph{causal linear time trend} that removes the well-documented multi-year decline in Beijing pollution; and standardized by the training-window standard deviation. We use four rolling windows, each three training years and one test year (the last ten months, as the record ends December 2019), aligned to the March start of the data; the preprocessing is redone per window, while order and shrinkage hyperparameters are chosen once on the first window by the common protocol of Section~\ref{sec:exp_protocol} and frozen thereafter, as on the other panels. The seasonal period is $s=24$, the truncation $m=1024$, the covariance prior $\mathrm{IW}(\mI_d,\,d+4)$. To isolate the seasonal moving-average term we add a \emph{seasonal VAR} ablation ($q=Q=0$) whose order is chosen over the balanced grid $(p,P)=(k,k)$, by the same protocol as the seasonal VARMA. For \texttt{bigtime}, whose internal cross-validation is prohibitively slow here, we instead fit its coefficient \emph{path} at each candidate lag (a single fit returns the whole penalty path) and select lag and penalty on the validation tail; its sparse-VAR and sparse-VARMA modes are scored exactly as our own models. The candidate lags are $p\in\{24,48,96\}$ for the sparse VAR and $(p,q)\in\{(24,2),(48,2)\}$ for the sparse VARMA, each fit on the cross-validation window and on the full training window, and the VAR($\infty$) inversion of the sparse VARMA is truncated at $200$ lags before the forecast recursion. Table~\ref{tab:beijing} reports the per-origin forecast MSE, averaged over the four windows.

\begin{table}[!t]\centering\small
\caption{Beijing ozone ($d=12$ hourly station anomalies): forecast MSE per origin (the mean over the $12$ stations), averaged equally over the test origins of the four rolling windows (test periods 2016--2019), at horizons $h=1,\dots,6$ hours. Order and ridge/prior hyperparameters are chosen by the common cross-validation protocol of Section~\ref{sec:exp_protocol} once on the first training window and held fixed across the four windows. Best in each column in bold.}\label{tab:beijing}
\begin{tabular*}{\linewidth}{@{\extracolsep{\fill}}l cccccc@{}}
\toprule
model & $h=1$ & $h=2$ & $h=3$ & $h=4$ & $h=5$ & $h=6$\\
\midrule
\multicolumn{7}{l}{\emph{Seasonal PAR-shrinkage VARMA$(p,p,p,p)_{24}$, $p$ and ridge variances by CV:}}\\
\quad RLS \ ($p{=}2$) & 0.1903 & \textbf{0.3410} & \textbf{0.4584} & \textbf{0.5529} & \textbf{0.6302} & \textbf{0.6917}\\
\quad JMAP \ ($p{=}3$) & \textbf{0.1901} & 0.3412 & 0.4590 & 0.5539 & 0.6311 & 0.6925\\
\midrule
\multicolumn{7}{l}{\emph{Dense VARMA$(p,p)$ (no seasonal factor), order and ridge variances by CV:}}\\
\quad JMAP & 0.1953$^{***}$ & 0.3572$^{***}$ & 0.4871$^{***}$ & 0.5935$^{***}$ & 0.6807$^{***}$ & 0.7501$^{***}$\\
\quad RLS & 0.1951$^{***}$ & 0.3499$^{***}$ & 0.4699$^{***}$ & 0.5660$^{***}$ & 0.6447$^{***}$ & 0.7069$^{***}$\\
\midrule
\multicolumn{7}{l}{\emph{Seasonal VAR$(p,0,P,0)_{24}$ (no MA), order by CV:}}\\
\quad RLS & 0.1953$^{***}$ & 0.3549$^{***}$ & 0.4812$^{***}$ & 0.5834$^{***}$ & 0.6669$^{***}$ & 0.7334$^{***}$\\
\quad JMAP & 0.1953$^{***}$ & 0.3550$^{***}$ & 0.4816$^{***}$ & 0.5840$^{***}$ & 0.6676$^{***}$ & 0.7342$^{***}$\\
\midrule
\multicolumn{7}{l}{\emph{Ridge / Bayesian VAR$(p)$, order and prior by CV:}}\\
\quad ridge VAR & 0.1951$^{***}$ & 0.3543$^{***}$ & 0.4790$^{***}$ & 0.5799$^{***}$ & 0.6626$^{***}$ & 0.7281$^{***}$\\
\quad Bayesian VAR & 0.1955$^{***}$ & 0.3547$^{***}$ & 0.4794$^{***}$ & 0.5803$^{***}$ & 0.6629$^{***}$ & 0.7284$^{***}$\\
\midrule
\multicolumn{7}{l}{\emph{Componentwise univariate baselines, order by CV:}}\\
\quad AR & 0.2197$^{***}$ & 0.4130$^{***}$ & 0.5548$^{***}$ & 0.6602$^{***}$ & 0.7384$^{***}$ & 0.7956$^{***}$\\
\quad ARMA & 0.2163$^{***}$ & 0.4014$^{***}$ & 0.5349$^{***}$ & 0.6336$^{***}$ & 0.7067$^{***}$ & 0.7605$^{***}$\\
\midrule
\multicolumn{7}{l}{\emph{Sparse VARMA package \texttt{bigtime} \citep{wilms2023sparse}, lag and penalty by CV:}}\\
\quad sparse VAR & 0.2577$^{***}$ & 0.4027$^{***}$ & 0.5304$^{***}$ & 0.6361$^{***}$ & 0.7211$^{***}$ & 0.7876$^{***}$\\
\quad sparse VARMA & 0.2508$^{***}$ & 0.4061$^{***}$ & 0.5333$^{***}$ & 0.6385$^{***}$ & 0.7228$^{***}$ & 0.7886$^{***}$\\
\midrule
\multicolumn{7}{l}{\emph{Naive baselines:}}\\
\quad training-set mean & 1.0354$^{***}$ & 1.0354$^{***}$ & 1.0354$^{***}$ & 1.0354$^{***}$ & 1.0354$^{***}$ & 1.0354$^{***}$\\
\quad stay-in-place (RW) & 0.2357$^{***}$ & 0.4764$^{***}$ & 0.6836$^{***}$ & 0.8644$^{***}$ & 1.0214$^{***}$ & 1.1551$^{***}$\\
\bottomrule
\end{tabular*}
\end{table}

The seasonal VARMA is again the most accurate model. The two seasonal PAR-shrinkage estimators lead at every horizon, tied to within $0.2\%$ (JMAP marginally lower at $h=1$, RLS thereafter). The remaining models cluster just above: at $h=1$ the seasonal VARMA beats every non-seasonal alternative (dense VARMA, seasonal VAR, ridge VAR, all nearly coincident) by $2.6$--$2.8\%$, and the margin widens with horizon (to ${\sim}5\%$ over the VAR family and ${\sim}6\%$ over the seasonal VAR by $h=6$), because the seasonal factor reaches a full day ahead by construction. At $h=1$ neither the seasonal VAR (drop the MA) nor the dense VARMA (drop the seasonal factor) improves on the plain ridge VAR, so neither ingredient alone helps: it is the seasonal moving-average factor, acting with the seasonal autoregression, that delivers the gain. As on the meteorological panel, the moving average is inert only in its non-seasonal form (the dense VARMA-RLS tracks the ridge VAR, sparse VAR $\approx$ sparse VARMA), while the seasonal MA is not. The gain again comes cheaply: the seasonal VARMA reaches this accuracy with on the order of $8d^2$ coefficients, whereas a pure autoregression must span a full diurnal cycle to approximate what the seasonal factor encodes in a single low-order term. Averaged over horizons and windows, the seasonal VARMA leads, the dense VARMA-RLS and the VAR family follow nearly coincident, and the componentwise and sparse models trail; the componentwise baselines, which discard cross-station coupling, are again among the weakest, confirming the coupling is informative.

The panel also speaks to sparsity. The \texttt{bigtime} sparse VAR and sparse VARMA are the least accurate multivariate models, and at $h=1$ they fall behind the stay-in-place baseline as well (${\sim}32$--$36\%$ behind the seasonal VARMA at $h=1$, and worse than every dense model at every horizon), and their own cross-validation explains why: on every window it drives the penalty to the least-sparse end of the path and selects the shortest lag offered, fighting the sparsity the method imposes and never reaching the diurnal lags the dense VAR exploits. Ozone's spatiotemporal dependence is diffuse, many small cross-station and multi-lag couplings rather than a few large ones, so an $\ell_1$/hierarchical penalty that zeroes coefficients discards usable signal, whereas the dense $\ell_2$ shrinkage of our estimators (and of the ridge VAR) keeps every coefficient, merely pulled toward zero. Sparsity is the wrong inductive bias here. Two qualifications are in order. First, the comparison mixes two differences: \texttt{bigtime} fits the moving-average part by penalized two-stage least squares rather than by a likelihood-based criterion, so part of its gap to the MAP estimators may reflect statistical efficiency rather than the sparsity penalty itself; the synthetic experiments of Section~\ref{sec:exp_synth} point the same way, with \texttt{bigtime} trailing the $\mSigma$-aware estimators even in the weakly-coupled regime once $\mSigma$ is estimable. Second, the conclusion is about these data, not about sparse methods in general: for data generated by a VARMA with sparse coefficients and a limited sample, the $\ell_1$ approach would likely recover the coefficients, and forecast, more accurately than our dense estimators. With the meteorological panel, this benchmark indicates that on real, periodically-structured, densely-coupled data the most accurate and most parameter-economical forecaster is a structured dense seasonal VARMA rather than a sparse one, stably across four independent pre-COVID years. A Diebold--Mariano test \citep{diebold1995comparing,harvey1997testing} on the pooled path ($n=32{,}720$) confirms every gap: the best seasonal VARMA leads its closest non-seasonal competitor (the dense VARMA-RLS) by $2.1$--$2.6\%$ across horizons, the seasonal VAR by up to $6.0\%$, the componentwise baselines by $10$--$21\%$, and the sparse models by $14$--$36\%$, all at $p<0.01$ (the two seasonal PAR-shrinkage estimators a practical tie).

\section*{Conclusion}
We have presented a scalable framework for estimating dense, truncated VARMA$(p,q,m)$ models. It combines a partial-autocorrelation reparametrization, under which stationarity and invertibility hold by construction so that fitting is unconstrained; regularized least-squares and $\Sigma$-marginalized MAP losses that see the data only through fixed-size sufficient statistics, making every gradient evaluation independent of the series length $T$; and a Fourier/Parseval evaluation that brings the per-gradient cost down to near-linear in the truncation length $m$. Within this framework we fit two estimators (an isotropic-noise regularized least-squares point estimate and the $\Sigma$-marginalized MAP) and extend both, at similar computational cost, to multiplicative seasonal dynamics, to exogenous regressors (VARMAX), and to incremental rolling refits. Empirically the estimators stay close to the oracle on data where the classical conditional-MLE implementation returns non-invertible fits, while remaining tractable well beyond the reach of existing packages, and on retail-demand, meteorological, and air-quality benchmarks they are competitive with or better than VAR, Bayesian-VAR, component-wise ARMA, and sparse-VARMA baselines. In these examples the moving-average and seasonal moving-average terms account for much of the advantage, capturing with a single low-order factor the over-differencing or diurnal structure that a pure autoregression can only approximate with many extra lags; and where the predictable signal is diffuse rather than sparse, the dense structured VARMA is both more accurate and more parameter-economical than its sparse counterpart. A natural further step is full Bayesian inference, averaging forecasts over the $\Sigma$-marginalized posterior rather than reporting its mode. Sampling this posterior efficiently is a challenging open problem (the marginalization leaves a heavy-tailed, high-dimensional density), and we leave it for future research.

\paragraph{Code availability.} The implementation, together with drivers that reproduce every experiment in the paper, is available at \url{https://github.com/paulindani/VARMA}. The code is Python/JAX throughout, with an optional R dependency for the \texttt{bigtime} baseline; each experiment runs from a single script, and a master driver reruns the full study.

\acks{DP is supported by a Nanyang Technological University Start-up Grant, project number: 024968-00001.}

\bibliography{References}

\appendix

\section{Proof of Proposition~\ref{prop:par_bijection}}\label{sec:proof_par}

Throughout this appendix, $\mathrm{chol}(\cdot)$ is the lower-triangular Cholesky factor with positive diagonal. We use two standing facts: (F1) $\mSigma\mapsto\mathrm{chol}(\mSigma)$ and triangular inversion are real-analytic on the positive-definite cone (forward substitution solves $\mL\mL^\top=\mSigma$ entrywise, dividing only by the positive diagonal entries); (F2) a product of lower-triangular matrices with positive diagonals is again one, and is therefore the Cholesky factor of its own Gram matrix, by uniqueness of the factorization.

\paragraph{Step 1: Stage 1 is a diffeomorphism onto the contraction ball.}
For $\mL_A=\mathrm{chol}(\mI+\mM^\top\mM)$ and $\mP=\mM\mL_A^{-\top}$ we have $\mP^\top\mP=\mL_A^{-1}\mM^\top\mM\mL_A^{-\top}=\mL_A^{-1}(\mL_A\mL_A^\top-\mI)\mL_A^{-\top}=\mI-(\mL_A^\top\mL_A)^{-1}$. The eigenvalues of $\mL_A^\top\mL_A$ agree with those of $\mL_A\mL_A^\top=\mI+\mM^\top\mM$, namely $1+\sigma_i(\mM)^2$, so the eigenvalues of $\mP^\top\mP$ are $\sigma_i(\mM)^2/(1+\sigma_i(\mM)^2)<1$ and $\sigma_{\max}(\mP)<1$. Conversely, for $\sigma_{\max}(\mP)<1$ the matrix $\mI-\mP^\top\mP$ is positive definite; setting $\mR=\mathrm{chol}(\mI-\mP^\top\mP)$ and $\mM=\mP\mR^{-\top}$ gives $\mI+\mM^\top\mM=\mR^{-1}\bigl(\mR\mR^\top+\mP^\top\mP\bigr)\mR^{-\top}=\mR^{-1}\mR^{-\top}$, whose Cholesky factor is $\mR^{-1}$ by (F2), so $\mL_A=\mR^{-1}$ and $\mM\mL_A^{-\top}=\mP\mR^{-\top}\mR^{\top}=\mP$: the preimage exists, and it is unique because $\mL_A$ is determined by $\mP$ through $\mL_A^{-1}=\mathrm{chol}(\mI-\mP^\top\mP)$ and then $\mM=\mP\mL_A^\top$. Both directions are compositions of matrix products, triangular inversions and Cholesky factorizations of positive-definite matrices, hence smooth by (F1).

\paragraph{Step 2: the backward pass is well defined, consistent and unique.}
Fix contractions $\mP_{1:r}$ and write $\mB_s:=\mI-\mP_s\mP_s^\top\succ\mz$. The recursion~\eqref{eq:par_backward} defines $\mL^f_r=\mI,\mL^f_{r-1},\ldots,\mL^f_0$, each lower triangular with positive diagonal by (F2), and each invertible, so $\mSigma^f_s:=\mL^f_s(\mL^f_s)^\top\succ\mz$. The forward variance updates are consistent with this chain: by~\eqref{eq:par_backward}, $\mL^f_{s-1}\,\mathrm{chol}(\mB_s)=\mL^f_s$, so $\mL^f_{s-1}\mB_s(\mL^f_{s-1})^\top=\mSigma^f_s$, and in particular the terminal forward covariance equals $\mI_d$: the anchor holds exactly. Uniqueness: if $\mL$ is lower triangular with positive diagonal and $\mL\,\mB_s\,\mL^\top=\mSigma^f_s$, then $\mL\,\mathrm{chol}(\mB_s)=\mathrm{chol}(\mSigma^f_s)$ by (F2), so $\mL=\mL^f_{s-1}$; the backward chain is therefore the unique solution of the anchoring equations, and it depends smoothly on $\mP_{1:r}$ by (F1). Finally, associate to the recursion the matrix sequence $\mGamma_0:=\mL^f_0(\mL^f_0)^\top$ and $\mGamma_s:=\sum_{j=1}^{s}\mA^{(s)}_j\mGamma_{s-j}$ for $s=1,\ldots,r$ (with $\mGamma_{-k}:=\mGamma_k^\top$). With these definitions the coefficient and covariance updates of~\eqref{eq:par_recursion} are precisely the block Levinson recursion, so at every order $s$ the stacks $\mA^{(s)}_{1:s}$ and $\widetilde{\mA}^{(s)}_{1:s}$ solve the forward and backward Yule--Walker systems of $\mGamma_{0:s}$ with residual covariances $\mSigma^f_s$ and $\mSigma^b_s$ \citep{whittle1963fitting,morf1978covariance}; in particular, at $s=r$ the output stack satisfies the order-$r$ Yule--Walker equations of $\mGamma_{0:r}$ with residual $\mSigma^f_r=\mI_d$.

\paragraph{Step 3: every contraction stack maps to a stable stack.}
Associate to the recursion the matrix polynomials $\mA^{(s)}(z):=\mI-\sum_{j=1}^{s}\mA^{(s)}_j z^j$ and the reversed backward polynomials $\mtx{B}^{(s)}(z):=z^{s}\widetilde{\mA}^{(s)}(1/z)$, for which the coefficient updates of~\eqref{eq:par_recursion} read
\[
\mA^{(s)}(z)=\mA^{(s-1)}(z)-z\,\mA^{(s)}_s\,\mtx{B}^{(s-1)}(z),
\qquad
\mtx{B}^{(s)}(z)=z\,\mtx{B}^{(s-1)}(z)-\widetilde{\mA}^{(s)}_s\,\mA^{(s-1)}(z).
\]
Normalize by the Cholesky factors, $\alpha_s(z):=(\mL^f_s)^{-1}\mA^{(s)}(z)$ and $\beta_s(z):=(\mL^b_s)^{-1}\mtx{B}^{(s)}(z)$. Substituting $\mA^{(s)}_s=\mL^f_{s-1}\mP_s(\mL^b_{s-1})^{-1}$, $\widetilde{\mA}^{(s)}_s=\mL^b_{s-1}\mP_s^\top(\mL^f_{s-1})^{-1}$ and the factor updates $\mL^f_s=\mL^f_{s-1}\mC_{f,s}$, $\mL^b_s=\mL^b_{s-1}\mC_{b,s}$ with $\mC_{f,s}:=\mathrm{chol}(\mI-\mP_s\mP_s^\top)$, $\mC_{b,s}:=\mathrm{chol}(\mI-\mP_s^\top\mP_s)$ gives
\[
\alpha_s(z)=\mC_{f,s}^{-1}\bigl(\alpha_{s-1}(z)-z\,\mP_s\,\beta_{s-1}(z)\bigr),
\qquad
\beta_s(z)=\mC_{b,s}^{-1}\bigl(z\,\beta_{s-1}(z)-\mP_s^\top\,\alpha_{s-1}(z)\bigr).
\]
The block map $\mtx{\Theta}_s:=\left(\begin{smallmatrix}\mC_{f,s}^{-1}&-\mC_{f,s}^{-1}\mP_s\\ -\mC_{b,s}^{-1}\mP_s^\top&\mC_{b,s}^{-1}\end{smallmatrix}\right)$ is $J$-unitary, $\mtx{\Theta}_s^\top J\,\mtx{\Theta}_s=J$ for $J=\mathrm{diag}(\mI,-\mI)$: the block identities reduce to $(\mI-\mP\mP^\top)^{-1}-\mP(\mI-\mP^\top\mP)^{-1}\mP^\top=\mI$ and $(\mI-\mP\mP^\top)^{-1}\mP=\mP(\mI-\mP^\top\mP)^{-1}$, both immediate from the Neumann series. Consequently, on $|z|=1$ (where $|z|^2=1$; here $^{*}$ denotes the conjugate transpose of the matrix-valued polynomials, while $^{\top}$ on the real coefficient blocks is the ordinary transpose),
\[
\alpha_s(z)^{*}\alpha_s(z)-\beta_s(z)^{*}\beta_s(z)
=\alpha_{s-1}(z)^{*}\alpha_{s-1}(z)-\beta_{s-1}(z)^{*}\beta_{s-1}(z),
\]
and since $\alpha_0=\beta_0=(\mL^f_0)^{-1}$ is constant, $\alpha_s(z)^{*}\alpha_s(z)=\beta_s(z)^{*}\beta_s(z)$ on the unit circle for every $s$. We now show by induction that $\det\alpha_s(z)\ne0$ for all $|z|\le1$. This holds at $s=0$. Assume it at $s-1$ and set $\mtx{W}(z):=\beta_{s-1}(z)\,\alpha_{s-1}(z)^{-1}$; by the induction hypothesis $\det\alpha_{s-1}$ is nonvanishing on the compact closed disk, hence on an open neighborhood of it, so $\mtx{W}$ is analytic there, and $\mtx{W}(z)^{*}\mtx{W}(z)=\mI$ on $|z|=1$ by the equal-modulus identity. For any unit vectors $u,v$, the scalar $u^{*}\mtx{W}(z)v$ is analytic on a neighborhood of the closed disk and bounded by $1$ on its boundary, so the maximum principle gives $\|\mtx{W}(z)\|\le1$ on the closed disk. Writing
\[
\alpha_s(z)=\mC_{f,s}^{-1}\bigl(\mI-z\,\mP_s\,\mtx{W}(z)\bigr)\,\alpha_{s-1}(z),
\qquad
\|z\,\mP_s\,\mtx{W}(z)\|\le\sigma_{\max}(\mP_s)<1\ \ (|z|\le1),
\]
all three factors are invertible on the closed disk, and the induction advances. At $s=r$ the anchor gives $\mL^f_r=\mI$, so $\mA^{(r)}(z)=\alpha_r(z)$ and $\det\bigl(\mI-\sum_{j}\mA_j z^j\bigr)\ne0$ for $|z|\le1$: the output satisfies the root condition of Assumption~\ref{ass:roots}, equivalently $\spr(\mathcal{C}_{\mA})<1$. This is the matrix form of the classical Schur--Levinson stability argument \citep{whittle1963fitting,morf1978covariance}. Together with Step~2, the sequence $\mGamma_{0:r}$ generated there satisfies the order-$r$ Yule--Walker equations of a stable stack with residual $\mI_d$; since for a stable stack this linear system (with the symmetry constraints $\mGamma_{-k}=\mGamma_k^\top$) is a rewriting of the discrete Lyapunov equation of the companion form and so has a unique solution \citep[Section~2.1]{lutkepohl2005new}, $\mGamma_{0:r}$ coincides with the stationary autocovariances of the $\mathrm{VAR}(r)$ with coefficients $\tau_r(\mM_{1:r})$ and innovation covariance $\mI_d$.

\paragraph{Step 4: correspondence with a stationary autoregression.}
Let $\mA_{1:r}$ be stable and let $\vz_t$ be the stationary $\mathrm{VAR}(r)$ process with coefficients $\mA_{1:r}$ and innovation covariance $\mI_d$. Its autocovariances $\mGamma_k$ exist and are unique: the stacked covariance of $(\vz_t,\ldots,\vz_{t-r+1})$ solves a discrete Lyapunov equation with the stable companion matrix $\mathcal{C}_{\mA}$ as transition and state-noise covariance $\mtx{Q}:=\mathrm{diag}(\mI_d,\mz)$, and the solution $\sum_{k\ge0}\mathcal{C}_{\mA}^k\mtx{Q}(\mathcal{C}_{\mA}^\top)^k$ is positive definite because the pair $(\mathcal{C}_{\mA},\mtx{Q})$ is reachable: $\mtx{Q}$ has rank $d$ only, but a shock entering the first block propagates through every lag block, so the first $r$ propagation blocks span all $rd$ state coordinates. For $0\le s\le r$ let $\mSigma^f_s$ and $\mSigma^b_s$ denote the covariances of the errors of the best linear prediction of $\vz_t$ from its $s$ most recent past values and of $\vz_{t-s}$ from the same block of values, respectively, and let $\mtx{\Delta}_s$ be the associated partial cross-covariance. The multivariate Levinson--Durbin recursion computes these quantities from $\mGamma_{0:r}$, and its coefficient updates coincide with the middle lines of~\eqref{eq:par_recursion} upon setting
\[
\mP_s \;:=\; (\mL^f_{s-1})^{-1}\,\mtx{\Delta}_s\,(\mL^b_{s-1})^{-\top},
\qquad
\mL^f_{s-1}=\mathrm{chol}(\mSigma^f_{s-1}),\quad \mL^b_{s-1}=\mathrm{chol}(\mSigma^b_{s-1}),
\]
since then $\mA^{(s)}_s=\mL^f_{s-1}\mP_s(\mL^b_{s-1})^{-1}=\mtx{\Delta}_s(\mSigma^b_{s-1})^{-1}$, the usual reflection-coefficient update. The joint covariance $\left(\begin{smallmatrix}\mSigma^f_{s-1}&\mtx{\Delta}_s\\ \mtx{\Delta}_s^\top&\mSigma^b_{s-1}\end{smallmatrix}\right)$ of the two prediction errors is positive definite (it is the Schur complement of the covariance of the conditioning block $(\vz_{t-1},\ldots,\vz_{t-s+1})$ in the positive-definite covariance of the full vector $(\vz_t,\ldots,\vz_{t-s})$), and after congruence by $\mathrm{diag}\bigl((\mL^f_{s-1})^{-1},(\mL^b_{s-1})^{-1}\bigr)$ its positive-definiteness is, by the Schur complement, equivalent to $\sigma_{\max}(\mP_s)<1$. The recursion's variance downdates are $\mSigma^f_s=\mL^f_{s-1}(\mI-\mP_s\mP_s^\top)(\mL^f_{s-1})^\top$ and $\mSigma^b_s=\mL^b_{s-1}(\mI-\mP_s^\top\mP_s)(\mL^b_{s-1})^\top$, matching Step~2 and the last line of~\eqref{eq:par_recursion}. For a $\mathrm{VAR}(r)$ the best linear predictor from $r$ lags is the autoregression itself, so $\mSigma^f_r=\mI_d$ and the order-$r$ predictor coefficients are $\mA_{1:r}$; by stationarity, $\mSigma^b_0=\mSigma^f_0=\mGamma_0$. The contraction stack extracted from $\mA_{1:r}$ therefore satisfies the anchoring equations of Step~2, whose solution is unique, and running the forward pass~\eqref{eq:par_recursion} on it returns the order-$r$ predictor coefficients, that is, $\mA_{1:r}$ itself. This gives surjectivity. Injectivity and smoothness now follow together. The stable coefficients and innovation covariance $\mI_d$ determine the Lyapunov solution $\mGamma_{0:r}$ uniquely, and by the last part of Step~3 this is the sequence generated in Step~2 from the input contractions; the Levinson extraction from $\mGamma_{0:r}$ is deterministic and, by the anchoring uniqueness of Step~2, returns those contractions, so with Step~1 the preimage is unique, and the inverse is the composition $\mA\mapsto\mGamma(\mA)\mapsto\mP(\mGamma)\mapsto\mM(\mP)$. Each constituent map is smooth on the open stable set: the first is the solution of a linear system with invertible, polynomially $\mA$-dependent coefficient matrix; the second stays on positive-definite covariances and applies (F1) at every order; the third is the closed form of Step~1. The forward direction is smooth by Steps~1--2, so $\tau_r$ is a smooth diffeomorphism onto the stable set. \hfill$\square$

We record one consequence used in Assumption~\ref{ass:consistency}(ii) and in Appendix~\ref{sec:pf_sarma}: padding a stable stack $\mA_{1:\bar p}$ with zero matrices up to order $p\ge\bar p$ leaves the lag polynomial unchanged, and the order-$p$ companion matrix has characteristic polynomial $z^{(p-\bar p)d}\det(z^{\bar p}\mI-\ldots)$, so its spectrum is that of the original companion matrix together with $(p-\bar p)d$ zero eigenvalues. Stability is preserved, and Proposition~\ref{prop:par_bijection} applied at order $p$ gives the padded stack unique order-$p$ PAR coordinates.

\section{Proof of Proposition~\ref{prop:trunc_bias}}\label{sec:proof_trunc_bias}
Throughout, $\|\cdot\|$ denotes the operator ($\ell_2\!\to\!\ell_2$) norm (which coincides with the $\|\cdot\|_{\mathrm{op}}$ of Proposition~\ref{prop:trunc_bias}), $\sigma_{\min}(\cdot)$ the smallest singular value, and $\mtx{M}^{*}$ the conjugate transpose of $\mtx{M}$. Introduce the infinite-VAR operator of the VARMA$(p,q)$ and its degree-$m$ truncation,
\[
\mtx{A}(z):=\mI-\sum_{j\ge1}\mPi_j z^j=\mTheta(z)^{-1}\mPhi(z),
\qquad
\mtx{A}_m(z):=\mI-\sum_{j=1}^{m}\mPi_j z^j,
\]
where the second equality is the operator form of the infinite-VAR representation~\eqref{eq:infiniteVAR}, obtained from the VARMA equation~\eqref{eq:VARMAdef} (so $\mPi_0=-\mI$ and $a_t=\mtx{A}(B)\vZ_t$), and $\mtx{A}_m$ is exactly the autoregressive operator of the truncated VARMA$(p,q,m)$ of Definition~\ref{def:tvarma}. We use the constants $G_\rho,\mu_{\min},\beta_\bullet,\beta_\circ,m_0$ of Proposition~\ref{prop:trunc_bias}. Under Assumption~\ref{ass:roots}, $\det\mtx{A}(z)=\det\mPhi(z)/\det\mTheta(z)\ne0$ for all $|z|\le1$ and $\mtx{A}$ is continuous on the compact disk, so $\beta_\bullet,\beta_\circ$ are both strictly positive, with $\beta_\bullet\le\beta_\circ$ and $\beta_\circ=\sqrt{\mu_{\min}}$.

\noindent\emph{Step 1 (geometric decay of the dropped coefficients, via a Cauchy estimate).} The zeros of $\det\mTheta(z)$ are the reciprocals of the nonzero eigenvalues of $\mathcal{C}_{\mTheta}$, hence have modulus $\ge 1/\spr(\mathcal{C}_{\mTheta})$; since $\mPhi$ is a matrix polynomial, $\mI-\mtx{A}(z)=\sum_{j\ge1}\mPi_j z^j$ is analytic on the disk $|z|<1/\spr(\mathcal{C}_{\mTheta})$. Fix $\rho\in(\spr(\mathcal{C}_{\mTheta}),1)$, so the circle $|z|=1/\rho$ lies strictly inside that disk. Cauchy's estimate for Taylor coefficients then gives, for every $j\ge1$,
\[
\mPi_j=\frac{1}{2\pi i}\oint_{|z|=1/\rho}\bigl(\mI-\mtx{A}(z)\bigr)\,z^{-j-1}\,dz,
\qquad\text{hence}\qquad
\|\mPi_j\|\le\rho^{\,j}\max_{|z|=1/\rho}\|\mI-\mtx{A}(z)\|=G_\rho\,\rho^{\,j},
\]
with $G_\rho$ as in~\eqref{eq:trunc_constants}, finite because $\mTheta(z)$ is invertible on $|z|=1/\rho<1/\spr(\mathcal{C}_{\mTheta})$. This is the explicit, non-asymptotic counterpart of Gelfand's formula~\eqref{eq:gelfand}, whose limit only asserts that such a constant exists \citep{trefethen2005spectra}. Consequently, for every $m\ge0$,
\begin{equation}\label{eq:tail_bound}
\delta_m:=\sum_{j>m}\|\mPi_j\|\;\le\;G_\rho\sum_{j>m}\rho^{j}=\frac{G_\rho\,\rho}{1-\rho}\,\rho^{m}.
\end{equation}

\noindent\emph{Step 2 (stationarity of the truncation and a uniform inverse bound).} For $|z|\le1$, $\|\mtx{A}(z)-\mtx{A}_m(z)\|=\bigl\|\sum_{j>m}\mPi_j z^{j}\bigr\|\le\delta_m$ by~\eqref{eq:tail_bound}. Since $\delta_m=G_\rho\rho^{m+1}/(1-\rho)$, the threshold $m_0$ of~\eqref{eq:trunc_m0} is exactly the smallest $m\ge p+q$ with $\delta_m\le\beta_\bullet/2$; hence $\delta_m\le\beta_\bullet/2\le\beta_\circ/2$ for all $m\ge m_0$. For such $m$, Weyl's inequality gives, on the closed disk, $\sigma_{\min}(\mtx{A}_m(z))\ge\beta_\bullet-\delta_m\ge\beta_\bullet/2>0$, so $\det\mtx{A}_m(z)\ne0$ on $|z|\le1$ and the truncated VARMA$(p,q,m)$ is stationary, with a well-defined spectral density; and on the unit circle, $\sigma_{\min}(\mtx{A}_m(e^{-i\omega}))\ge\beta_\circ-\delta_m\ge\beta_\circ/2$, whence $\|\mtx{A}_m(e^{-i\omega})^{-1}\|\le2/\beta_\circ$.

\noindent\emph{Step 3 (spectral-density gap, controlled by the on-circle conditioning).} For $m\ge m_0$ both processes are stationary, driven by the same white noise of covariance $\mSigma$, with matrix spectral densities
\[
\mtx{f}(\omega)=\tfrac{1}{2\pi}\,\mtx{A}(e^{-i\omega})^{-1}\mSigma\,\mtx{A}(e^{-i\omega})^{-*},\qquad
\mtx{f}_m(\omega)=\tfrac{1}{2\pi}\,\mtx{A}_m(e^{-i\omega})^{-1}\mSigma\,\mtx{A}_m(e^{-i\omega})^{-*}.
\]
Abbreviate $\mtx{A}:=\mtx{A}(e^{-i\omega})$, $\mtx{A}_m:=\mtx{A}_m(e^{-i\omega})$; on the unit circle $\|\mtx{A}^{-1}\|=\|\mtx{A}^{-*}\|\le\beta_\circ^{-1}$ and $\|\mtx{A}_m^{-1}\|\le2/\beta_\circ$ by Step~2. The resolvent identity $\mtx{A}^{-1}-\mtx{A}_m^{-1}=\mtx{A}^{-1}(\mtx{A}_m-\mtx{A})\mtx{A}_m^{-1}$ gives $\|\mtx{A}^{-1}-\mtx{A}_m^{-1}\|\le2\beta_\circ^{-2}\delta_m$, and writing
\[
\mtx{A}^{-1}\mSigma\mtx{A}^{-*}-\mtx{A}_m^{-1}\mSigma\mtx{A}_m^{-*}=(\mtx{A}^{-1}-\mtx{A}_m^{-1})\mSigma\mtx{A}^{-*}+\mtx{A}_m^{-1}\mSigma(\mtx{A}^{-*}-\mtx{A}_m^{-*})
\]
yields, uniformly in $\omega\in[-\pi,\pi]$,
\[
\|\mtx{f}(\omega)-\mtx{f}_m(\omega)\|\le\frac{1}{2\pi}\Bigl(2\beta_\circ^{-2}\delta_m\,\|\mSigma\|\,\beta_\circ^{-1}+2\beta_\circ^{-1}\,\|\mSigma\|\,2\beta_\circ^{-2}\delta_m\Bigr)=\frac{3\|\mSigma\|}{\pi\beta_\circ^{3}}\,\delta_m.
\]
(The role of $\beta_\circ^{-1}=\mu_{\min}^{-1/2}$ as the on-circle conditioning is discussed in the remark below.)

\noindent\emph{Step 4 (autocovariances).} The autocovariances are the Fourier coefficients of the spectral densities, $\mGamma^{\infty}_l=\int_{-\pi}^{\pi}e^{il\omega}\mtx{f}(\omega)\,d\omega$ and $\mGamma^{(m)}_l=\int_{-\pi}^{\pi}e^{il\omega}\mtx{f}_m(\omega)\,d\omega$. Hence, for every $l\ge0$ and every $m\ge m_0$, using Step~3, $\beta_\circ^{3}=\mu_{\min}^{3/2}$, and~\eqref{eq:tail_bound},
\[
\bigl\|\mGamma^{\infty}_l-\mGamma^{(m)}_l\bigr\|\le\int_{-\pi}^{\pi}\|\mtx{f}(\omega)-\mtx{f}_m(\omega)\|\,d\omega\le\frac{6\|\mSigma\|}{\mu_{\min}^{3/2}}\,\delta_m\le\frac{6\|\mSigma\|\,G_\rho\,\rho}{(1-\rho)\,\mu_{\min}^{3/2}}\,\rho^{m}=C_\rho\,\rho^{m}.
\]
This is exactly~\eqref{eq:gamma_bias} with the explicit constant stated there, uniform in $l$, completing the proof. $\hfill\square$

\begin{remark}[Dimension dependence]
The two constants are explicit but model-dependent: $G_\rho$ reflects the non-normality of the moving-average inversion $\mTheta^{-1}\mPhi$, and $\mu_{\min}$ is the on-circle floor of the squared smallest singular value of the autoregressive operator $\mtx{A}(e^{-i\omega})$; together with $\|\mSigma\|$ it yields the spectral-density upper bound of \citet{basu2015regularized} (it is not itself an eigenvalue floor of the spectral density); neither carries an explicit dimension factor. A fully a priori, dimension-explicit decay constant is instead available from the Kreiss matrix theorem, $\|\mathcal{C}_{\mTheta}^{k}\|\le e\,(qd)\,K(\mathcal{C}_{\mTheta}/\rho)\,\rho^{k}$ with $K(\cdot)$ the Kreiss constant \citep{spijker1991}, where the linear factor $qd$ (the order of $\mathcal{C}_{\mTheta}$) is unavoidable in the worst case; the resolvent constant $G_\rho$ above is the sharper, model-adapted alternative.
\end{remark}

\section{Proof of Theorem~\ref{thm:pi_rate}}\label{sec:proof_pi_rate}

Throughout, $\|\cdot\|$ is the operator norm, $\|\cdot\|_F$ the Frobenius norm, $N=T-m$, and $C,c'>0$ denote constants depending only on the quantities listed in Theorem~\ref{thm:pi_rate}, whose values may change from occurrence to occurrence; the constants $c$, $m_1$, $T_0$ of the theorem are fixed by finitely many such requirements collected along the proof. We abbreviate $\Lambda:=\Lambda_\delta$ and use the standing conditions $T\ge T_0$, $m_1\le m\le T/2$ (so $N\ge T/2$) and $\Lambda\le cT$ without further comment. Because $B$, $p$, $q$, $d$, and the true-model quantities of Assumption~\ref{ass:consistency} are fixed, every quantity introduced in Lemmas~\ref{lem:uniform_geom} and~\ref{lem:filtered} ($\rho_\ast$, $\bar G$, $\bar L$, $\bar\Pi$, $\bar q$, and their companions) is a constant independent of $T$ and $m$ throughout this appendix.

 For $\vartheta\in K_B$ let $(\mPhi_{1:p}(\vartheta),\mTheta_{1:q}(\vartheta)):=\mathcal{T}(\vartheta)$, let $\mPi_j(\vartheta)$ be the induced infinite-VAR coefficients~\eqref{eq:mPirec}--\eqref{eq:mPirec2}, and, as in Appendix~\ref{sec:proof_trunc_bias},
\[
\mtx{A}(z;\vartheta):=\mI-\sum_{j\ge1}\mPi_j(\vartheta)z^j=\mTheta(z;\vartheta)^{-1}\mPhi(z;\vartheta).
\]
Write $\mPi^0_j:=\mPi_j(\vartheta^0)$, $\mtx{A}^0:=\mtx{A}(\cdot\,;\vartheta^0)$ and $\Delta_j(\vartheta):=\mPi_j(\vartheta)-\mPi^0_j$. Zero-padding does not change the lag polynomials, so $\mPi^0_{1:\infty}$ are exactly the infinite-VAR coefficients of the data-generating process, $\va_t=\mtx{A}^0(B)\vz_t$. Let $\mathcal{F}_t:=\sigma(\va_s:s\le t)$; under Assumption~\ref{ass:roots} the stationary solution is causal, so $\vz_t$ is $\mathcal{F}_t$-measurable and $\va_{t+1}$ is independent of $\mathcal{F}_t$.

\subsection{Preliminaries}\label{sec:pf_prelim}

\begin{lemma}[Uniform geometry on $K_B$]\label{lem:uniform_geom}
There exist $\rho_\ast\in(0,1)$ and $\bar G,\bar L<\infty$, depending only on $(B,p,q,d)$, such that for all $\vartheta,\vartheta'\in K_B$:
(i) $\|\mPi_j(\vartheta)\|\le\bar G\rho_\ast^{\,j}$ for every $j\ge1$, so $\sum_{j\ge1}\|\mPi_j(\vartheta)\|\le\bar\Pi:=\bar G\rho_\ast/(1-\rho_\ast)$;
(ii) $\sum_{j\ge1}\|\mPi_j(\vartheta)-\mPi_j(\vartheta')\|_F\le\bar L\,\|\vartheta-\vartheta'\|_2$.
\end{lemma}

\begin{proof}
$\mathcal{T}$ is a smooth diffeomorphism onto the open stable set (Proposition~\ref{prop:par_bijection}) and $\spr(\cdot)$ is continuous, so $\bar\rho:=\max_{\vartheta\in K_B}\max\{\spr(\mathcal{C}_{\mPhi(\vartheta)}),\spr(\mathcal{C}_{\mTheta(\vartheta)})\}$ is attained on the compact ball and is $<1$; set $\rho_\ast:=(1+\bar\rho)/2$. On the compact set $\{|z|\le1/\rho_\ast\}\times K_B$ the polynomial $\det\mTheta(z;\vartheta)$ is continuous and zero-free (its zeros have modulus at least $1/\bar\rho>1/\rho_\ast$), so $\|\mTheta(z;\vartheta)^{-1}\|$ is uniformly bounded there, $(z,\vartheta)\mapsto\mtx{A}(z;\vartheta)=\mTheta(z;\vartheta)^{-1}\mPhi(z;\vartheta)$ is $C^1$, and
\[
\bar G:=\max_{|z|=1/\rho_\ast,\,\vartheta\in K_B}\bigl\|\mI-\mtx{A}(z;\vartheta)\bigr\|,
\qquad
L_0:=\max_{|z|=1/\rho_\ast,\,\vartheta\in K_B}\bigl\|\partial_\vartheta\mtx{A}(z;\vartheta)\bigr\|
\]
are finite, $\partial_\vartheta\mtx{A}$ denoting the differential of $\vartheta\mapsto\mtx{A}(z;\vartheta)$ as a linear map from $(\RR^{n_{\mathrm{par}}},\|\cdot\|_2)$ to $(\C^{d\times d},\|\cdot\|_F)$. The Cauchy coefficient estimate of Step~1 of Appendix~\ref{sec:proof_trunc_bias}, on the circle $|z|=1/\rho_\ast$, gives (i); applied to $\mtx{A}(\cdot\,;\vartheta)-\mtx{A}(\cdot\,;\vartheta')$ together with the mean value inequality along the segment $[\vartheta,\vartheta']$ in the convex ball $K_B$, it gives $\|\mPi_j(\vartheta)-\mPi_j(\vartheta')\|_F\le\rho_\ast^{\,j}L_0\|\vartheta-\vartheta'\|_2$, and summing the geometric series yields (ii) with $\bar L:=L_0\rho_\ast/(1-\rho_\ast)$.
\end{proof}

\begin{lemma}[Second moments of filtered processes]\label{lem:filtered}
Set
\[
c_0:=\frac{\lambda_{\min}(\mSigma)}{(1+\bar\Pi)^2},\qquad
C_0:=\frac{\|\mSigma\|}{\mu_{\min}},
\]
with $\mu_{\min}$ the on-circle constant of Proposition~\ref{prop:trunc_bias} for the true model. Let $\{\mtx{B}_j\}_{j\ge1}$ satisfy $\sum_j\|\mtx{B}_j\|<\infty$ and $\vs_t:=\sum_{j\ge1}\mtx{B}_j\vz_{t-j}$. Then
(a) $c_0\sum_j\|\mtx{B}_j\|_F^2\le\E\|\vs_t\|^2\le C_0\sum_j\|\mtx{B}_j\|_F^2$; and
(b) for any $n,t_1$, the covariance matrix of the stacked Gaussian vector $(\vs_{t_1},\dots,\vs_{t_1+n-1})$ has operator norm at most $C_0\bigl(\sum_j\|\mtx{B}_j\|\bigr)^2$.
\end{lemma}

\begin{proof}
The spectral density of $\vz$ is $\mtx{f}(\omega)=\frac{1}{2\pi}\mtx{A}^0(e^{-i\omega})^{-1}\mSigma\,\mtx{A}^0(e^{-i\omega})^{-\ast}$ (Appendix~\ref{sec:proof_trunc_bias}, Step~3), with on-circle eigenvalue bounds $\lambda_{\min}(\mtx{f})\ge\lambda_{\min}(\mSigma)/\bigl(2\pi\|\mtx{A}^0\|^2\bigr)\ge c_0/(2\pi)$, since $\max_{|z|=1}\|\mtx{A}^0(z)\|\le1+\bar\Pi$ by Lemma~\ref{lem:uniform_geom}(i), and $\lambda_{\max}(\mtx{f})\le\|\mSigma\|/(2\pi\mu_{\min})=C_0/(2\pi)$, the bound of \citet{basu2015regularized}. The filtered process is stationary Gaussian with spectral density $\mtx{f}_{\vs}=\widetilde{\mtx{B}}\,\mtx{f}\,\widetilde{\mtx{B}}^\ast$, where $\widetilde{\mtx{B}}(\omega):=\sum_j\mtx{B}_je^{-ij\omega}$. Part~(a) is Parseval's identity $\frac{1}{2\pi}\int_{-\pi}^{\pi}\|\widetilde{\mtx{B}}\|_F^2\,d\omega=\sum_j\|\mtx{B}_j\|_F^2$ combined with the two-sided eigenvalue bounds on $\mtx{f}$ inside $\E\|\vs_t\|^2=\int_{-\pi}^{\pi}\tr\mtx{f}_{\vs}\,d\omega$. Part~(b): the stacked covariance is a principal submatrix of the bi-infinite block-Toeplitz covariance of $\vs$, whose operator norm is at most $2\pi\,\mathrm{ess\,sup}_\omega\lambda_{\max}(\mtx{f}_{\vs})\le\bigl(\sum_j\|\mtx{B}_j\|\bigr)^2C_0$ \citep[Proposition~2.3]{basu2015regularized}.
\end{proof}

We use two standard probabilistic facts. First \citep[Lemma~1]{laurent2000adaptive}: if $Q=\sum_i\lambda_i\xi_i^2$ with fixed $\lambda_i\ge0$ and $\xi_i$ i.i.d.\ $N(0,1)$, then, writing $\Sigma_1:=\sum_i\lambda_i$, $\Sigma_2:=\sum_i\lambda_i^2\le\lambda_{\max}\Sigma_1$,
\begin{equation}\label{eq:LM}
\PP\bigl(Q\ge\Sigma_1+2\sqrt{\Sigma_2x}+2\lambda_{\max}x\bigr)\le e^{-x},
\qquad
\PP\bigl(Q\le\Sigma_1-2\sqrt{\Sigma_2x}\bigr)\le e^{-x},\qquad x>0.
\end{equation}
Second, a scalar self-normalized martingale inequality in the spirit of \citet[Theorem~1]{abbasiyadkori2011} (see also \citealp{penalaishao2009}): if $\xi_t$ is $\mathcal{F}_{t+1}$-measurable with $\E[e^{\lambda\xi_t}\mid\mathcal{F}_t]\le e^{\lambda^2s_t^2/2}$ for all $\lambda\in\RR$ and $\mathcal{F}_t$-measurable $s_t$, then for every $\delta'\in(0,1)$, with probability at least $1-\delta'$,
\begin{equation}\label{eq:selfnorm}
\Bigl|\sum_{t}\xi_t\Bigr|\;\le\;\sqrt{2\,(1+V)\,\log\bigl(\sqrt{1+V}/\delta'\bigr)},
\qquad V:=\sum_{t}s_t^2,
\end{equation}
the sums running over any fixed index window. This scalar specialization of the method of mixtures follows by integrating the exponential supermartingale $\exp(\lambda\sum_t\xi_t-\lambda^2V/2)$ against a standard Gaussian in $\lambda$ (the unit mixing variance is the source of the two $1+V$ factors) and applying Markov's inequality to the resulting supermartingale; see \citet{abbasiyadkori2011} and \citet{penalaishao2009}.

\subsection{Decomposition and high-probability events}\label{sec:pf_events}

For $t=m,\dots,T-1$ define
\[
\vv_t(\vartheta):=\sum_{j=1}^{m}\bigl(\mPi^0_j-\mPi_j(\vartheta)\bigr)\vz_{t+1-j},
\qquad
\vw_t:=\sum_{j>m}\mPi^0_j\vz_{t+1-j},
\]
so that the residual $\vr_t(\vartheta)=\vz_{t+1}-\sum_{j\le m}\mPi_j(\vartheta)\vz_{t+1-j}$ appearing in $L_m$ and $\mR$ decomposes exactly as $\vr_t(\vartheta)=\va_{t+1}+\vv_t(\vartheta)+\vw_t$, with $\vv_t(\vartheta^0)=\vct{0}$. Set
\[
M(\vartheta):=\sum_t\vv_t(\vartheta)^{\!\top}\va_{t+1},\quad
\mtx{C}_a(\vartheta):=\sum_t\vv_t(\vartheta)\va_{t+1}^{\top},\quad
W(\vartheta):=\sum_t\vv_t(\vartheta)^{\!\top}\vw_t,
\]
\[
Q_v(\vartheta):=\sum_t\|\vv_t(\vartheta)\|^2,\quad
Q_w:=\sum_t\|\vw_t\|^2,\quad
Q_a:=\sum_t\|\va_{t+1}\|^2,
\]
\[
q_m(\vartheta):=\E\|\vv_t(\vartheta)\|^2,\qquad
q_w:=\E\|\vw_t\|^2.
\]
Expanding the squares and outer products yields the exact identities
\begin{align}
L_m(\vartheta)-L_m(\vartheta^0)&=2M(\vartheta)+2W(\vartheta)+Q_v(\vartheta),\label{eq:excess_id}\\
\mR(\vartheta)-\mR(\vartheta^0)&=\mtx{V}(\vartheta)+\mtx{C}(\vartheta)+\mtx{C}(\vartheta)^{\top},
\qquad
\mtx{V}:=\sum_t\vv_t\vv_t^{\top},\quad
\mtx{C}:=\mtx{C}_a+\sum_t\vv_t\vw_t^{\top}.\label{eq:R_id}
\end{align}
Note $M=\tr\mtx{C}_a$ (so $|M|\le\sqrt d\,\|\mtx{C}_a\|_F$), $\tr\mtx{V}=Q_v$, and, pathwise by Cauchy--Schwarz, $|W|\le\sqrt{Q_wQ_v}$ and $\|\sum_t\vv_t\vw_t^{\top}\|_F\le\sqrt{Q_vQ_w}$. By Lemmas~\ref{lem:uniform_geom} and~\ref{lem:filtered}, applied with $\mtx{B}_j=(\mPi^0_j-\mPi_j(\vartheta))\mathbf{1}\{j\le m\}$ and $\mtx{B}_j=\mPi^0_j\mathbf{1}\{j>m\}$ respectively, we have for all $\vartheta\in K_B$
\begin{gather}
c_0\sum_{j\le m}\|\Delta_j(\vartheta)\|_F^2\;\le\;q_m(\vartheta)\;\le\;C_0\sum_{j\le m}\|\Delta_j(\vartheta)\|_F^2\;\le\;\bar q:=\frac{4dC_0\bar G^2}{1-\rho_\ast^2},\label{eq:qm_sandwich}\\
\lambda_{\max}\Bigl(\mathrm{Cov}\bigl(\vv_m(\vartheta),\dots,\vv_{T-1}(\vartheta)\bigr)\Bigr)\;\le\;C_0\Bigl(\sum_{j\le m}\|\Delta_j\|\Bigr)^{\!2}\;\le\;4C_0\bar\Pi^2\;=:\;\bar\lambda_v,\label{eq:v_spec}\\
q_w\le C\rho_\ast^{2m},\qquad
\lambda_{\max}\Bigl(\mathrm{Cov}\bigl(\vw_m,\dots,\vw_{T-1}\bigr)\Bigr)\le C\rho_\ast^{2m},\label{eq:w_bounds}
\end{gather}
using $\|\Delta_j\|\le\|\Delta_j\|_F\le\sqrt d\,\|\Delta_j\|\le2\sqrt d\,\bar G\rho_\ast^{\,j}$ and $\sum_{j\ge1}\|\Delta_j\|\le2\bar\Pi$ in~\eqref{eq:v_spec}.

Fix a $\tau_T$-net $\mathcal{N}\subset K_B$ with $\tau_T:=T^{-4}$ and $|\mathcal{N}|\le(3B/\tau_T)^{n_{\mathrm{par}}}$, so that $\log|\mathcal{N}|\le5\,n_{\mathrm{par}}\log T$ for $T\ge T_0$, and set $x_\ast:=\log\bigl(10\,d^2\,|\mathcal{N}|\,T/\delta\bigr)\le C\Lambda$. Let $H:=H_1\cap\dots\cap H_5$, where each event below has probability at least $1-\delta/5$ (so $\PP(H)\ge1-\delta$):
\begin{itemize}
\item[$H_1$:] $\max_{1\le t\le T}\|\vz_t\|^2\vee\max_{m<t\le T}\|\va_t\|^2\le Z^2:=C\,(1+\log(T/\delta))$. \emph{[Each $\|\vz_t\|^2$ (resp.\ $\|\va_t\|^2$) is a Gaussian quadratic form with $\Sigma_1\le dC_0$ (resp.\ $\tr\mSigma$); apply the upper tail of~\eqref{eq:LM} with $x=\log(10T/\delta)$ and a union bound over the at most $2T$ vectors.]}
\item[$H_2$:] for all $\vartheta\in\mathcal{N}$: $\tfrac12\,Nq_m(\vartheta)-C\Lambda\le Q_v(\vartheta)\le2\,Nq_m(\vartheta)+C\Lambda$. \emph{[\eqref{eq:LM} for the Gaussian quadratic form $Q_v(\vartheta)$, with $\Sigma_1=Nq_m(\vartheta)$ and $\lambda_{\max}\le\bar\lambda_v$ by~\eqref{eq:v_spec}, at $x=x_\ast\le C\Lambda$; absorb the deviation via $2\sqrt{ab}\le\tfrac12a+2b$; union over $\mathcal{N}$.]}
\item[$H_3$:] $Q_w\le C\rho_\ast^{2m}N$ and $Q_a\le2N\tr\mSigma$. \emph{[Upper tail of~\eqref{eq:LM} with~\eqref{eq:w_bounds}, $x=x_\ast\le C\Lambda\le2CcN$.]}
\item[$H_4$:] for all $\vartheta\in\mathcal{N}$ and $1\le k,l\le d$:
\[
\bigl|(\mtx{C}_a(\vartheta))_{kl}\bigr|\le\sqrt{2\bigl(1+\|\mSigma\|Q_v(\vartheta)\bigr)\bigl(\log\sqrt{1+\|\mSigma\|Q_v(\vartheta)}+x_\ast\bigr)}.
\]
\emph{[$(\mtx{C}_a)_{kl}=\sum_t(\vv_t)_k(\va_{t+1})_l$; conditionally on $\mathcal{F}_t$ the summand is exactly $N\bigl(0,(\vv_t)_k^2\mSigma_{ll}\bigr)$ with $(\vv_t)_k$ $\mathcal{F}_t$-measurable, so~\eqref{eq:selfnorm} applies with $V=\mSigma_{ll}\sum_t(\vv_t)_k^2\le\|\mSigma\|Q_v(\vartheta)$; union over the $d^2|\mathcal{N}|$ pairs.]}
\item[$H_5$:] $\bigl\|\sum_t\va_{t+1}\va_{t+1}^{\top}-N\mSigma\bigr\|_F\le C\sqrt{N\Lambda}$. \emph{[Entrywise by~\eqref{eq:LM} at $x=x_\ast\le C\Lambda\le CN$. The diagonal entries are Gaussian quadratic forms; for $k\ne l$, polarization writes the summand as a difference of two such forms, $(\va_{t+1})_k(\va_{t+1})_l=\tfrac14\bigl[((\va_{t+1})_k+(\va_{t+1})_l)^2-((\va_{t+1})_k-(\va_{t+1})_l)^2\bigr]$, and~\eqref{eq:LM} applies to each. Every entry then deviates by at most $C(\sqrt{Nx_\ast}+x_\ast)\le C\sqrt{N\Lambda}$, and a union bound over the $d^2$ entries with $\|\cdot\|_F\le d\max_{k,l}|\cdot_{kl}|$ gives the event. $H_5$ is not unioned over $\mathcal{N}$, so a tail parameter $C\log(d/\delta)$ would suffice; $x_\ast$ is kept for uniform bookkeeping.]}
\end{itemize}

\paragraph{From the net to $K_B$.}
Let $\vartheta\in K_B$ and $\vartheta'\in\mathcal{N}$ with $\|\vartheta-\vartheta'\|_2\le\tau_T$. On $H_1\cap H_3$, every functional appearing in $H_2$--$H_4$ (namely $M$, $\mtx{C}_a$, $Q_v$, $W$, $\sum_t\vv_t\vw_t^{\top}$ and $Nq_m$) changes by at most $1$ between $\vartheta$ and $\vartheta'$, for $T\ge T_0$. One representative case: by Lemma~\ref{lem:uniform_geom}(ii), $\|\vv_t(\vartheta)-\vv_t(\vartheta')\|\le\bar L\,\tau_TZ$ and $\|\vv_t(\vartheta)\|\le2\bar\Pi Z$ on $H_1$, so, explicitly, $|Q_v(\vartheta)-Q_v(\vartheta')|\le\sum_t\bigl(\|\vv_t(\vartheta)\|+\|\vv_t(\vartheta')\|\bigr)\|\vv_t(\vartheta)-\vv_t(\vartheta')\|\le CTZ^2\bar\Pi\bar L\,\tau_T\le C\Lambda\,T^{-3}\le T^{-1}\le1$ for $T\ge T_0$, using $Z^2\le C\Lambda$ from $H_1$. The remaining differences are bounded the same way ($W$ and $\sum_t\vv_t\vw_t^{\top}$ via Cauchy--Schwarz against $\sqrt{Q_w}\le C\sqrt T$ on $H_3$). For $H_4$ both sides of the inequality move: the left side changes by at most $1$ as above, while on the right $Q_v$ changes by at most $1$ and the map $x\mapsto\sqrt{2(1+\|\mSigma\|x)(\log\sqrt{1+\|\mSigma\|x}+x_\ast)}$ is increasing in $x\ge0$, so, writing $F$ for that map, $|(\mtx{C}_a(\vartheta))_{kl}|\le|(\mtx{C}_a(\vartheta'))_{kl}|+1\le F\bigl(Q_v(\vartheta')\bigr)+1\le F\bigl(Q_v(\vartheta)+1\bigr)+1$ on $H_4$, and the transition from $H_4$ to (U2) absorbs both the shift by $1$ and the additive constant. Finally, $N|q_m(\vartheta)-q_m(\vartheta')|\le N\cdot2\sqrt{\bar qC_0}\,\bar L\,\tau_T\le CT^{-3}\le1$ by Lemma~\ref{lem:filtered}(a). Consequently, on $H$, \emph{for all} $\vartheta\in K_B$:
\begin{align}
&\text{(U1)}\quad \tfrac12\,Nq_m(\vartheta)-C\Lambda\;\le\;Q_v(\vartheta)\;\le\;2Nq_m(\vartheta)+C\Lambda\;\le\;2N\bar q+C\Lambda\;\le\;CN;\nonumber\\
&\text{(U2)}\quad \|\mtx{C}_a(\vartheta)\|_F\le C\sqrt{\bigl(1+Q_v(\vartheta)\bigr)\Lambda}+1,
\qquad |M(\vartheta)|\le\sqrt d\,\|\mtx{C}_a(\vartheta)\|_F;\nonumber\\
&\text{(U3)}\quad Q_w\le C\rho_\ast^{2m}N;\qquad
\text{(U4)}\quad \|\mtx{C}(\vartheta)\|_F\le\|\mtx{C}_a(\vartheta)\|_F+\sqrt{Q_v(\vartheta)\,Q_w}.\nonumber
\end{align}
For (U2) we used that, on $H_2$, $\|\mSigma\|Q_v\le CT$ at net points, so the logarithm in $H_4$ is at most $C\log T+x_\ast\le C\Lambda$, and $\|\mtx{C}_a\|_F\le d\max_{kl}|(\mtx{C}_a)_{kl}|$. Finally, the penalty in~\eqref{eq:rls_loss} and~\eqref{eq:marginalized_MAP} satisfies $0\le\mathrm{pen}(\vartheta)\le P_B:=B^2/(2\sigma^2_{\min})$ on $K_B$, where $\sigma^2_{\min}$ is the variance floor of Assumption~\ref{ass:consistency}(iii), since the group sums of squares add up to $\|\vartheta\|_2^2\le B^2$. This bound is the only place the group variances enter the proof, and it holds pathwise for any data-dependent selection respecting the floor: the events $H_1,\ldots,H_5$ involve the data alone, so the conclusion covers cross-validated variances without any modification of the argument.

\subsection{Proof for the RLS estimator}\label{sec:pf_rls}

Since $\hat\vartheta\in K_B$ has $\mathcal{L}_m^{\mathrm{RLS}}(\hat\vartheta)\le\min_{K_B}\mathcal{L}_m^{\mathrm{RLS}}+\eta\le\mathcal{L}_m^{\mathrm{RLS}}(\vartheta^0)+\eta$ (as $\vartheta^0\in K_B$), and $\mathcal{L}_m^{\mathrm{RLS}}=\tfrac12L_m+\mathrm{pen}$, we get $L_m(\hat\vartheta)-L_m(\vartheta^0)\le2\bigl(\mathrm{pen}(\vartheta^0)-\mathrm{pen}(\hat\vartheta)\bigr)+2\eta\le2P_B+2\eta$. Abbreviating $\hat Q:=Q_v(\hat\vartheta)$ and using~\eqref{eq:excess_id}, $|W|\le\sqrt{Q_wQ_v}$, (U2) and (U3), on $H$:
\[
\hat Q\;\le\;2|M(\hat\vartheta)|+2|W(\hat\vartheta)|+2P_B+2\eta
\;\le\;2\sqrt d\Bigl(C\sqrt{(1+\hat Q)\Lambda}+1\Bigr)+2\sqrt{C\rho_\ast^{2m}N\,\hat Q}+2P_B+2\eta.
\]
The two square roots absorb into the left side after rescaling by their constants: $\sqrt{uv}\le\tfrac12(u+v)$ with $(u,v)=\bigl(\hat Q/8,\;8C^2\Lambda\bigr)$ and $(u,v)=\bigl(\hat Q/8,\;8C^2\rho_\ast^{2m}N\bigr)$ gives
\[
2\sqrt d\,C\sqrt{(1+\hat Q)\Lambda}\;\le\;\tfrac18\hat Q+C'\Lambda+C',
\qquad
2\sqrt{C\rho_\ast^{2m}N\,\hat Q}\;\le\;\tfrac18\hat Q+C'\rho_\ast^{2m}N,
\]
the first after $\sqrt{1+\hat Q}\le1+\sqrt{\hat Q}$, so the $\hat Q$-contributions total at most $\tfrac14\hat Q$. Rearranging,
\begin{equation}\label{eq:Qhat_bound}
\hat Q\;\le\;C\bigl(\Lambda+\rho_\ast^{2m}N+\eta\bigr),
\end{equation}
after absorbing $P_B+1\le C\Lambda$. By (U1), $q_m(\hat\vartheta)\le2(\hat Q+C\Lambda)/N\le C(\Lambda/N+\rho_\ast^{2m}+\eta/N)$, so by~\eqref{eq:qm_sandwich} and $N\ge T/2$,
\[
\sum_{j\le m}\|\Delta_j(\hat\vartheta)\|_F^2\;\le\;q_m(\hat\vartheta)/c_0\;\le\;C\Bigl(\frac{\Lambda}{T}+\rho_\ast^{2m}+\frac{\eta}{T}\Bigr),
\]
\[
\sum_{j>m}\|\Delta_j(\hat\vartheta)\|_F^2\;\le\;d\sum_{j>m}\bigl(2\bar G\rho_\ast^{\,j}\bigr)^2\;\le\;C\rho_\ast^{2m},
\]
the tail using $\|\mtx{M}\|_F\le\sqrt d\,\|\mtx{M}\|$ and Lemma~\ref{lem:uniform_geom}(i) for both $\hat\vartheta$ and $\vartheta^0$. Summing the two displays gives~\eqref{eq:pi_rate}. \hfill$\square$

\subsection{Proof for the JMAP estimator}\label{sec:pf_jmap}

Write $\mtx{S}^0:=\mPsi_0+\mR(\vartheta^0)$ and $\hat{\mtx{S}}:=\mPsi_0+\mR(\hat\vartheta)$; both are positive definite since $\mPsi_0\succ\mz$. The estimate has $\mathcal{L}_m^{\mathrm{marg}}(\hat\vartheta)\le\min_{K_B}\mathcal{L}_m^{\mathrm{marg}}+\eta\le\mathcal{L}_m^{\mathrm{marg}}(\vartheta^0)+\eta$; with $0\le\mathrm{pen}\le P_B$, dividing by $(\nu_0+N)/2\ge N/2$ gives
\begin{equation}\label{eq:jmap_basic}
\log\det\hat{\mtx{S}}-\log\det\mtx{S}^0\;\le\;\frac{2P_B+2\eta}{N}.
\end{equation}
The argument parallels Appendix~\ref{sec:pf_rls}: the basic inequality~\eqref{eq:jmap_basic} caps how much the fitted log-determinant may exceed the reference one; localizing around $\mtx{S}^0$ splits the difference into a positive-semidefinite \emph{signal} term, whose trace is proportional to the prediction-error mass $\hat Q:=Q_v(\hat\vartheta)$, and a \emph{cross} term that the events of Appendix~\ref{sec:pf_events} render negligible in operator norm; the resulting scalar inequality in $\hat Q$ is then closed exactly as before.

\emph{Spectral bounds for $\mtx{S}^0$.} Since $\mR(\vartheta^0)=\sum_t(\va_{t+1}+\vw_t)(\va_{t+1}+\vw_t)^{\top}$ and $\sum_t\vw_t\vw_t^{\top}\succeq\mz$, we have, on $H_3\cap H_5$, using $\|\sum_t\va_{t+1}\va_{t+1}^{\top}-N\mSigma\|\le\|\cdot\|_F\le C\sqrt{N\Lambda}$ from $H_5$ and the bilinear Cauchy--Schwarz bound $\|\sum_t\va_{t+1}\vw_t^{\top}\|\le\sqrt{Q_a\,Q_w}\le CN\rho_\ast^{m}$ (for unit vectors $u,v$: $u^{\top}\bigl(\sum_t\va_{t+1}\vw_t^{\top}\bigr)v\le\sqrt{\sum_t(u^{\top}\va_{t+1})^2}\sqrt{\sum_t(\vw_t^{\top}v)^2}\le\sqrt{Q_aQ_w}$, with $Q_a,Q_w$ bounded on $H_3$), and noting that the cross term appears together with its transpose, hence with a factor $2$,
\[
\lambda_{\min}(\mtx{S}^0)\ge N\lambda_{\min}(\mSigma)-C\sqrt{N\Lambda}-2\sqrt{Q_aQ_w}\;\ge\;N\lambda_{\min}(\mSigma)-C\sqrt{N\Lambda}-CN\rho_\ast^{m}\ge N c_\Sigma,
\]
\[
\lambda_{\max}(\mtx{S}^0)\le \|\mPsi_0\|+N\|\mSigma\|+C\sqrt{N\Lambda}+CN\rho_\ast^{m}+Q_w\le NC_\Sigma,
\]
with $c_\Sigma:=\lambda_{\min}(\mSigma)/2$ and $C_\Sigma:=3\|\mSigma\|$, for $T\ge T_0$, $m\ge m_1$ and $c$ small enough that $C\sqrt{\Lambda/N}\le C\sqrt{2c}\le\lambda_{\min}(\mSigma)/4$ and $C\rho_\ast^{m_1}\le\lambda_{\min}(\mSigma)/4$.

\emph{Log-determinant localization.} By~\eqref{eq:R_id}, $\hat{\mtx{S}}=\mtx{S}^0+\mtx{V}(\hat\vartheta)+\mtx{C}(\hat\vartheta)+\mtx{C}(\hat\vartheta)^{\top}$, so factoring out $\mtx{S}^{0\,1/2}$ on both sides,
\[
\log\det\hat{\mtx{S}}-\log\det\mtx{S}^0=\log\det(\mI+\mtx{E}_V+\mtx{E}_C),
\]
\[
\mtx{E}_V:=\mtx{S}^{0\,-1/2}\mtx{V}(\hat\vartheta)\mtx{S}^{0\,-1/2},
\qquad
\mtx{E}_C:=\mtx{S}^{0\,-1/2}\bigl(\mtx{C}(\hat\vartheta)+\mtx{C}(\hat\vartheta)^{\top}\bigr)\mtx{S}^{0\,-1/2}.
\]
The two normalized perturbations play different roles. The \emph{signal} term $\mtx{E}_V\succeq\mz$ can only increase the log-determinant, and its trace measures the prediction-error mass: since $\tr(\mtx{S}^{0\,-1}\mtx{V})\in[\tr\mtx{V}/\lambda_{\max}(\mtx{S}^0),\ \tr\mtx{V}/\lambda_{\min}(\mtx{S}^0)]$ for $\mtx{V}\succeq\mz$ and $\tr\mtx{V}=\hat Q$,
\[
\frac{\hat Q}{NC_\Sigma}\;\le\;\tr\mtx{E}_V\;\le\;\frac{\hat Q}{Nc_\Sigma}\;\le\;C,
\]
the last bound because $\hat Q\le CN$ by (U1). The \emph{cross} term $\mtx{E}_C$ collects the martingale interaction of $\vv$ with the innovations and its interaction with the truncation tail; on the event it is small in operator norm: by (U2) and (U4), $\|\mtx{C}(\hat\vartheta)\|_F\le C\sqrt{(1+\hat Q)\Lambda}+1+\sqrt{\hat Q\,Q_w}$, which with the crude bounds $\hat Q\le CN$ and $Q_w\le C\rho_\ast^{2m}N$ (U3) gives $\|\mtx{C}(\hat\vartheta)\|_F\le C(\sqrt{N\Lambda}+N\rho_\ast^{m})$, whence
\[
e:=\|\mtx{E}_C\|\;\le\;\frac{\|\mtx{C}(\hat\vartheta)+\mtx{C}(\hat\vartheta)^{\top}\|}{\lambda_{\min}(\mtx{S}^0)}\;\le\;\frac{2\|\mtx{C}(\hat\vartheta)\|_F}{Nc_\Sigma}\;\le\;C\bigl(\sqrt{\Lambda/N}+\rho_\ast^{m}\bigr)\;\le\;\tfrac12
\]
for $c$ small, $m\ge m_1$, $T\ge T_0$ (as before, $\sqrt{\Lambda/N}\le\sqrt{2c}$). The mechanism is now visible: \eqref{eq:jmap_basic} allows the log-determinant to grow by at most $2P_B/N$, while the signal term forces growth of order $\tr\mtx{E}_V\gtrsim\hat Q/N$ that the small cross term cannot cancel. We make this quantitative in three elementary steps. \emph{Step 1} (absorb the cross term): $\mtx{E}_C\succeq-e\mI$, so by Loewner monotonicity of $\log\det$ over positive-definite matrices,
\[
\log\det(\mI+\mtx{E}_V+\mtx{E}_C)\;\ge\;\log\det\bigl((1-e)\mI+\mtx{E}_V\bigr)\;=\;\sum_{i=1}^{d}\log(1-e+\lambda_i),
\]
where $\lambda_1,\dots,\lambda_d\ge0$ are the eigenvalues of $\mtx{E}_V$ (the shifted matrix is positive definite since $e\le\tfrac12$). \emph{Step 2} (separate the shift): $1-e+\lambda_i\ge(1-e)(1+\lambda_i)$ (expanding, this is just $e\lambda_i\ge0$) and $\log(1-e)\ge-2e$ on $[0,\tfrac12]$, so
\[
\sum_{i=1}^{d}\log(1-e+\lambda_i)\;\ge\;\sum_{i=1}^{d}\log(1+\lambda_i)\;-\;2de.
\]
\emph{Step 3} (curvature): by concavity, $\log(1+x)\ge x\log(1+b)/b\ge x/(1+b)$ for $0\le x\le b$; taking $b=\max_i\lambda_i\le\tr\mtx{E}_V$,
\[
\sum_{i=1}^{d}\log(1+\lambda_i)\;\ge\;\frac{\sum_i\lambda_i}{1+\max_i\lambda_i}\;\ge\;\frac{\tr\mtx{E}_V}{1+\tr\mtx{E}_V}\;\ge\;\frac{\tr\mtx{E}_V}{C}\;\ge\;\frac{\hat Q}{CN},
\]
the penultimate inequality because $\tr\mtx{E}_V\le C$ as shown above. Combining the three steps with~\eqref{eq:jmap_basic} and re-inserting the sharp form of $\|\mtx{C}(\hat\vartheta)\|_F$ into $2de\le4d\|\mtx{C}(\hat\vartheta)\|_F/(Nc_\Sigma)$,
\[
\frac{\hat Q}{CN}\;\le\;\frac{2P_B+2\eta}{N}+2de
\;\le\;\frac{C}{N}\Bigl(P_B+\eta+\sqrt{(1+\hat Q)\Lambda}+1+\sqrt{\hat Q\,\rho_\ast^{2m}N}\Bigr).
\]
Multiplying through by $CN$ and applying $\sqrt{ab}\le\tfrac{a}{4}+b$ (valid since $(\sqrt a/2-\sqrt b)^2\ge0$) to each square root (with $a=1+\hat Q$, $b=C^2\Lambda$ for the first and $a=\hat Q$, $b=C^2\rho_\ast^{2m}N$ for the second, so the $a$-contributions total $\tfrac12\hat Q+\tfrac14$ and are absorbed into the left-hand side) yields
\[
\hat Q\;\le\;C\bigl(\Lambda+\rho_\ast^{2m}N+\eta+P_B+1\bigr)\;\le\;C\bigl(\Lambda+\rho_\ast^{2m}N+\eta\bigr),
\]
which is~\eqref{eq:Qhat_bound}. The conclusion then follows exactly as in Appendix~\ref{sec:pf_rls}. \hfill$\square$

\subsection{Proof of Corollary~\ref{cor:pi_consistency}, and optimality}\label{sec:pf_cor}

For any $J\ge1$, by Cauchy--Schwarz on the first $J$ lags and Lemma~\ref{lem:uniform_geom}(i) on the tail (both $\hat\vartheta,\vartheta^0\in K_B$),
\begin{equation}\label{eq:ell1_split}
\sum_{j\ge1}\|\Delta_j(\hat\vartheta)\|
\;\le\;\sqrt J\Bigl(\sum_{j\ge1}\|\Delta_j(\hat\vartheta)\|_F^2\Bigr)^{1/2}+\frac{2\bar G\rho_\ast^{\,J+1}}{1-\rho_\ast}.
\end{equation}
\emph{(ii)} If $m\ge\log T/\log(1/\rho_\ast)$ then $\rho_\ast^{2m}\le T^{-2}\le\Lambda_\delta/T$, and on the event $E_T$ of the statement also $\eta/T\le C'\,n_{\mathrm{par}}\log T/T\le C'\Lambda_\delta/T$. On the intersection of $E_T$ with the theorem's event, which has probability at least $1-\delta-\delta'$, \eqref{eq:pi_rate} gives $\sum_j\|\Delta_j\|_F^2\le{}C\Lambda_\delta/T$; taking $J:=\lceil\log T/(2\log(1/\rho_\ast))\rceil$, so that $\rho_\ast^{J}\le T^{-1/2}$ and $\sqrt J\le C\sqrt{\log T}$, \eqref{eq:ell1_split} is at most $C\bigl(\sqrt{\Lambda_\delta\log T/T}+T^{-1/2}\bigr)\le C\sqrt{\Lambda_\delta\log T/T}$.
\emph{(i)} Fix $\delta\in(0,1/2)$, $\epsilon>0$ and $t>0$. Under $m\to\infty$ and $m\le T/2$ the standing conditions of the theorem hold for all large $T$ (for fixed $\delta$, $\Lambda_\delta=o(T)$), and $\beta_T:=C(\Lambda_\delta/T+\rho_\ast^{2m}+t)\to C\,t$. On the intersection of the theorem's event with $\{\eta/T\le t\}$, the bound~\eqref{eq:pi_rate} is at most $\beta_T$, and \eqref{eq:ell1_split} with $J=J_T:=\lceil\beta_T^{-1/2}\rceil$ is at most $C\beta_T^{1/4}+C\rho_\ast^{J_T}$, which is smaller than $\epsilon$ for all large $T$ once $t$ is chosen small enough (depending on $\epsilon$), since $\rho_\ast<1$ is fixed. Hence
\[
\limsup_T\;\PP\Bigl(\sum_j\|\Delta_j\|>\epsilon\Bigr)\;\le\;\delta+\limsup_T\;\PP(\eta/T>t)\;=\;\delta,
\]
using $\eta/T\overset{P}{\to}0$; since $\delta$ was arbitrary, the sum tends to $0$ in probability. \hfill$\square$

\begin{remark}[The $T^{-1/2}$ dependence is unavoidable]\label{rem:pi_lower}
The $T^{-1/2}$ dependence in Corollary~\ref{cor:pi_consistency}(ii) cannot be improved, already in a one-parameter Gaussian submodel; we make no claim of minimax optimality in $n_{\mathrm{par}}$ for the summed operator-norm loss, which would require a matching dimension-dependent packing bound. Take $\mTheta=\mz$, $\mSigma=\mI_d$, $\mPhi_1=\gamma\,\vct{e}_1\vct{e}_1^{\top}$ (so $\mPi_1=\mPhi_1$), and the two hypotheses $\gamma\in\{0,h\}$ with $h:=1/\sqrt T$, each observed from its stationary initial law. The Gaussian chain rule bounds the Kullback--Leibler divergence between the two path laws by $\tfrac12\sum_{t=2}^{T}\E_0[(hz_{t-1,1})^2]$ plus the initial-state term $\tfrac12(-h^2-\log(1-h^2))\le h^4$, hence by $Th^2/2+h^4\le1$. Pinsker's inequality gives total variation at most $\sqrt{1/2}$, so Le~Cam's two-point method yields a testing error of at least $(1-\sqrt{1/2})/2\ge0.14$. Since $\sum_{j\ge1}\|\hat\mPi_j-\mPi_j\|\ge\|\hat\mPi_1-\mPi_1\|$ and the two candidate values of $\mPi_1$ are $h$ apart, every estimator incurs $\sum_{j\ge1}\|\hat\mPi_j-\mPi_j\|\ge1/(2\sqrt T)$ with probability at least $0.14$ under one of the two models.
\end{remark}

\subsection{Extension to the multiplicative seasonal VARMA}\label{sec:pf_sarma}

\begin{proof}[Proof of Corollary~\ref{cor:sarma_rate}]
Only Lemma~\ref{lem:uniform_geom} and the well-specification argument interact with the parametrization; we verify these, after which every step of Appendices~\ref{sec:pf_prelim}--\ref{sec:pf_cor} applies verbatim with $n_{\mathrm{par}}=(p+q+P+Q)d^2$, which enters the proof only through the net cardinality and hence $\Lambda_\delta$.

\emph{(i) Expanded stability and uniform geometry.} For $\vartheta\in K_B$ write
\[
\bar{\mPhi}(z;\vartheta):=\mPhi^{(s)}(z^{s};\vartheta)\,\mPhi(z;\vartheta),
\qquad
\bar{\mTheta}(z;\vartheta):=\mTheta^{(s)}(z^{s};\vartheta)\,\mTheta(z;\vartheta)
\]
for the expanded operators of~\eqref{eq:sarma}, and $\mtx{A}(z;\vartheta):=\bar{\mTheta}(z;\vartheta)^{-1}\bar{\mPhi}(z;\vartheta)$. The four per-factor PAR maps are smooth, so, exactly as in the proof of Lemma~\ref{lem:uniform_geom},
\[
\bar\rho:=\max_{\vartheta\in K_B}\max\bigl\{\spr(\mathcal{C}_{\mPhi}),\spr(\mathcal{C}_{\mPhi^{(s)}}),\spr(\mathcal{C}_{\mTheta}),\spr(\mathcal{C}_{\mTheta^{(s)}})\bigr\}
\]
is attained and $<1$.
Factorwise stability implies expanded stability because determinants multiply, $\det\bar{\mTheta}(z;\vartheta)=\det\mTheta^{(s)}(z^{s};\vartheta)\,\det\mTheta(z;\vartheta)$, and the seasonal factor's zeros satisfy $|z|^{s}\ge1/\bar\rho$, while the regular factor's zeros have modulus at least $1/\bar\rho\ge(1/\bar\rho)^{1/s}$ (as $s\ge1$); hence every zero of $\det\bar{\mTheta}$ has modulus at least $(1/\bar\rho)^{1/s}>1$, and $\rho_\ast:=\bigl((1+\bar\rho)/2\bigr)^{1/s}=\tilde\rho^{\,1/s}$ gives a common analytic disk of radius $1/\rho_\ast>1$ for $(z,\vartheta)\mapsto\mtx{A}(z;\vartheta)$ on $K_B$. The four-factor PAR map is smooth, so the Cauchy-estimate proof of Lemma~\ref{lem:uniform_geom} applies unchanged with this $\rho_\ast$, the constants depending additionally on $s$ through the polynomial degrees.

\emph{(ii) Well-specification and true-model constants.} Zero-padding each factor to the fitted orders leaves the four factor polynomials, hence the expanded operators and the transfer function, unchanged, and each factor's PAR map is a bijection onto its stable set (Proposition~\ref{prop:par_bijection}, together with the padding observation at the end of Appendix~\ref{sec:proof_par}), so the padded true coordinates $\vartheta^0$ are well defined and lie in $K_B$ by assumption. The expanded operators are the ordered products $\bar{\mPhi}(z)=\mPhi^{(s)}(z^{s})\mPhi(z)$ and $\bar{\mTheta}(z)=\mTheta^{(s)}(z^{s})\mTheta(z)$ of~\eqref{eq:sarma}; the factors need not commute for $d>1$, so the reversed products are in general different matrix polynomials, but the determinants agree with either ordering, $\det\bar{\mPhi}(z)=\det\mPhi^{(s)}(z^{s})\,\det\mPhi(z)$, and $|z|\le1\Rightarrow|z^{s}|\le1$ shows that the expanded true operators (an ordinary VARMA whose AR and MA polynomials have degrees at most $\bar p+s\bar P$ and $\bar q+s\bar Q$; the leading coefficients $-\mPhi^{(s)}_{\bar P}\mPhi_{\bar p}$ and $-\mTheta^{(s)}_{\bar Q}\mTheta_{\bar q}$ can vanish for $d>1$) satisfy Assumption~\ref{ass:roots}, so the constants $\mu_{\min}$, $c_0$, $C_0$ of Lemma~\ref{lem:filtered} are defined for the true seasonal process exactly as before.

\emph{(iii) Everything else.} The seasonal RLS and JMAP losses are~\eqref{eq:rls_loss} and~\eqref{eq:marginalized_MAP} evaluated through the expanded $\mPi_{1:m}(\vartheta)$, with the group-ridge penalty on the four unconstrained stacks, again bounded by $P_B$ on $K_B$. The decomposition, the events, and the assemblies of Appendices~\ref{sec:pf_events}--\ref{sec:pf_cor} use the parametrization only through Lemma~\ref{lem:uniform_geom}, the dimension $n_{\mathrm{par}}$, and the true-model constants, and therefore apply verbatim. Finally, $\log(1/\rho_\ast)=\log(1/\tilde\rho)/s$, so the condition $m\ge\log T/\log(1/\rho_\ast)$ of Corollary~\ref{cor:pi_consistency}(ii) reads $m\ge s\log T/\log(1/\tilde\rho)$.
\end{proof}

\section{Computational cost}\label{sec:runtimes}
All timings below were measured on the runs reported in Section~\ref{sec:Experiments}, on an Apple M4 Pro (14 cores, 48\,GB). Every timed unit ran with nothing else computing on the machine; a unit is one model on the real-data panels and one (cell, seed, method) triple on the synthetic benchmark. The units run strictly one after another, and the parallelism sits inside each unit: the ridge-grid cross-validation runs one process per grid point, and the rolling-window refits one process per window. Each figure is wall-clock seconds, so the numbers are comparable across the estimators reported here but not across hardware or implementations. The complexity statements of Section~\ref{sec:complexity_S2} are the implementation-independent counterpart, and the comparisons in Section~\ref{sec:Experiments} rest on forecast accuracy rather than on these timings.

Table~\ref{tab:runtime_synth} reports the synthetic sweep: each entry is the mean over the eight seeds and the three training lengths of a dimension, and the standard deviations, a few tenths of a second, show how stable the per-seed cost is. Cost grows with the dimension and is nearly flat in $T$, which is the behaviour the Parseval evaluation of Section~\ref{sec:complexity_S2} predicts: the sufficient statistics are formed once per window, after which each optimizer iteration is independent of $T$. \texttt{bigtime} is fast on these small, short-memory problems, and its cost shows the opposite $T$-profile, roughly doubling from $T=800$ to $3200$ at $d=40$, as its least-squares path is refit over all $T$ observations.

\begin{table}[htbp]
\centering
\caption{Synthetic benchmark, wall-clock seconds per seed (mean $\pm$ standard deviation over the eight seeds and the three training lengths $T$ at each dimension).  A unit comprises the pilot, the eight-point ridge-grid cross-validation (one process per grid point, in parallel), and the refit at the selected $\sigma$, timed with nothing else running; \texttt{bigtime}'s unit is its own pipeline, the penalty-path fits on the cross-validation and full windows plus the selection.}\label{tab:runtime_synth}
\begin{tabular}{l ccc c ccc}
\toprule
& \multicolumn{3}{c}{dense} && \multicolumn{3}{c}{sparse}\\
\cmidrule{2-4}\cmidrule{6-8}
& $d{=}10$ & $d{=}20$ & $d{=}40$ && $d{=}10$ & $d{=}20$ & $d{=}40$\\
\midrule
RLS  & $2.6\pm0.4$ & $8.5\pm0.5$ & $36.6\pm0.6$ && $2.2\pm0.3$ & $8.2\pm0.2$ & $36.8\pm0.5$\\
JMAP & $2.6\pm0.0$ & $8.5\pm0.2$ & $37.9\pm0.5$ && $2.6\pm0.0$ & $8.6\pm0.2$ & $37.8\pm0.7$\\
\texttt{bigtime} & $0.4\pm0.2$ & $0.6\pm0.1$ & $1.7\pm0.4$ && $0.4\pm0.0$ & $0.6\pm0.1$ & $1.7\pm0.4$\\
\bottomrule
\end{tabular}
\end{table}

On the real-data panels the cost is reported per model: its cross-validation plus all rolling-window refits (Table~\ref{tab:runtime_real}). Dominick's is small enough that the entire suite of ten estimators takes under three minutes. The two hourly panels are heavier, mainly through the truncation $m=1024$: every objective evaluation forms $\mPi$ at $1024$ lags, so a seasonal VARMA there costs a few minutes where the synthetic fits at $m=20$ cost seconds.

\begin{table}[t]
\centering
\caption{Real-data panels, wall-clock seconds per model: its cross-validation plus all rolling-window refits ($4$ windows for Beijing, $6$ for Singapore, a single frozen window for Dominick's), each model timed with nothing else running.  Within a model the ridge-grid CV runs one process per grid point and the refits one process per window.  The \texttt{bigtime} entry covers its sparse-VAR and sparse-VARMA modes together, which share one set of fits and are therefore counted once in the total.}\label{tab:runtime_real}
\begin{tabular}{l rrr}
\toprule
& Dominick's & Beijing & Singapore\\
\midrule
seasonal VARMA RLS      & --   & $169$ & $201$\\
seasonal VARMA JMAP     & --   & $193$ & $225$\\
VARMAX RLS              & $43$ & --    & --\\
VARMAX JMAP             & $36$ & --    & --\\
dense VARMA RLS         & $29$ & $144$ & $200$\\
dense VARMA JMAP        & $28$ & $148$ & $221$\\
seasonal VAR RLS        & -- & $67$ & $108$\\
seasonal VAR JMAP       & -- & $77$ & $133$\\
ridge VAR(X)            & $3$ & $5$   & $6$\\
Bayesian VAR(X)         & $9$ & $320$ & $167$\\
componentwise AR        & --  & $4$   & $5$\\
componentwise ARMA(X)      & $8$ & $37$ & $37$\\
\texttt{bigtime}           & $5$ & $1514$ & $2383$\\
training mean, random walk & $<1$ & $3$ & $3$\\
\bottomrule
\end{tabular}
\end{table}

Two features of these numbers matter for practice. RLS and JMAP cost essentially the same, since they differ only in the scalar form of the loss and share the sufficient statistics, and the eight ridge-grid fits are mutually independent, so the selection parallelizes to one process per grid point. Cost grows with the dimension but is nearly flat in $T$, as the sufficient-statistics evaluation predicts.

The \texttt{bigtime} baseline is by far the most expensive model on the hourly panels: its penalty paths at lags up to $96$ are fit in \textsf{R} (each fit single-threaded, the candidate fits and windows running concurrently), and the sparse-VARMA fit at $p=48$ dominates each window.

\section{Connection with state-space models}\label{sec:statespace}

VARMA models are known to have a strong connection with state-space models. As explained in \cite[Section 14.6]{box2015time}, this model can be represented in an equivalent state-space form. Let $r=\max(p,q+1)$, and let $\mTheta_{q+1},\ldots,\mTheta_{r}:=\mz$ and $\mPsi_{p+1},\ldots, \mPsi_r:=\mz$.

As explained \cite[Section 14.5.2]{box2015time}, under Assumption \ref{ass:roots}, the VARMA process also has an infinite moving average representation that can be written as 
\[\vZ_t=\mPsi(B) \va_t:=\sum_{j=0}^{\infty}\mPsi_j \va_{t-j} \]
where $B$ is the backward shift operator. 
As explained in \cite[Section 14.4.1]{box2015time}, the matrices $\mPsi_j$ can be determined from the relation $\mPhi(B)\mPsi(B)=\mTheta(B)$, and satisfy the recursion
\begin{align}\label{eq:mPsirec}
\mPsi_j=\mPhi_1 \mPsi_{j-1}+\mPhi_2\mPsi_{j-2}+\ldots+\mPhi_p \mPsi_{j-p}-\mTheta_j, \quad j=1,2,\ldots,
\end{align}
where $\mPsi_0=\mI$, $\mPsi_j=\mz$ for $j<0$, and $\mTheta_j=\mz$ for $j>q$.

Let $\vY_t$ be an $r\cdot d$ dimensional (unobservable) vector called the state vector. Let us denote $\vY_t(1):=(\vY_t)_{1:d}, \ldots, \vY_t(r):=(\vY_t)_{(r-1)d+1: rd}$. Let $\mtx{H}=[\mtx{I},\mtx{0},\ldots,\mtx{0}]$ be projecting to the first $d$ components, and let $\vZ_t:=\mtx{H} \vY_t=\vY_t(1)$. 

It is explained in  \cite{box2015time} that if we define the stationary VAR process
\begin{align}\vY_t&=\mtx{F}\,\vY_{t-1}+\mPsi \va_t, \text{ where}\\
\label{eq:defmF}\mtx{F}&:=\left[\begin{matrix}\mz & \mI & \mz & \ldots & \mz\\
\mz & \mz & \mI & \ldots & \mz\\
\vdots & \vdots & \vdots & \ddots & \vdots\\
\mz & \mz & \mz & \ldots & \mI\\
\mPhi_r & \mPhi_{r-1} & \ldots  & \ldots &\mPhi_1
\end{matrix}\right] \quad \text{and} \quad \mPsi:=\left[\begin{matrix}\mI\\
\mPsi_1\\
\vdots\\
\mPsi_{r-1}\\
\mPsi_r
\end{matrix}\right],
\end{align}
then $\vZ_t=\vY_t(1)$ is a VARMA process satisfying \eqref{eq:VARMAdef}. The state-space transition matrix $\mtx{F}\in\R^{rd\times rd}$ has the AR coefficients $\mPhi_{1:r}$ in its bottom block-row; it is structurally similar to the $pd\times pd$ AR companion $\mathcal{C}_{\mPhi}$ of Section~\ref{sec:spectral_constraints}, but acts on a higher-dimensional state $\vY_t$ that absorbs the MA component into a recursive form. 

Let $\mGamma_l=\E(\vZ_{t-l} \vZ_{t}')$ denote the lag-$l$ covariance matrices of the VARMA process (by definition, $\mGamma_{-l}=\mGamma_l'$). As explained in Section 14.4.2, these satisfy that
\begin{align}\label{eq:Gammal1}
\mGamma_l&=\sum_{j=1}^{p} \mGamma_{l-j}\mPhi_j'-\sum_{j=l}^{q} \mPsi_{j-l}\mSigma \mTheta_j' \quad 
\text{for}\quad l=0,\ldots,q,\\
\label{eq:Gammal2}
\mGamma_l&=\sum_{j=1}^{p} \mGamma_{l-j}\mPhi_j' \quad \text{for}\quad l>q,
\end{align}
with the convention that $\mTheta_0=-\mI$.

\section{Online versions of the pointwise estimators}\label{sec:online}

In deployment the training window slides: at step $r=0,1,2,\ldots$ the model is refit on the window $\mathcal{W}_r:=\{r\Delta+1,\ldots,r\Delta+T\}$ of length $T$ with stride $\Delta\le T$, standardized per coordinate as $\vZ^{(r)}_t:=(\vY_t-\mu_r)\oslash\sigma_r$ with $(\mu_r,\sigma_r)$ computed from $\mathcal{W}_r$ and $\oslash$ denoting coordinate-wise division. Recomputing the sufficient-statistics pre-pass from scratch at every shift costs $O(T\,m\,d^2)$; the exact identities below reduce the per-shift cost to $O(\Delta\,m\,d^2)$, independent of $T$. Because the standardization changes with $r$, so that every entry of $\vZ^{(r)}$ re-used from the previous window takes a new value, the updates are maintained for the \emph{raw} series $\vY$, and the standardization is applied afterwards in closed form. The discussion covers any objective whose data-dependence is exhausted by $(\mS_0,\mS_1,\mS_2)$, hence both RLS and JMAP.

\paragraph{Raw updates and standardization.}
Every entry of the raw triple $(\mS^{\mathrm{raw}}_0,\mS^{\mathrm{raw}}_1,\mS^{\mathrm{raw}}_2)$ (Eq.~\eqref{eq:LmPiTheta:def} with $\vZ$ replaced by $\vY$) is a sum of outer products over the $N=T-m$ target rows of the window, so a shift by $\Delta$ drops the first $\Delta$ contributions and adds $\Delta$ new ones: with lag blocks $L_j(t):=\vY_{m-j+t}$, the indices aligned along the union segment of the two windows,
\begin{equation}\label{eq:rawupdate}
\begin{aligned}
\mS_0^{\mathrm{raw}}\bigl(\mathcal{W}_{r}\bigr)
&=\mS_0^{\mathrm{raw}}\bigl(\mathcal{W}_{r-1}\bigr)
+\sum_{t=N}^{N+\Delta-1}\!\!\vY_{m+t}\vY_{m+t}'
-\sum_{t=0}^{\Delta-1}\!\vY_{m+t}\vY_{m+t}',\\[2pt]
\mS_1^{\mathrm{raw}}(j;\mathcal{W}_{r})
&=\mS_1^{\mathrm{raw}}(j;\mathcal{W}_{r-1})
+\sum_{t=N}^{N+\Delta-1}\!\!L_j(t)\,\vY_{m+t}'
-\sum_{t=0}^{\Delta-1}\!L_j(t)\,\vY_{m+t}',\\[2pt]
\mS_2^{\mathrm{raw}}(i,j;\mathcal{W}_{r})
&=\mS_2^{\mathrm{raw}}(i,j;\mathcal{W}_{r-1})
+\sum_{t=N}^{N+\Delta-1}\!\!L_i(t)\,L_j(t)'
-\sum_{t=0}^{\Delta-1}\!L_i(t)\,L_j(t)',
\end{aligned}
\end{equation}
at $O(\Delta\,m\,d^2)$ flops and $O(m\,d^2)$ memory traffic per shift; the FFT-cached autocovariance summary of Section~\ref{sec:complexity_S2} updates by the same incremental sums. Maintaining in addition the first-moment sums $\vct{M}^{\mathrm{tgt}}_r:=\sum_{t<N}\vY_{m+t}$ and $\vct{M}^{\mathrm{lag}}_r(j):=\sum_{t<N}L_j(t)$, together with the within-window sums of $\vY_t$ and $\vY_t\odot\vY_t$ defining $(\mu_r,\sigma_r)$ --- all updated by the same drop-and-add pattern --- the normalized statistics of $\vZ^{(r)}$ follow in closed form by expanding the standardization; representatively,
\begin{equation}\label{eq:znorm_stats}
\mS_2(\vZ^{(r)};i,j)=\diag(\sigma_r)^{-1}\bigl[\mS^{\mathrm{raw}}_2(i,j)-\vct{M}^{\mathrm{lag}}_r(i)\,\mu_r'-\mu_r\,(\vct{M}^{\mathrm{lag}}_r(j))'+N\mu_r\mu_r'\bigr]\diag(\sigma_r)^{-1},
\end{equation}
and analogously for $\mS_0$ and $\mS_1$, with the target moments $\vct{M}^{\mathrm{tgt}}_r$ in place of the corresponding lag moments. No observation in the window is re-touched; the transformation costs $O(m^2d^2)$, or $O(m\,d^2)$ in the FFT representation, in which $\mS_2$ is never materialized.

\paragraph{Warm start, exactness and memory.}
The loss for window $r$ is then built exactly as in Sections~\ref{sec:RLS}--\ref{sec:MAP} and minimized by the same unconstrained L-BFGS, initialized at the previous window's optimum in the unconstrained PAR coordinates; feasibility is automatic through the map. Equations~\eqref{eq:rawupdate}--\eqref{eq:znorm_stats} are exact in real arithmetic, so the incrementally maintained statistics define the same mathematical objective as a from-scratch pre-pass; floating-point results can differ slightly, the two computations summing in different orders. Warm-starting from the preceding fit is empirically effective, typically terminating in a few iterations, but we claim no uniform iteration complexity and no equality of the returned local optima; the only statistical approximation, the use of a finite window, is the usual sliding-window modelling choice. The persistent state is $O(m\,d^2)$ floats in the FFT representation, independent of $T$. The experiments of Section~\ref{sec:Experiments} use few, long windows and refit each one afresh; the identities above are the incremental alternative when refits are frequent.

\section{Cost of the seasonal and exogenous extensions}\label{sec:cost_appendix}

Section~\ref{sec:complexity_S2} gives the leading $O(m\log m\,d^2+m\,d^3)$ cost of the base VARMA loss; this appendix accounts for the full per-gradient cost, including the multiplicative-seasonal factors of Section~\ref{sec:seasonal_varma} and the exogenous regressors of the VARMAX extension. Write $\ell_{\mPhi}=p+sP$ and $\ell_{\mTheta}=q+sQ$ for the orders of the expanded operators $(\mPhi_{1:\ell_{\mPhi}},\mTheta_{1:\ell_{\mTheta}})$, and let $n_{\mathrm{fft}}=\Theta(m)$ be the FFT length (the smallest power of two with $n_{\mathrm{fft}}\ge 2m$). The data enter only through a one-time pre-pass that forms the sufficient statistics; this is the sole $T$-dependent step. Direct summation costs $O(T\,m\,d^2)$ flops. The implementation instead evaluates the lagged cross-products by blocked (overlap--save) FFT correlation with length-$\Theta(m)$ blocks (distinct from the length-$n_{\mathrm{fft}}$ transforms below): the forward transforms of each block are shared across the $d^2$ channel pairs, the per-block cross-spectra are accumulated, and a single length-$\Theta(m)$ inverse transform per pair then recovers all $m$ lags, giving $O(T\log m\,d + T\,d^2 + m\log m\,d^2) = O(T\log m\,d^2)$ for the pre-pass. Each subsequent value-and-gradient evaluation is independent of $T$ and costs
\begin{equation}\label{eq:sarma_cost}
\begin{aligned}
&\underbrace{O\!\bigl(m\log m\,d^2 + m\,d^3\bigr)}_{\text{(i) FFT}}
\;+\;\underbrace{O\!\bigl((p^2{+}P^2{+}q^2{+}Q^2)\,d^3\bigr)}_{\text{(ii) PAR maps}}
\;+\;\underbrace{O\!\bigl((P(p{+}1){+}Q(q{+}1))\,d^3\bigr)}_{\text{(iii) factor convolution}}\\
&\qquad\;+\;\underbrace{O\!\bigl((\ell_{\mPhi}{+}\ell_{\mTheta})\,d^2\bigr)}_{\text{(iv) assembly}}
\;+\;\underbrace{O\!\bigl(d^3\bigr)}_{\text{(v) Cholesky}}.
\end{aligned}
\end{equation}
Term~(i) forms the truncated infinite-VAR coefficients $\mPi_{1:m}=\bigl[\mI-\mTheta(z)^{-1}\mPhi(z)\bigr]_{1:m}$ by FFT inversion and evaluates the $\mS_2$ quadratic form by a Parseval (frequency-domain) identity, both at length $n_{\mathrm{fft}}$. Term~(ii) is the four Heaps/Ansley--Kohn PAR maps, each a Durbin--Levinson recursion costing $O(k^2 d^3)$ for a factor of order $k$. Term~(iii) convolves the regular and seasonal factors into the expanded operators ($K(k{+}1)$ block products for a $k\times K$ factor pair). Term~(iv) assembles and zero-pads those expanded coefficients and evaluates the per-factor penalties. Term~(v) is the Cholesky log-determinant of the $\Sigma$-marginalized loss~\eqref{eq:marginalized_MAP}, absent for the RLS loss of Section~\ref{sec:RLS}. The plain VARMA is the case $P=Q=0$, for which (iii) vanishes and (ii) reduces to $O((p^2{+}q^2)d^3)$.

The orders $(p,q,P,Q)$ and the seasonal period $s$ enter only the subdominant terms (ii)--(iv): the leading term~(i) is independent of all of them, because the degree-$\ell_{\mTheta}$ expanded MA polynomial is zero-padded to length $n_{\mathrm{fft}}$ before transformation, so its FFT costs the same regardless of how many of its coefficients are nonzero. A seasonal VARMA therefore evaluates at the same leading $O(m\log m\,d^2+m\,d^3)$ per-gradient cost as a plain VARMA of equal dimension; the seasonal structure adds only the $O(s(P{+}Q)\,d^2)$ growth of~(iv) and the $O((P^2{+}Q^2)d^3)$ of~(ii)--(iii). These subdominant terms are dominated by~(i) whenever $m\gtrsim\ell_{\mPhi}$ (already required for the truncation bound~\eqref{eq:gamma_bias} to be tight), but \eqref{eq:sarma_cost} holds for any $(p,q,P,Q,m,d)$, including high-order regimes (e.g.\ a full-diurnal regular factor $p=s$) where term~(ii) can rival~(i). Reverse-mode automatic differentiation returns the gradient at the same asymptotic cost and stores no recursion intermediates, the FFT inversion having replaced the sequential Box--Jenkins scan. The optimisation itself runs over the $(p{+}P{+}q{+}Q)d^2$ unconstrained PAR coordinates, an $s$-fold smaller vector than the $(\ell_{\mPhi}{+}\ell_{\mTheta})d^2$ of the equivalent unrestricted VARMA$(\ell_{\mPhi},\ell_{\mTheta})$.

\paragraph{Exogenous regressors (VARMAX).} The VARMAX and seasonal-VARMAX extensions of Section~\ref{sec:seasonal_varma} add only lower-order work. The covariates enter the sufficient statistics through the cross-covariances of $\vZ$ with $\vct{x}$, formed in the same $T$-dependent pre-pass, and the exogenous coefficients $\mPi^{\mathrm{x}}(z)=[\mTheta^{(s)}(z^{s})\,\mTheta(z)]^{-1}\mtx{\Xi}(z)$ are obtained by passing $\mtx{\Xi}$ through the same FFT inversion already used to form $\mPi_{1:m}$. The per-gradient cost is therefore that of~\eqref{eq:sarma_cost} at the same leading order, the exogenous block adding only an $O(m\log m\,d\,d_x + m\,d^2 d_x)$ FFT-and-contraction term; in particular it remains independent of $T$ and near-linear in $m$.

\section{Extension to VARMAX models}\label{sec:VARMAX_appendix}

The framework of Sections~\ref{sec:VARMAintro}--\ref{sec:complexity_S2} and the online updates of Appendix~\ref{sec:online} extend to models with eXogenous inputs (VARMAX) by a single augmentation: the exogenous regressors enter the residual, the sufficient statistics, and the Fourier evaluation alongside the lagged outputs, while the PAR map $\mathcal{T}$, the $\Sigma$-marginalisation, the four $(\mM^{\mPhi},\mM^{\mTheta})$ ridge penalties, the Box--Jenkins recursion~\eqref{eq:mPirec}--\eqref{eq:mPirec2} for $\mPi_{1:m}$, and the unconstrained L-BFGS optimisation all carry over verbatim. We record only what changes.

\subsection{The truncated VARMAX$(p,q,s,m)$ model}\label{sec:VARMAX_model}

Let $\vY_t\in\R^d$ be the target series and $\vX_t\in\R^{d_x}$ a vector of exogenous covariates observed in synchrony with $\vY_t$. The VARMAX$(p,q,s)$ model is
\begin{equation}\label{eq:VARMAXdef}
\vY_t \;=\; \sum_{j=1}^p \mPhi_j\,\vY_{t-j} \,+\, \sum_{l=0}^s \mB_l\,\vX_{t-l} \,+\, \va_t \,-\, \sum_{k=1}^q \mTheta_k\,\va_{t-k},
\end{equation}
with $\va_t\overset{\mathrm{iid}}{\sim}N(\bm{0},\mSigma)$, $\mPhi_{1:p}\in\R^{p\times d\times d}$, $\mTheta_{1:q}\in\R^{q\times d\times d}$, and exogenous regression coefficients $\mB_{0:s}\in\R^{(s+1)\times d\times d_x}$. Stationarity is again controlled by Assumption~\ref{ass:roots} on $(\mPhi,\mTheta)$ only; the regression block $\mB_{0:s}$ plays no role in the spectral-radius conditions.

Pre-multiplying~\eqref{eq:VARMAXdef} by $\mTheta(B)^{-1}$ gives the VAR$(\infty)$ form
\[
\vY_t=\sum_{j\ge1}\mPi_j\vY_{t-j}+\sum_{l\ge0}\widetilde{\mB}_l\vX_{t-l}+\va_t,
\]
in which $\mPi_{1:\infty}$ are exactly the pure-VARMA coefficients of~\eqref{eq:mPirec}--\eqref{eq:mPirec2} and the exogenous weights satisfy the analogous Box--Jenkins recursion
\begin{equation}\label{eq:Btilde_rec}
\widetilde{\mB}_l \;=\; \mB_l \,+\, \sum_{k=1}^{\min(l,q)}\mTheta_k\,\widetilde{\mB}_{l-k},\qquad \mB_l=\mz\ (l>s),\ \ \widetilde{\mB}_l=\mz\ (l<0).
\end{equation}
Truncating both sums at lag $m$ gives the working model
\begin{equation}\label{eq:VARMAXtrunc}
\vY_t \;\approx\; \sum_{j=1}^m \mPi_j\,\vY_{t-j} \,+\, \sum_{l=0}^{m}\widetilde{\mB}_l\,\vX_{t-l} \,+\, \va_t.
\end{equation}
The endogenous truncation bound~\eqref{eq:gamma_bias} applies verbatim to $\mPi_{1:m}$ (Proposition~\ref{prop:trunc_bias}). The exogenous weights inherit the same geometric coefficient decay ($\widetilde{\mB}(z)=\mTheta(z)^{-1}\mB(z)$ has the MA inverse as its only non-polynomial factor, so $\|\widetilde{\mB}_l\|\le C'_\rho\,\rho^{l}$ for every $\rho>\spr(\mathcal{C}_{\mTheta})$), but we do not claim a stationary autocovariance bound, since $\{\vX_t\}$ may be deterministic, nonstationary, or merely observed; the dropped tail instead perturbs the one-step mean by $\sum_{l>m}\widetilde{\mB}_l\vX_{t-l}$, negligible at the $m\sim10^3$ we use whenever the covariates do not grow faster than $\rho^{-l}$ (in particular for bounded $\{\vX_t\}$). The regression block $\mB_{0:s}$ enters~\eqref{eq:VARMAXdef} linearly and is unconstrained.

\subsection{Augmented sufficient statistics and residual scatter}\label{sec:VARMAX_stats}

In addition to the VARMA sufficient statistics $\mS_0,\mS_1,\mS_2$ of Section~\ref{sec:LSloss}, the VARMAX residual scatter requires three exogenous cross-statistics computed over the same $N=T-m$ target rows:
\begin{equation}\label{eq:VARMAX_suffstats}
\begin{aligned}
\mS_1^{xy}(l) &\;:=\;\textstyle \sum_t \vX_{t-l}\,\vY_t', & l&=0,\ldots,m,\\
\mS_2^{yx}(j,l) &\;:=\;\textstyle \sum_t \vY_{t-j}\,\vX_{t-l}', & j&=1,\ldots,m,\quad l=0,\ldots,m,\\
\mS_2^{xx}(l,l') &\;:=\;\textstyle \sum_t \vX_{t-l}\,\vX_{t-l'}', & l,l'&=0,\ldots,m.
\end{aligned}
\end{equation}
With $\vr_t:=\vY_t-\sum_j\mPi_j\vY_{t-j}-\sum_l\widetilde{\mB}_l\vX_{t-l}$ the truncated residual of~\eqref{eq:VARMAXtrunc}, the residual scatter $\mR^{(X)}:=\sum_t\vr_t\vr_t'$ extends~\eqref{eq:R_residual_cov} additively:
\begin{equation}\label{eq:R_VARMAX}
\begin{aligned}
\mR^{(X)} \;=\; \mR \;&-\; \sum_{l}\widetilde{\mB}_l\,\mS_1^{xy}(l) \;-\; \sum_{l}\mS_1^{xy}(l)'\widetilde{\mB}_l'\\
&+\; \sum_{j,l} \bigl(\mPi_j\,\mS_2^{yx}(j,l)\,\widetilde{\mB}_l' + \widetilde{\mB}_l\,\mS_2^{yx}(j,l)'\,\mPi_j'\bigr) \;+\; \sum_{l,l'}\widetilde{\mB}_l\,\mS_2^{xx}(l,l')\,\widetilde{\mB}_{l'}',
\end{aligned}
\end{equation}
where $\mR$ is the VARMA scatter~\eqref{eq:R_residual_cov} unchanged. The truncated least-squares criterion is $L_m^{(X)}=\tr\mR^{(X)}$, which depends on the data only through $\{\mS_0,\mS_1,\mS_2,\mS_1^{xy},\mS_2^{yx},\mS_2^{xx}\}$.

\subsection{Regularised least squares and $\Sigma$-marginalised MAP}\label{sec:VARMAX_RLS}

The regularisation of Sections~\ref{sec:RLS}--\ref{sec:MAP} carries over with one addition: fixed-variance ridge penalties on the exogenous coefficients,
\begin{equation}\label{eq:B_prior}
(\mB_l)_{ab} \;\overset{\mathrm{iid}}{\sim}\; N\bigl(0,\sigma^2_{\mB,\,g(a,b)}\bigr),\qquad l=0,\ldots,s,
\end{equation}
with two group variances $\sigma^2_{\mB,\mathrm{direct}},\sigma^2_{\mB,\mathrm{indirect}}$ separating own-channel from cross-channel regressions (tied to a single scalar, or grouped by lag, if preferred). The RLS and $\Sigma$-marginalised losses are then those of Eqs.~\eqref{eq:rls_loss} and~\eqref{eq:marginalized_MAP} with $\mR\to\mR^{(X)}$ of~\eqref{eq:R_VARMAX} and the two $\mB$-ridges appended:
\begin{equation}\label{eq:rls_VARMAX}
\begin{aligned}
\mathcal{L}_m^{\mathrm{RLS},X} &= \tfrac{1}{2}L_m^{(X)} + (\text{4 }\mM\text{-ridges}) + \sum_{g}\tfrac{S_{\mB,g}}{2\sigma^2_{\mB,g}},\\
\mathcal{L}_m^{\mathrm{marg},X} &= \tfrac{\nu_0+N}{2}\,\log\det\!\bigl(\mPsi_0+\mR^{(X)}\bigr) + (\text{the same six ridges}),
\end{aligned}
\end{equation}
with $S_{\mB,g}:=\sum_{l=0}^{s}\sum_{a,b:\,g(a,b)=g}(\mB_l)_{ab}^{2}$ over the two groups $g\in\{\mathrm{direct},\mathrm{indirect}\}$ and $L_m^{(X)}=\tr\mR^{(X)}$. The L-BFGS state is $(\mM^{\mPhi}_{1:p},\mM^{\mTheta}_{1:q},\mB_{0:s})\in\R^{(p+q)d^2+(s+1)d\,d_x}$, unconstrained in every coordinate; the optimisation of Sections~\ref{sec:RLS}--\ref{sec:MAP} is otherwise unchanged.

\subsection{Fourier evaluation and cost}\label{sec:VARMAX_fourier}

The block-Toeplitz Parseval evaluation of Section~\ref{sec:complexity_S2} applies to each new second-moment block of~\eqref{eq:VARMAX_suffstats}: the symmetric exogenous quadratic $\mS_2^{xx}$ exactly as for $\mS_2$ (caching $\widehat{\mC}^{xx}_{\tau}=\sum_t \vX_t\vX_{t+\tau}'$, $\tau=0,\ldots,m-1$), and the non-symmetric cross block $\mS_2^{yx}$ likewise but caching the cross-autocovariance $\widehat{\mC}^{yx}_{\tau}=\sum_t\vY_t\vX_{t+\tau}'$ for both signs of $\tau$; the linear cross term $\mS_1^{xy}$ is already $O(m\,d\,d_x)$ and needs no acceleration. The per-evaluation cost of $\mathcal{L}_m^{\mathrm{RLS},X}$ or $\mathcal{L}_m^{\mathrm{marg},X}$ is then
\begin{equation}\label{eq:flopcount_VARMAX}
O\bigl(m\log m\,(d+d_x)^2 \,+\, m\,(d+d_x)^3\bigr),
\end{equation}
reducing to~\eqref{eq:flopcount} when $d_x=0$, with a single $O(T\log m\,(d+d_x)^2)$ pre-pass (caching $\widehat{\mC}^{yy}$, $\widehat{\mC}^{xx}$, $\widehat{\mC}^{yx}$) amortised across all L-BFGS iterations and rolling-window refits.

\subsection{Online estimation}\label{sec:VARMAX_online}

The rolling-window machinery of Appendix~\ref{sec:online} extends with no structural change. The raw bundle gains the three exogenous blocks $\mS_1^{xy,\mathrm{raw}},\mS_2^{yx,\mathrm{raw}},\mS_2^{xx,\mathrm{raw}}$, each a sum of outer products to which the drop-and-add identity~\eqref{eq:rawupdate} applies verbatim, at per-shift cost $O(\Delta\,m\,(d+d_x)^2)$ independent of $T$; the closed-form standardisation~\eqref{eq:znorm_stats} applies separately to the $\vY$ and $\vX$ blocks via the same conjugate identity~\eqref{eq:znorm_stats}. Warm-started L-BFGS (Appendix~\ref{sec:online}) is unchanged, the warm start now carrying the unconstrained regression block $\mB_{0:s}$ with no projection, at the Fourier cost~\eqref{eq:flopcount_VARMAX} per refit.

\end{document}